\documentclass[10pt]{article}
\usepackage[T1]{fontenc}

\usepackage{amsmath, dsfont, bm, mathrsfs}
\usepackage{tikz}
\usepackage{mathtools}
\usepackage{appendix}
\usepackage{comment}
\usepackage{wrapfig}
\usepackage{amsthm}
\usepackage{amssymb}
\usepackage{natbib}

\newtheorem{theorem}{Theorem}

\newtheorem{lemma}[theorem]{Lemma}
\newtheorem{proposition}[theorem]{Proposition}
\newtheorem{definition}{Definition}
\newtheorem{assumption}{Assumption}

\theoremstyle{remark}
\newtheorem{remark}{Remark}

\theoremstyle{remark}
\newtheorem*{remark*}{Remark}

\newcommand{\defeq}{\mathrel{\mathop:}=}

\newcommand{\XX}{\mathbf{X}}
\newcommand{\Xii}{\XX_i}

\newcommand{\Xjj}{\XX_j}
\newcommand{\Xkk}{\XX_k}
\newcommand{\xsmall}{\bm{x}}

\newcommand{\xdensity}{\mu}
\newcommand{\vdensity}{\lambda}
\newcommand{\nvdensity}{\tilde{\vdensity}}

\newcommand{\hsmall}{\bm{h}}
\newcommand{\hx}{\hsmall_{\Vzero,\xsmall}}
\newcommand{\Ht}{\mathbf{H}}
\newcommand{\Hti}{\Ht_i}
\newcommand{\Htj}{\Ht_j}

\newcommand{\PH}{\mathbf{P}}

\newcommand{\yii}{y_i}
\newcommand{\ysmall}{\bm{y}}

\newcommand{\esmall}{\bm{\epsilon}}

\newcommand{\zzero}{\bm{z}_0}
\newcommand{\zzeroorth}{\bm{z}_{0,\perp}}

\newcommand{\hatf}{\hat{\f}}

\newcommand{\f}{f}
\newcommand{\ft}{\tilde{f}}

\newcommand{\F}{\mathbf{F}}

\newcommand{\fg}{\f_g}
\newcommand{\fl}{\hatf^{\ell_2}}
\newcommand{\flinf}{\fl_{\infty}}
\newcommand{\flinfSet}{\mathcal{F}^{\infty}}
\newcommand{\fv}{\f_{\Vzero}^g}
\newcommand{\Fv}{\F_{\Vzero}^g}
\newcommand{\flmap}{\fl}

\newcommand{\learnableSet}{\mathcal{F}^{\ell_2}}

\newcommand{\hf}{h}

\newcommand{\s}{s}

\newcommand{\gb}{g}
\newcommand{\gbi}{\gb_{i,p}}
\newcommand{\gbinf}{\gb_{i,\infty}}

\newcommand{\Vzero}{\mathbf{V}_0}
\newcommand{\Vzeroi}{\mathbf{V}_{0}[i]}
\newcommand{\Vzeroj}{\mathbf{V}_{0}[j]}
\newcommand{\Vzerok}{\mathbf{V}_{0}[k]}
\newcommand{\Vzerol}{\mathbf{V}_{0}[l]}

\newcommand{\DV}{\Delta \mathbf{V}}
\newcommand{\tDV}{\overline{\DV}}
\newcommand{\DVl}{\DV^{\ell_2}}
\newcommand{\DVs}{\DV^*}

\newcommand{\DVj}{\DV[j]}
\newcommand{\tDVj}{\tDV[j]}

\newcommand{\Vg}{\mathbf{V}^{\text{GD}}}
\newcommand{\Vgk}{\Vg_k}
\newcommand{\DVg}{\Delta \Vg}
\newcommand{\DVgk}{\Delta\Vgk}

\newcommand{\vstari}{\bm{v}_{*,i}}

\newcommand{\capr}{\mathcal{B}^r}
\newcommand{\capsr}{\mathcal{B}_{\bm{v}_*}^r}
\newcommand{\capsrz}{\mathcal{B}_{\bm{v}_*}^{r_0}}
\newcommand{\capsrh}{\mathcal{B}^{\hat{r}}}
\newcommand{\capsdelta}{\mathcal{B}^\frac{\delta}{n^2C_d}}
\newcommand{\capsrpx}{\mathcal{B}_{\bm{v}_*,+}^{r,\xsmall}}
\newcommand{\capsrnx}{\mathcal{B}_{\bm{v}_*,-}^{r,\xsmall}}
\newcommand{\capsrX}{\mathcal{B}_{\bm{v}_{*,i}}^{r_i}}
\newcommand{\capsrpX}{\mathcal{B}_{\bm{v}_{*,i},+}^{r_i,\Xii}}
\newcommand{\capsrnX}{\mathcal{B}_{\bm{v}_{*,i},-}^{r_i,\Xii}}
\newcommand{\capsrzpX}{\mathcal{B}_{\bm{v}_{*,i},+}^{r_0,\Xii}}
\newcommand{\capsrznX}{\mathcal{B}_{\bm{v}_{*,i},-}^{r_0,\Xii}}

\newcommand{\sd}{\mathcal{S}^{d-1}}
\newcommand{\sn}{\mathcal{S}^{n-1}}
\newcommand{\identity}{\mathbf{I}}
\DeclareMathOperator*{\expectation}{\mathsf{E}}
\DeclareMathOperator*{\vari}{\mathsf{Var}}
\DeclareMathOperator*{\argmax}{arg\,max}
\DeclareMathOperator*{\argmin}{arg\,min}
\DeclareMathOperator*{\prob}{\mathsf{Pr}}
\DeclareMathOperator*{\rank}{\mathsf{rank}}

\newcommand{\TODO}{{\color{red}TODO}}

\newcommand{\bino}{\mathsf{Bino}}

\newcommand{\CXix}{\mathcal{C}_{\Xii,\xsmall}^{\Vzero}}
\newcommand{\Czx}{\mathcal{C}_{\bm{z},\xsmall}^{\Vzero}}

\newcommand{\zeroIndex}{\mathcal{Z}_{\Xi}}

\newcommand{\musd}{\lambda_{d-1}}

\newcommand{\eig}{\mathsf{eig}}

\newcommand{\diag}{\mathsf{diag}}

\makeatletter
\newcommand*\bigcdot{\mathpalette\bigcdot@{.5}}
\newcommand*\bigcdot@[2]{\mathbin{\vcenter{\hbox{\scalebox{#2}{$\m@th#1\bullet$}}}}}
\makeatother

\newcommand{\w}{\bm{w}}
\newcommand{\wj}{\w_j}

\newcommand{\DX}{D_{\XX}}

\usepackage[verbose=true,letterpaper]{geometry}
\AtBeginDocument{
  \newgeometry{
    textheight=9in,
    textwidth=6.5in,
    top=1in,
    headheight=14pt,
    headsep=25pt,
    footskip=30pt
  }
}

\usepackage{hyperref}
\hypersetup{
    colorlinks=true,
    linkcolor=blue,
    filecolor=blue,      
    urlcolor=blue,
    citecolor=blue,
}
\usepackage{breakurl}

\title{On the Generalization Power of\\ Overfitted Two-Layer Neural Tangent Kernel Models}

\newcommand*\samethanks[1][\value{footnote}]{\footnotemark[#1]}
\author{Peizhong Ju\thanks{School of Electrical and Computer Engineering, 
 Purdue University. Email: \texttt{\{jup,linx\}@purdue.edu}} \and Xiaojun Lin\samethanks \and Ness B. Shroff\thanks{Department of ECE and CSE, The Ohio State University. Email: \texttt{shroff.11@osu.edu}}}

\date{March 1, 2021}

\begin{document}

\maketitle

\begin{abstract}
In this paper, we study the generalization performance of min $\ell_2$-norm overfitting solutions for the neural tangent kernel (NTK) model of a two-layer neural network with ReLU activation that has no bias term. We show that, depending on the ground-truth function, the test error of overfitted NTK models exhibits characteristics that are different from the ``double-descent'' of other overparameterized linear models with simple Fourier or Gaussian features. Specifically, for a class of learnable functions, we provide a new upper bound of the generalization error that approaches a small limiting value, even when the number of neurons $p$ approaches infinity. This limiting value further decreases with the number of training samples $n$. For functions outside of this class, we provide a lower bound on the generalization error that does not diminish to zero even when $n$ and $p$ are both large.
\end{abstract}

\section{Introduction}


Recently, there is significant interest in understanding why overparameterized deep neural networks (DNNs) can still generalize well \citep{zhang2016understanding,advani2017high}, which seems to defy the classical understanding of \emph{bias-variance tradeoff} in statistical
learning \citep{bishop2006pattern, hastie2009elements, stein1956inadmissibility, james1992estimation,lecun1991second,tikhonov1943stability}.
Towards this direction,
a recent line of study
	has focused on overparameterized linear models
	\citep{belkin2018understand, belkin2019two, bartlett2019benign,
	hastie2019surprises,muthukumar2019harmless,ju2020overfitting,mei2019generalization}. For linear models with simple features (e.g., Gaussian features
	and Fourier features) 
	\citep{belkin2018understand, belkin2019two, bartlett2019benign,
	hastie2019surprises,muthukumar2019harmless,ju2020overfitting},
	an interesting ``double-descent'' phenomenon has been observed. Thus, there is a region where the
	number of model parameters (or linear features) is larger than the number
	of samples (and thus overfitting occurs), but the generalization error actually decreases
	with the number of features. However, linear models with these simple features are still quite different from nonlinear neural networks. Thus, although such results provide some hint why
	overparameterization and overfitting may be harmless, it is still unclear whether similar conclusions apply to neural networks.
	
In this
	paper, we are interested in linear models based on the neural
	tangent kernel (NTK) \citep{jacot2018neural}, which can be viewed as a useful
	intermediate step towards modeling nonlinear neural networks.
	Essentially, NTK can be seen as a linear approximation of neural
	networks when the weights of the neurons do not change much.  Indeed,
	\cite{li2018learning,du2018gradient} have shown that, for a wide and
	fully-connected two-layer neural network, both the neuron
	weights and their
	activation patterns do not change much after
	gradient descent (GD) training with a sufficiently small step size. As a result, such a shallow
	and wide neural network is approximately linear in the weights
	when there are a sufficient number of neurons, which suggests
	the utility of the NTK model.

	Despite its linearity, however, characterizing the double descent of such a NTK model remains elusive.
	The work in 
	\cite{mei2019generalization} also studies the double-descent of
	a linear version of two-layer neural network. It uses the so-called
	``random-feature'' model, where the bottom-layer weights are
	random and fixed, and only the top-layer weights are trained.
	(In comparison, the NTK model for such a two-layer neural network
	corresponds to training only the bottom-layer weights.)
	However, the setting there requires the
	number of neurons, the number of samples, and the data dimension to
	all grow proportionally to infinity. In contrast, we are
	interested in the setting where the number of samples is given,
	and the number of neurons is allowed to be much larger than
	the number of samples. As a consequence of the different
	setting, in \cite{mei2019generalization}
	eventually only \emph{linear} ground-truth functions can be learned. (Similar settings are also studied in \cite{d2020double}.) In
	contrast, we will show that far more complex functions can be
	learned in our setting. 
	In a related work, \cite{ghorbani2019linearized} shows that
	both the random-feature model and the NTK model can approximate
	highly \emph{nonlinear} ground-truth functions with a sufficient number of
	neurons.  However, \cite{ghorbani2019linearized} mainly studies
	the \emph{expressiveness} of the models, and therefore does not explain
	why overfitting solutions can still generalize well. To the best
	of our knowledge, our work is the first to characterize the
	double-descent of overfitting solutions based on the NTK model.

Specifically, in this paper we study the generalization error of the min
$\ell_2$-norm overfitting solution for a linear model based on the NTK
of a two-layer neural network with ReLU activation that has no bias. Only the bottom-layer weights are trained. 
We are interested in min $\ell_2$-norm overfitting solutions because
gradient descent (GD) can be shown to converge to such solutions while
driving the training error to zero \citep{zhang2016understanding} (see also Section~\ref{sec.system_model}). 
Given a class of ground truth functions (see details in Section~\ref{sec.main}), 
which we refer to as ``learnable functions,'' 
our main result (Theorem~\ref{th.combine}) provides an upper bound on the
generalization error of the min $\ell_2$-norm overfitting solution for
the two-layer NTK model with $n$ samples and $p$ neurons (for any finite $p$ larger than a polynomial function of $n$). 
This upper bound confirms that the generalization error of the overfitting solution indeed exhibits descent in the overparameterized regime when $p$ increases.
Further, our upper bound can also account for the noise in the training samples.


Our results reveal several important insights. First, we find that the
(double) descent of the overfitted two-layer NTK model is drastically different from
that of linear models with simple Gaussian or Fourier features \citep{belkin2018understand, belkin2019two,bartlett2019benign,	hastie2019surprises,muthukumar2019harmless}. 
%
%
Specifically, for linear models with simple features, when the number of
features $p$ increases, the generalization error will eventually
grow again and approach the so-called ``null risk'' \citep{hastie2019surprises},
which is the error of a trivial model that predicts zero.
%
In contrast, for the class of learnable functions described earlier, the
generalization error of the overfitted NTK model will continue to descend as $p$ grows to infinity, and will
approach a limiting value that depends on the number of samples $n$.
Further, when there is no noise, this limiting value will decrease to
zero as the number of samples $n$ increases. 
%
This difference is shown in
Fig.~\ref{fig.compare}(a). As $p$ increases, the test mean-square-error (MSE) of
min-$\ell_1$ and min-$\ell_2$ overfitting solutions for Fourier features
(blue and red curves) eventually grow back to the null risk (the black
dashed line), even though they exhibit a descent at smaller $p$. In contrast, the error of the overfitted NTK model continues to descend to a
much lower level.

The second important insight is that the aforementioned behavior critically
depends on the ground-truth function belonging to the class of ``learnable functions.'' Further, this class of learnable functions depend on the specific network architecture. For our NTK model (with RELU activation that has no bias), we precisely characterize this class of learnable functions. Specifically, 
for ground-truth functions 
that are outside the class of learnable functions, 
we show a lower bound on the generalization error that does not diminish to zero for any $n$ and $p$ (see Proposition~\ref{th.equal_set} and Section~\ref{sec.feature_of_set}). 
This difference is shown in Fig.~\ref{fig.compare}(b), where we use an
almost identical setting as Fig.~\ref{fig.compare}(a), 
except a different ground-truth function.
We can see in Fig.~\ref{fig.compare}(b) that the test-error of the overfitted NTK model is
always above the null risk and looks very different from that in
Fig.~\ref{fig.compare}(a).
We note that whether certain functions are learnable or not critically depends on the specific structure of the NTK model, such as the choice of the activation unit. Recently, \citep{Srikant21} shows that all polynomials can be learned by 2-layer NTK model with ReLU activation that has a bias term, provided that the number of neurons $p$ is sufficiently large.  (See further discussions in Remark~\ref{remark.bias_enlarge}. However, \citep{Srikant21} does not characterize the descent of generalization errors as $p$ increases.) This difference in the class of learnable functions between the two settings (ReLU with or without bias) also turns out to be consistent with the difference in the expressiveness of the neural networks. That is, shallow networks with biased-ReLU are known to be  universal function approximators \citep{ji2019neural}, while those without bias can only approximate the sum of linear functions and even functions \citep{ghorbani2019linearized}.

\begin{figure}[t!]
    \centering
    \includegraphics[width=0.4\textwidth]{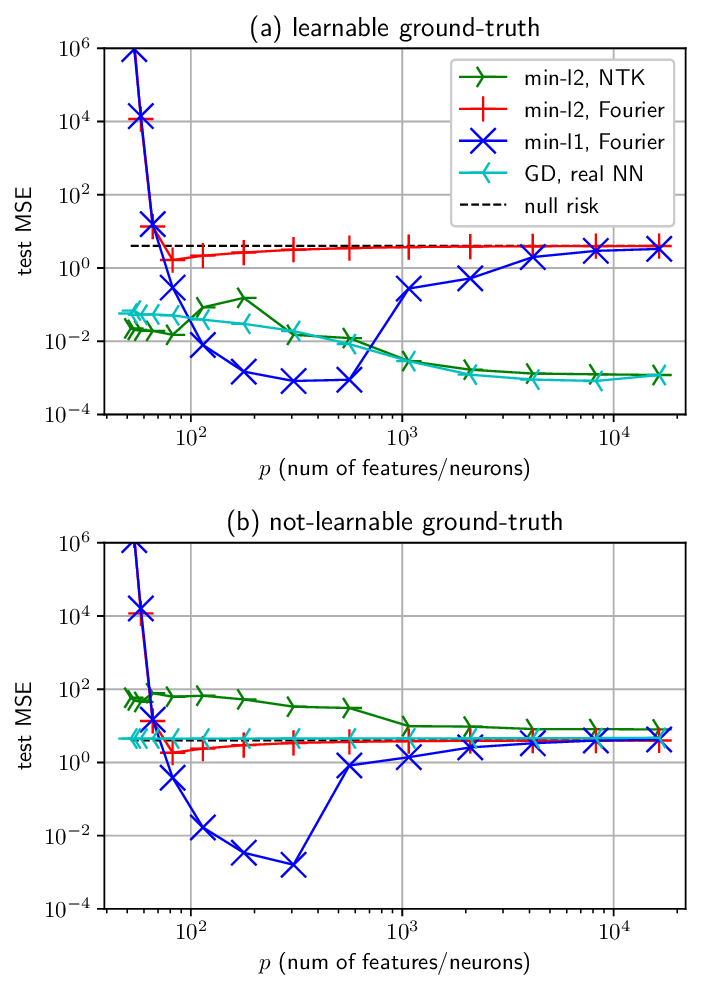}
    \caption{The test mean-square-error(MSE) vs. the number of
	    features/neurons $p$ for \textbf{(a)} learnable function  and \textbf{(b)} not-learnable function when $n=50$, $d=2$,
	    $\|\esmall\|_2^2=0.01$. The corresponding ground-truth are \textbf{(a)}  $\f(\theta)=\sum_{k\in\{0,
	    1,2,4\}}(\sin(k\theta)+\cos(k\theta))$, and \textbf{(b)}  $\f(\theta)=\sum_{k\in\{3, 5, 7, 9\}}(\sin(k\theta)+\cos(k\theta))$. (Note that in 2-dimension every input $\xsmall$ on a unit circle can be represented by an angle $\theta\in[-\pi, \pi]$. See the end of Section~\ref{sec.feature_of_set}.) Every curve is the average of $9$ random simulation runs. For GD on the real neural network (NN), we use the step size $1/\sqrt{p}$ and the number of training epochs is fixed at $2000$.}
    \label{fig.compare}
\end{figure}

A closely related result to ours is the work in \cite{arora2019fine}, which
characterizes the generalization performance of wide two-layer neural
networks whose bottom-layer weights are trained by gradient descent (GD) to
overfit the training samples. In particular, our class of learnable
functions almost coincides with that of \cite{arora2019fine}.
This is not surprising
because, when the number of neurons is large, NTK becomes a close
approximation of such two-layer neural networks. In that sense, the
results in \cite{arora2019fine} are even more faithful in following the GD
dynamics of the original two-layer network. However, the advantage of the
NTK model is that it is easier to analyze. In particular, the results in
this paper can quantify how the generalization error descends with $p$.
In contrast, the results in \cite{arora2019fine} provide only a
generalization bound that is independent of $p$ (provided that $p$ is
sufficiently large), but do not quantify the descent behavior as $p$
increases. 
Our numerical results in Fig.~\ref{fig.compare}(a) suggest that, over a wide range of $p$, the
descent behavior of the NTK model (the green curve) matches well with that of two-layer
neural networks trained by gradient descent (the cyan curve). 
Thus, we believe that our results also provide guidance for the latter model. The work in \citep{fiat2019decoupling} studies a different neural network architecture with gated ReLU, whose NTK model turns out to be the same as ours. However, similar to \cite{arora2019fine}, the result in \cite{fiat2019decoupling} does not capture the speed of descent with respect to $p$ either.
Second,
\cite{arora2019fine} only provides upper bounds on the generalization
error. There is no corresponding lower bound to explain whether 
ground-truth functions outside a certain class are
\emph{not} learnable. Our result in Proposition~\ref{th.equal_set} provides such a lower
bound, and therefore more completely characterizes the class of learnable
functions. 
(See further comparison in Remark~\ref{remark.main} of Section~\ref{sec.main} and Remark~\ref{remark.h_inf} of Section~\ref{sec.proof_combine}.)
Another related work \cite{allen2019learning} also
characterizes the class of learnable functions for two-layer and three-layer
networks. 
However, \cite{allen2019learning} studies a training method that takes a
new sample in every iteration, and thus does not overfit all training data.
Finally, our paper studies generalization of NTK models for the regression setting, which is different from the classification setting that assumes a separability condition, e.g., in \citep{ji2019polylogarithmic}.

\section{Problem Setup}\label{sec.system_model}


We assume the following data model $y=\f(\xsmall)+\epsilon$, with the input $\xsmall\in\mathds{R}^d$, the output $y\in\mathds{R}$, the noise $\epsilon\in\mathds{R}$, and $\f:\ \mathds{R}^d\mapsto\mathds{R}$ denotes the ground-truth function. Let $(\Xii,\ y_i)$, $i=1,2,\cdots, n$ denote $n$ training samples. We collect them as $\XX=[\XX_1\ \XX_2\ \cdots\ \XX_n]\in\mathds{R}^{d\times n}$, $\ysmall=[y_1\ y_2\ \cdots\ y_n]^T\in\mathds{R}^n$, $\esmall=[\epsilon_1\ \epsilon_2\ \cdots\ \epsilon_n]^T\in\mathds{R}^n$, and $\F(\XX)=[\f(\XX_1)\ \f(\XX_2)\ \cdots\ \f(\XX_n)]^T\in\mathds{R}^n$. Then, the training samples can be written as $\ysmall=\F(\XX)+\esmall$. After training (to be described below), we denote the trained model by the function $\hatf$. Then, for any new test data $\xsmall$, we will calculate the test error by $|\hatf(\xsmall)-\f(\xsmall)|$, and the mean squared error (MSE) by $\expectation_{\xsmall}[\hatf(\xsmall)-\f(\xsmall)]^2$.


\begin{figure}
    \centering
    \begin{tikzpicture}
\draw  [fill](0,4.5) node (v1) {} circle (0.1);
\draw [fill][color=red] (-1,3.5) node (v2) {} circle (0.1);
\draw  [fill][color=blue](1,3.5) node (v3) {} circle (0.1);
\draw [fill] (-1,2.5) node (v5) {} circle (0.1);
\draw[fill]  (1,2.5) node (v7) {} circle (0.1);
\draw[fill][color=black]  (0,3.5) node (v9) {} circle (0.1);
\draw[color=red]  (v1) edge (v2);
\draw[color=blue]  (v1) edge (v3);
\draw[color=red]  (v2) edge (v5);
\draw[color=red]  (v2) edge (v7);

\draw[color=blue]  (v3) edge (v5);
\draw[color=blue]  (v3) edge (v7);

\node at (3.8,4.5) {output};
\node at (3.8,4.0) {top layer weights $\bm{w}$};
\node at (3.8,2.5) {input $\bm{x}=[x_1\ x_2]^T$};
\node at (3.8,3.0) {bottom-layer weights $\mathbf{V}_0$};
\node at (3.8,3.5) {hidden-layer: ReLU};

\node at (-1,4) {{\color{red}$\bm{w}_{1}$}};
\node at (1,4) {{\color{blue}$\bm{w}_{3}$}};
\node at (-1.5,3) {{\color{red}$\mathbf{V}_0[1]$}};
\node at (0,3) {{\color{black}$\mathbf{V}_0[2]$}};
\node at (1.5,3) {{\color{blue}$\mathbf{V}_0[3]$}};
\node at (0,4) {{\color{black}$\bm{w}_{2}$}};
\node at (-1,2) {$x_1$};
\node at (1,2) {$x_2$};

\draw[color=black]  (v9) edge (v5);
\draw[color=black]  (v9) edge (v7);
\draw[color=black]  (v1) edge (v9);

\end{tikzpicture}
    \caption{A two-layer neural network where $d=2$, $p=3$.}
    \label{fig.struct}
\end{figure}
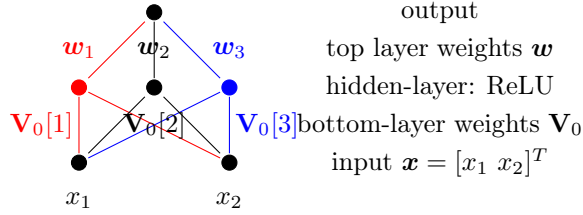


For training, consider a fully-connected two-layer neural network with $p$ neurons. Let $\wj\in\mathds{R}$ and $\Vzeroj\in\mathds{R}^d$ denote the top-layer and bottom-layer weights, respectively, of the $j$-th neuron, $j=1,2,\cdots,p$ (see Fig.~\ref{fig.struct}). We collect them into $\w=[\w_1\ \w_2\ \cdots\ \w_p]^T\in\mathds{R}^p$, and $\Vzero=[\Vzero[1]^T\ \Vzero[2]^T\ \cdots\ \Vzero[p]^T]^T\in\mathds{R}^{dp}$ (a column vector with $dp$ elements). Note that with this notation, for any row or column vector $\bm{v}$ with $dp$ elements,  $\bm{v}[j]$ denotes a (row/column) vector that consists of the $(jd+1)$-th to $(jd+d)$-th elements of $\bm{v}$.
We choose ReLU as the activation function for all neurons and there is no bias term in the ReLU activation function.


Now we are ready to introduce the NTK model \citep{jacot2018neural}. We fix the top-layer weights $\w$, and let the initial bottom-layer weights $\Vzero$ be randomly chosen. We then train only the bottom-layer weights.
Let $\Vzero+\tDV$ denote the bottom-layer weights after training. Thus, the change of the output after training is
\begin{align*}
    &\sum_{j=1}^p \wj \bm{1}_{\{\xsmall^T(\Vzeroj+\tDVj)>0\}}\cdot (\Vzeroj+\tDVj)^T\xsmall\\
    &-\sum_{j=1}^p \wj \bm{1}_{\{\xsmall^T\Vzeroj>0\}}\cdot\Vzeroj^T\xsmall.
\end{align*}
In the NTK model, one assumes that $\tDV$ is very small. As a result, $\bm{1}_{\{\xsmall^T(\Vzeroj+\tDVj)>0\}}= \bm{1}_{\{\xsmall^T\Vzeroj>0\}}$ for most $\xsmall$. Thus, the change of the output can be approximated by
\begin{align*}
    \sum_{j=1}^p \wj\bm{1}_{\{\xsmall^T\Vzeroj>0\}}\cdot\tDVj^T\xsmall=\hx \DV,
\end{align*}
where $\DV\in\mathds{R}^{dp}$ is given by
$\DVj\defeq\wj\tDVj,\ j=1,2,\cdots,p$,
and $\hx\in\mathds{R}^{1\times (dp)}$ is given by
\begin{align}\label{eq.def_hx}
    \hx[j]\defeq\bm{1}_{\{\xsmall^T\Vzeroj>0\}}\cdot\xsmall^T,\ j=1,2,\cdots,p.
\end{align}
In the NTK model, we assume that the output of the trained model is exactly given by Eq.~\eqref{eq.def_hx}, i.e.,
\begin{align}\label{eq.NTK_model}
    \hatf_{\DV, \Vzero}(\xsmall)\defeq \hx \DV.
\end{align}
In other words, the NTK model can be viewed as a linear approximation of the two-layer network when the change of the bottom-layer weights is small.


Define $\Ht\in\mathds{R}^{n\times (dp)}$ such that its $i$-th row is $\Hti\defeq\hsmall_{\Vzero,\Xii}$. 
Throughout the paper, we will focus on the following min-$\ell_2$-norm overfitting solution 
\begin{align*}
    \DVl\defeq\argmin_{\bm{v}} \|\bm{v}\|_2,\text{ subject to }\Ht\bm{v} = \ysmall.
\end{align*}
Whenever $\DVl$ exists, it can be written in closed form as
\begin{align}\label{eq.DVl}
     \DVl=\Ht^T(\Ht\Ht^T)^{-1}\ysmall.
\end{align}
The reason that we are interested in $\DVl$ is that gradient descent (GD) or stochastic gradient descent (SGD) for the NTK model in Eq.~\eqref{eq.NTK_model} is known to converge to $\DVl$ (proven in Supplementary Material, Appendix~\ref{ap.GD_same_l2}).

Using Eq.~\eqref{eq.NTK_model} and  Eq.~\eqref{eq.DVl}, the trained model is then
\begin{align}\label{def.fl}
    \fl(\xsmall) \defeq \hx \DVl.
\end{align}
In the rest of the paper, we will study the generalization error of Eq.~\eqref{def.fl}.

We collect some assumptions. Define the unit sphere in $\mathds{R}^d$ as: $\sd\defeq\left\{\bm{v}\in\mathds{R}^d\ |\ \|\bm{v}\|_2=1 \right\}$. Let $\xdensity(\cdot)$ denote the distribution of the input $\xsmall$.
Without loss of generality, we make the following assumptions: \textbf{(i)} the inputs $\xsmall$ are \emph{i.i.d.} uniformly distributed in $\sd$, and the initial weights $\Vzeroj$'s are \emph{i.i.d.} uniformly distributed in all directions in $\mathds{R}^d$; \textbf{(ii)} $p\geq n/d$ and $d\geq 2$; \textbf{(iii)} $\Xii\nparallel \Xjj$ for any $i\neq j$, and $\Vzerok\nparallel \Vzerol$ for any $k\neq l$. 
We provide detailed justification of those assumptions in Supplementary Material, Appendix~\ref{ap.justify_assump}. 

\section{Learnable Functions and Generalization Performance}\label{sec.main}

We now show that the generalization performance of the overfitted NTK model in Eq.~\eqref{def.fl} crucially depends on the ground-truth function $\f(\cdot)$, where good generalization performance only occurs when the ground-truth function is ``learnable.''
Below, we first describe a candidate class of ground-truth functions, and explain why they may correspond to the class of ``learnable functions.'' Then, we will give an upper-bound on the generalization performance for this class of ground-truth functions. Finally, we will give a lower-bound on the generalization performance when the ground-truth functions are outside of this class.


We first define a set $\learnableSet$ of ground-truth functions.
\begin{definition}\label{def.learnableSet}
$\learnableSet\defeq\big\{\f \stackrel{\text{a.e.}}{=}\fg\ \big|\ \fg(\xsmall)=\int_{\sd}\xsmall^T\bm{z}\frac{\pi-\arccos (\xsmall^T\bm{z})}{2\pi}g(\bm{z}) d\xdensity(\bm{z}),\ \|g\|_1< \infty\big\}$.
\end{definition}
Note that in Definition~\ref{def.learnableSet}, $\stackrel{\text{a.e.}}{=}$ means two functions equals almost everywhere, and $\|g\|_1\defeq\int_{\sd}|g(\bm{z})|d\xdensity(\bm{z})$.
The function $g(\bm{z})$ may be any finite-value function in $L^1(\sd\mapsto\mathds{R})$. Further, we also allow $g(\bm{z})$ to contain (as components) Dirac $\delta$-functions on $\sd$. Note that a $\delta$-function $\delta_{\bm{z}_0}(\bm{z})$ has zero value for all $\bm{z}\in \sd\setminus \{\bm{z}_0\}$, but $\|\delta_{\bm{z}_0}\|_1\defeq \int_{\sd}\delta_{\bm{z}_0}(\bm{z})d\xdensity(\bm{z})=1$. Thus, the function $g(\bm{z})$ may contain any sum of $\delta$-functions and finite-value $L^1$-functions. \footnote{Alternatively, we can also interpret $g(\bm{z})$ as a signed measure \citep{rao1983theory} on $\sd$. Then, $\delta$-functions correspond to point masses, and the condition $\|g\|_1<\infty$ implies that the corresponding unsigned version of the measure on $\sd$ is bounded.}


To see why $\learnableSet$ may correspond to the class of learnable functions, we can first examine what the learned function $\fl$ in Eq.~\eqref{def.fl} should look like. Recall that $\Ht^T=[\Ht_1^T\ \cdots\ \Ht_n^T]$. Thus, $\hx\Ht^T=\sum_{i=1}^n (\hx\Ht_i^T)\bm{e}_i^T$, where $\bm{e}_i\in\mathds{R}^n$ denotes the $i$-th standard basis. Combining Eq.~\eqref{eq.DVl} and Eq.~\eqref{def.fl}, we can see that the learned function in Eq.~\eqref{def.fl} is of the form
\begin{align}\label{eq.fl_before_limit}
    \fl(\xsmall)=&\hx\Ht^T(\Ht\Ht^T)^{-1}\ysmall\nonumber\\
    =&\sum_{i=1}^n\left(\frac{1}{p}\hx\Hti^T\right)p\bm{e}_i^T(\Ht\Ht^T)^{-1}\ysmall.
\end{align}
For all $\xsmall, \bm{z}\in\sd$, define $\Czx\defeq\{j\in\{1,2,\cdots,p\}\ |\ \bm{z}^T\Vzeroj>0, \xsmall^T\Vzeroj>0\}$, and its cardinality is given by
\begin{align}\label{eq.card_c}
    \left|\Czx\right|=\sum_{j=1}^p\bm{1}_{\{\bm{z}^T\Vzeroj>0,\ \xsmall^T\Vzeroj>0\}}.
\end{align}
Then, using Eq.~\eqref{eq.def_hx}, we can show $\frac{1}{p}\hx\Hti^T=\xsmall^T\Xii\frac{|\CXix|}{p}$. It is not hard to show that
\begin{align}\label{eq.large_p_converge}
    \frac{|\Czx|}{p}\stackrel{\text{P}}{\rightarrow}\frac{\pi-\arccos (\xsmall^T\bm{z})}{2\pi},\text{ as }p\to\infty.
\end{align}
where $\stackrel{\text{P}}{\rightarrow}$ denotes converge in probability.
(see Supplementary Material, Appendix~\ref{app.c_divide_p_convergence}). Thus, if we let
\begin{align}\label{eq.g_as_delta}
    g(\bm{z})=\sum_{i=1}^n p\bm{e}_i^T(\Ht\Ht^T)^{-1}\ysmall \delta_{\Xii}(\bm{z}),
\end{align}
then as $p\to\infty$, Eq.~\eqref{eq.fl_before_limit} should approach a function in $\learnableSet$. This explains why $\learnableSet$ is a candidate class of ``learnable functions.'' 
However, note that the above discussion only addresses the \emph{expressiveness} of the model. It is still unclear whether any function in $\learnableSet$ can be learned with low generalization error. The following result provides the answer.

For some $m\in\left[1,\ \frac{\ln n}{\ln \frac{\pi}{2}}\right]$, define (recall that $d$ is the dimension of $\xsmall$)
\begin{align}
    &J_m(n,d)\defeq 2^{2d+5.5}d^{0.5d}n^{\left(2+\frac{1}{m}\right)(d-1)}\label{eq.def_Jmnd}.
\end{align}
\begin{theorem}\label{th.combine}
Assume a ground-truth function  $\f\stackrel{\text{a.e.}}{=}\fg\in\learnableSet$ where $\|g\|_\infty<\infty$\footnote{
The requirement of $\|g\|_\infty<\infty$ can be relaxed. We show in Supplementary Material, Appendix~\ref{app.g_is_delta} that, even when $g$ is a $\delta$-function (so $\|g\|_\infty=\infty$), we can still have a similar result of Eq.~\eqref{eq.main_conclusion} but Term~1 will have a slower speed of decay $O(n^{-\frac{1}{2(d-1)}(1-\frac{1}{q})})$ with respect to $n$ instead of $O(n^{-\frac{1}{2}}(1-\frac{1}{q}))$ shown in Eq.~\eqref{eq.main_conclusion}. Term~4 of Eq.~\eqref{eq.main_conclusion} will also be different when $g$ is a $\delta$-function, but it still goes to zero when $p$ and $n$ are large.
}, $n\geq 2$, $m\in\left[1,\ \frac{\ln n}{\ln \frac{\pi}{2}}\right]$, $d\leq n^4$, and $ p\geq 6 J_m(n,d)\ln\left(4n^{1+\frac{1}{m}}\right)$.
Then, for any $q\in[1,\ \infty)$ and for almost every $\xsmall\in \sd$, we must have\footnote{The notion $\prob\limits_{M}$ in Eq.~\eqref{eq.main_conclusion} emphasizes that randomness is in $M$.}
\begin{align}
&\prob_{\Vzero,\XX}\Big\{|\fl(\xsmall)-\f(\xsmall)|\geq \underbrace{n^{-\frac{1}{2}\left(1-\frac{1}{q}\right)}}_{\text{Term 1}}\nonumber\\
&+\underbrace{\left(1+\sqrt{J_m(n,d)n}\right)p^{-\frac{1}{2}\left(1-\frac{1}{q}\right)}}_{\text{Term 2}}+\underbrace{\sqrt{J_m(n,d)n}\|\esmall\|_2}_{\text{Term 3}},\nonumber\\
&\quad\text{ for all }\esmall\in\mathds{R}^n\Big\}\leq 2e^2\bigg(\underbrace{\exp\left(-\frac{\sqrt[q]{n}}{8\|g\|_\infty^2}\right)}_{\text{Term 4}}\nonumber\\
&+\underbrace{\exp\left(-\frac{\sqrt[q]{p}}{8\|g\|_1^2}\right)}_{\text{Term 5}}+\underbrace{\exp\left(-\frac{\sqrt[q]{p}}{8n\|g\|_1^2}\right)}_{\text{Term 6}}\bigg)+\underbrace{\frac{4}{\sqrt[m]{n}}}_{\text{Term 7}}.\label{eq.main_conclusion}
\end{align}
\end{theorem}

To interpret Theorem~\ref{th.combine}, 
we can first focus on the noiseless case, where $\esmall$ and Term~3 are zero. If we fix $n$ and let $p \to \infty$, then Terms 2, 5, and 6 all approach zero. 
We can then conclude that, in the noiseless and heavily overparameterized setting ($p\to\infty$), the generalization error will converge to a small limiting value (Term~1) that depends only on $n$. Further, this limiting value (Term~1) will converge to zero (so do Terms~4 and 7) as $n\to\infty$, i.e., when there are sufficiently many training samples. Finally, Theorem~\ref{th.combine} holds even when there is noise.

The parameters of $q$ and $m$ can be tuned to make Eq.~\eqref{eq.main_conclusion} sharper when $n$ and $p$ are large. For example, as we increase $q$, Term 1 will approach $n^{-0.5}$. Although a larger $q$ makes Terms 4, 5, and 6 bigger, as long as $n$ and $p$ are sufficiently large, those terms will still be close to 0. Similarly, if we increase $m$, then $J_m(n,d)$ will approach the order of $n^{2(d-1)}$. As a result, Term~3 approaches the order of $n^{2d-0.5}$ times $\|\esmall\|_2$ and the requirement $p\geq 6 J_m(n,d)\ln\left(4n^{1+\frac{1}{m}}\right)$ approaches the order of $n^{2(d-1)}\ln n$.

\begin{remark}\label{remark.main}
We note that \cite{arora2019fine} shows that, for two-layer neural networks whose bottom-layer weights are trained by gradient descent, the generalization error for sufficiently large $p$ has the following upper bound: for any $\zeta>0$,
\begin{align}\label{eq.arora}
    &\prob\bigg\{\expectation_{\xsmall}|\hatf(\xsmall)-\f(\xsmall)|\leq \sqrt{\frac{2\ysmall^T(\mathbf{H}^\infty)^{-1}\ysmall}{n}}\nonumber\\
    &+O\bigg(\sqrt{\frac{\log \frac{n}{\zeta\cdot\min\eig (\mathbf{H}^{\infty})}}{n}}\bigg)\bigg\}\geq 1-\zeta,
\end{align}
where $\mathbf{H}^\infty=\lim\limits_{p\to\infty}(\Ht\Ht^T/p)\in\mathds{R}^{n\times n}$.
For certain class of learnable functions (we will compare them with our $\learnableSet$ in Section~\ref{sec.feature_of_set}), the quantity $\ysmall^T(\mathbf{H}^\infty)^{-1}\ysmall$ is bounded. Thus, $\sqrt{\frac{2\ysmall^T(\mathbf{H}^\infty)^{-1}\ysmall}{n}}$ also decreases at the speed $1/\sqrt{n}$. The second $O(\cdot)$-term in Eq.~\eqref{eq.arora}  contains the minimum eigenvalue of $\mathbf{H}^\infty$, which decreases with $n$. (Indeed,  we show that this minimum eigenvalue is upper bounded by $O(n^{-\frac{1}{d-1}})$ in Supplementary Material, Appendix~\ref{app.upper_bound_eig}.) 
Thus, Eq.~\eqref{eq.arora} may decrease a little bit slower than $1/\sqrt{n}$, which is consistent with Term~1 in Eq.~\eqref{eq.main_conclusion} (when $q$ is large).
Note that the term  $2\ysmall^T(\mathbf{H}^\infty)^{-1}\ysmall$ in Eq.~\eqref{eq.arora} captures how the complexity of the ground-truth function affects the generalization error. Similarly, the norm of $g(\cdot)$ also captures the impact\footnote{Although Term~1 in Eq.~\eqref{eq.main_conclusion} in its current form does not depend on $g(\cdot)$, it is possible to modify our proof so that the norm of $g(\cdot)$ also enters Term~1.} of the complexity of the ground-truth function in Eq.~\eqref{eq.main_conclusion}.
However, we caution that the GD solution in \cite{arora2019fine} is based on the original neural network, which is usually different from our min $\ell_2$-norm solution based on the NTK model (even though they are close for very large $p$). Thus, the two results may not be directly comparable.
\end{remark}

Theorem~\ref{th.combine} reveals several important insights on the generalization performance when the ground-truth function belongs to $\learnableSet$.







\textbf{(i) Descent in the overparameterized region:} When $p$ increases, both sides of Eq.~\eqref{eq.main_conclusion} decreases, suggesting that the test error of the overfitted NTK model decreases with $p$.
In Fig.~\ref{fig.compare}(a), we choose a ground-truth function in $\learnableSet$ (we will explain why this function is in $\learnableSet$ later in Section~\ref{sec.feature_of_set}). The test MSE of the aforementioned NTK model (green curve) confirms the overall trend\footnote{This curve oscillates at the early stage when $p$ is small. We suspect it is because, at small $p$, the convergence in Eq.~\eqref{eq.large_p_converge} has not occurred yet, and thus the randomness in $\Vzeroj$ makes the simulation results more volatile.} of descent in the overparameterized region. We note that while  \cite{arora2019fine} provides a generalization error upper-bound for large $p$ (i.e., Eq.~\eqref{eq.arora}), the upper bound there does not capture the dependency in $p$ and thus does not predict this descent.

More importantly, we note a significant difference between the descent in Theorem~\ref{th.combine} and that of min $\ell_2$-norm overfitting solutions for linear models with simple features \citep{belkin2018understand, belkin2019two,bartlett2019benign,	hastie2019surprises,muthukumar2019harmless,liao2020random,jacot2020implicit}. For example, for linear models with Gaussian features, we can obtain (see, e.g., Theorem~2 of \cite{belkin2019two}):
\begin{align}\label{eq.Gaussian}
    \text{MSE}=\|\f\|_2^2 \left(1-\frac{n}{p}\right) 
		+ \frac{\sigma^2 n}{p-n-1},\text{ for $p \ge n+2$}
\end{align}
where $\sigma^2$ denotes the variance of the noise. If we let $p\to\infty$ in Eq.~\eqref{eq.Gaussian}, we can see that the MSE quickly approaches $\|\f\|_2^2$, which is referred to as the ``null risk'' \citep{hastie2019surprises}, i.e., the MSE of a model that predicts zero. Note that the null-risk is at the level of the signal, and thus is quite large.
In contrast, as $p\to\infty$, the test error of the NTK model converges to a value determined by $n$ and $\esmall$ (and is independent of the null risk).
This difference is confirmed in Fig.~\ref{fig.compare}(a), where the test MSE for the NTK model (green curve) is much lower than the null risk (the dashed line) when $p\to\infty$, while both the min $\ell_2$-norm (the red curve) and  the min $\ell_1$-norm solutions (the blue curve) \cite{ju2020overfitting} with Fourier features rise to the null risk when $p\to\infty$. Finally, note that the descent in Theorem~\ref{th.combine} requires $p$ to increase much faster than $n$. Specifically, to keep Term~2 in Eq.~\eqref{eq.main_conclusion} small,
it suffices to let $p$ increase a little bit faster than $\Omega(n^{4d-1})$. This is again quite different from the descent shown in Eq.~\eqref{eq.Gaussian} and in other related work using Fourier and Gaussian features \citep{liao2020random, jacot2020implicit}, where $p$ only needs to grow
proportionally with $n$.

\textbf{(ii) Speed of the descent:} 
Since Theorem~\ref{th.combine} holds for finite $p$, it also characterizes the speed of descent. In particular, Term 2 is proportional to $p^{-\frac{1}{2}\left(1-\frac{1}{q}\right)}$, which approaches $1/\sqrt{p}$ when $q$ is large. 
Again, such a speed of descent is not captured in \cite{arora2019fine}.
As we show in Fig.~\ref{fig.compare}(a), the test error of the gradient descent solution under the original neural network (cyan curve) is usually quite close to that of the NTK model (green curve). Thus, our result provides useful guidance on how fast the generalization error descends with $p$ for such neural networks.

\begin{figure}[t!]
    \centering
    \includegraphics[width=0.4\textwidth]{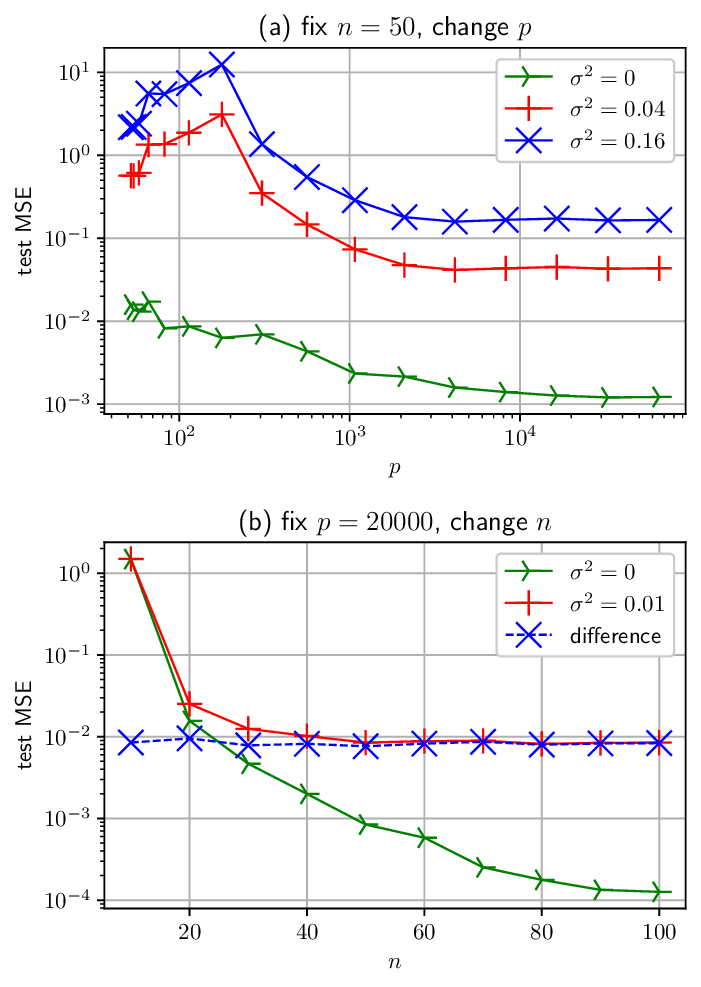}
    \caption{The test MSE of the overfitted NTK model for the same ground-truth function as Fig.~\ref{fig.compare}(a). \textbf{(a)} We fix $n=50$ and increase $p$ for different noise level $\sigma^2$. \textbf{(b)} We fix $p=20000$ and increase $n$. All data points in this figure are the average of five random simulation runs.}
    \label{fig.noise_effect}
\end{figure}

\textbf{(iii) The effect of noise:} 
Term 3 in Eq.~\eqref{eq.main_conclusion} characterizes the impact of the noise $\esmall$, which does not decrease or increase with $p$.
Notice that this is again very different from Eq.~\eqref{eq.Gaussian}, i.e., results of min $\ell_2$-norm overfitting solutions for simple features, where the noise term $\frac{\sigma^2 n}{p-n-1}\to 0$ when $p\to\infty$.
We use Fig.~\ref{fig.noise_effect}(a) to validate this insight.
In Fig.~\ref{fig.noise_effect}(a), we fix $n=50$ and plot curves of test MSE of NTK overfitting solution as $p$ increases. We let the noise $\epsilon_i$ in the $i$-th training sample be \emph{i.i.d.} Gaussian with zero mean and variance $\sigma^2$. The green, red, and blue curves in Fig.~\ref{fig.noise_effect}(a) corresponds to the situation $\sigma^2=0$, $\sigma^2=0.04$, and $\sigma^2=0.16$, respectively. We can see that all three curves become flat when $p$ is very large, and this phenomenon implies that the gap across different noise levels does not decrease when $p\to\infty$, which is in contrast to Eq.~\eqref{eq.Gaussian}.

In Fig.~\ref{fig.noise_effect}(b), we instead fix $p=20000$, and increase $n$). We plot the test MSE both for the noiseless setting (green curve) and for $\sigma^2=0.01$ (red curve). The difference between the two curves (dashed blue curve) then captures the impact of noise, which is related to Term~3 in Eq.~\eqref{eq.main_conclusion}. Somewhat surprisingly, we find that the dashed blue curve is insensitive to $n$, which suggests that Term~3 in Eq.~\eqref{eq.main_conclusion} may have room for improvement.



In summary, we have shown that any ground-truth function in $\learnableSet$ leads to low generalization error for overfitted NTK models. It is then natural to ask what happens if the ground-truth function is not in $\learnableSet$.
Let $\overline{\learnableSet}$ denote the closure\footnote{We consider the normed space of all functions in  $L^2(\sd\mapsto\mathds{R})$. Notice that although $g(\bm{z})$ in Definition~\ref{def.learnableSet} may not be in $L^2$, $\f_g$ is always in $L^2$. Specifically, $\f_g(\xsmall)$ is bounded for every $\xsmall\in\sd$ when $\|g\|_1<\infty$.} of $\learnableSet$, and $D(\f,\learnableSet)$ denotes the $L^2$-distance between $f$ and $\learnableSet$ (i.e., the infimum of the $L^2$-distance from $f$ to every function in $\learnableSet$).

\begin{proposition}\label{th.equal_set}
\textbf{(i)} For any given $(\XX, \ysmall)$, there exists a function $\flinf\in \learnableSet$ such that, uniformly over all $\xsmall\in\sd$,  $\fl(\xsmall)\stackrel{\text{P}}{\rightarrow}\flinf(\xsmall)$ as $p\to\infty$.
\textbf{(ii)}
Consequently, if the ground-truth function $\f\notin \overline{\learnableSet}$ (or equivalently, $D(\f,\learnableSet)>0$), then the MSE of $\flinf$ (with respect to the ground-truth function $\f$) is at least $D(\f,\learnableSet)$.
\end{proposition}

Intuitively, Proposition~\ref{th.equal_set} (proven in Supplementary Material Appendix~\ref{ap.equal_set}) suggests that, if a ground-truth function is outside the closure of $\learnableSet$, then no matter how large $n$ is,
the test error of a NTK model with infinitely many neurons cannot be small (regardless whether or not the training samples contain noise). 
We validate this in Fig.~\ref{fig.compare}(b), where a ground-truth function is chosen outside $\overline{\learnableSet}$.  
The test MSE of NTK overfitting solutions (green curve) is above null risk (dashed black line) and thus is much higher compared with Fig.~\ref{fig.compare}(a). We also plot the test MSE of the GD solution of the  real neural network (cyan curve), which seems to show the same trend.

Comparing Theorem~\ref{th.combine} and Proposition~\ref{th.equal_set}, we can clearly see 
that, all functions in $\learnableSet$ are learnable by the overfitted NTK model, and all functions not in $\overline{\learnableSet}$ are not.





\section{What Exactly are the Functions in \texorpdfstring{$\learnableSet$}{the learnable set}?}\label{sec.feature_of_set}

Our expression for learnable functions in Definition~\ref{def.learnableSet} is still in an indirect form, i.e., through the unknown function $g(\cdot)$. In \cite{arora2019fine}, the authors show that all functions of the form $(\xsmall^T\bm{a})^{l}$, $l\in\{0,1,2,4,6,\cdots\}$ are learnable by GD (assuming large $p$ and small step size), for a similar 2-layer network with ReLU activation that has no bias. In the following, we will show that our learnable functions in Definition~\ref{def.learnableSet} also have a similar form. Further, we can show that any functions of the form $(\xsmall^T\bm{a})^{l}$, $l\in\{3,5,7,\cdots\}$ are not learnable. Our characterization uses an interesting connection to harmonics and filtering on $\sd$, which may be of independent interest.


Towards this end, we first note that the integral form in Definition~\ref{def.learnableSet} can be viewed as a convolution on $\sd$ (denoted by $\circledast$). Specifically, for any $\fg\in\learnableSet$, we can rewrite it as
\begin{align}
    &\fg(\xsmall) =g\circledast \hf(\xsmall)\defeq\int_{\mathsf{SO}(d)}g(\mathbf{S}\bm{e})\hf(\mathbf{S}^{-1}\xsmall)d\mathbf{S},\label{eq.convolution}\\
    &\hf(\xsmall)\defeq \xsmall^T\bm{e}\frac{\pi-\arccos (\xsmall^T\bm{e})}{2\pi},\label{eq.h_in_convolution}
\end{align}
where $\bm{e}\defeq [0\ 0\ \cdots\ 0\ 1]^T\in\mathds{R}^d$,
and $\mathbf{S}$ is a $d\times d$ orthogonal matrix that denotes a rotation in $\sd$, chosen from the set $\mathsf{SO}(d)$ of all rotations.
An important property of the convolution Eq.~\eqref{eq.convolution} is that it corresponds to multiplication in the frequency domain, similar to Fourier coefficients. To define such a transformation to the frequency domain, we use a set of  hyper-spherical harmonics $\Xi_{\mathbf{K}}^l$ \citep{vilenkin1968special, dokmanic2009convolution} when $d\geq 3$, which forms an orthonormal basis for functions on $\sd$. These harmonics are indexed by $l$ and $\mathbf{K}$, where $\mathbf{K}=(k_1,k_2,\cdots,k_{d-2})$ and $l=k_0\geq k_1\geq k_2\geq \cdots \geq k_{d-2}\geq 0$ (those $k_i$'s and $l$ are all non-negative integers). Any function $f\in L^2(\sd\mapsto\mathds{R})$ (including even  $\delta$-functions \citep{li2013integral}) can be decomposed uniquely into these harmonics, i.e., $\f(\xsmall)=\sum_{l}\sum_{\mathbf{K}}c_{\f}(l,\mathbf{K})\Xi_{\mathbf{K}}^l(\xsmall)$,
where $c_{\f}(\cdot,\cdot)$ are projections of $\f$ onto the basis function.
In Eq.~\eqref{eq.convolution}, let $c_g(\cdot,\cdot)$ and $c_h(\cdot,\cdot)$ denote the coefficients corresponding to the decompositions of $g$ and $h$, respectively. Then, we must have \citep{dokmanic2009convolution}
\begin{align}\label{eq.freq_product}
    c_{\fg}(l,\mathbf{K})=\Lambda\cdot c_{g}(l,\mathbf{K})c_{\hf}(l,\bm{0}),
\end{align}
where $\Lambda$ is some normalization constant. 
Notice that in Eq.~\eqref{eq.freq_product}, the coefficient for $h$ is $c_h(l,\bm{0})$ instead of $c_h(l,\mathbf{K})$, which is due to the intrinsic rotational symmetry of such convolution \citep{dokmanic2009convolution}.

The above decomposition has an interesting ``filtering'' interpretation as follows. We can regard the function $\hf$ as a ``filter'' or ``channel,'' while the function $g$ as a transmitted ``signal.'' Then, the function $\fg$ in Eq.~\eqref{eq.convolution} and Eq.~\eqref{eq.freq_product} can be regarded as the received signal after $g$ goes through the channel/filter $h$. Therefore, when coefficient $c_h(l,\bm{0})$ of $\hf$ is non-zero, then the corresponding coefficient $c_{\fg}(l,\mathbf{K})$ for $\fg$ can be any value (because we can arbitrarily choose $g$). In contrast, if a coefficient $c_h(l,\bm{0})$ of $\hf$ is zero, then the corresponding coefficient $c_{\fg}(l,\mathbf{K})$ for $\fg$ must also be zero for all $\mathbf{K}$.

Ideally, if $\hf$ contains all ``frequencies,'' i.e., all coefficients $c_h(l,\bm{0})$ are non-zero, then $\fg$ can also contain all ``frequencies,'' which means that $\learnableSet$ can contain almost all functions. Unfortunately, this is not true for the function $h$ given in Eq.~\eqref{eq.h_in_convolution}. 
Specifically, using the harmonics defined in \cite{dokmanic2009convolution}, the basis $\Xi_{\bm{0}}^l$ for $(l,\bm{0})$ turns out to have the form
\begin{align}\label{eq.K_zero_form}
    \Xi_{\bm{0}}^l(\xsmall)=\sum_{k=0}^{\lfloor\frac{l}{2}\rfloor}(-1)^k\cdot a_{l,k}\cdot (\xsmall^T\bm{e})^{l-2k},
\end{align}
where $a_{l,k}$ are positive constants. 
Note that the expression Eq.~\eqref{eq.K_zero_form} contains either only even powers of $\xsmall^T\bm{e}$ (if $l$ is even) or odd powers of $\xsmall^T\bm{e}$ (if $l$ is odd). Then, for the function $h$ in Eq.~\eqref{eq.h_in_convolution},
we have the following proposition (proven in Supplementary Material, Appendix~\ref{app.zero_component}). We note that \citep{basri2019convergence} has a similar harmonics
analysis, where the  expression of $c_h(l,\bm{0})$ is given. However, it is not obvious that   the expression of $c_h(l,\bm{0})$ for all $l=0,1,2,4,6,\cdots$ given in \citep{basri2019convergence} must be non-zero, which is made clear by Proposition~\ref{prop.coefficients} as follows.
\begin{proposition}\label{prop.coefficients}
$c_h(l,\bm{0})$ is zero for $l=3,5,7,\cdots$ and is non-zero for $l=0,1,2,4,6,\cdots$.
\end{proposition}


We are now ready to characterize what functions are in $\learnableSet$. By the form of Eq.~\eqref{eq.K_zero_form}, for any non-negative integer $k$, any even power $(\xsmall^T\bm{e})^{2k}$ is a linear combination of $\Xi_{\bm{0}}^0,\Xi_{\bm{0}}^2,\cdots,\Xi_{\bm{0}}^{2k}$, and any odd power $(\xsmall^T\bm{e})^{2k+1}$ is a linear combination of $\Xi_{\bm{0}}^1,\Xi_{\bm{0}}^3,\cdots,\Xi_{\bm{0}}^{2k+1}$. 
By Proposition~\ref{prop.coefficients}, we thus conclude that any function $f_g(\xsmall)=(\xsmall^T\bm{e})^l$ where $l\in\{0,1,2,4,6,\cdots\}$ can be written in the form of Eq.~\eqref{eq.freq_product} in the frequency domain, and thus are in $\learnableSet$. In contrast, any function $f(\xsmall)=(\xsmall^T\bm{e})^l$ where $l\in\{3, 5, 7,\cdots\}$ cannot be written in the form of Eq.~\eqref{eq.freq_product}, and are thus not in $\learnableSet$. Further, the $\ell_2$-norm of any latter function will also be equal to its distance to $\flinfSet$. Therefore, the generalization-error lower-bound in Proposition~\ref{th.equal_set} will apply (with $D(\f,\learnableSet)=\|\f\|_2$).
Finally, by Eq.~\eqref{eq.convolution}, $\learnableSet$ is invariant under rotation and finite linear summation. Therefore, 
any finite sum of $(\xsmall^T\bm{a})^l$, $l=0,1,2,4,6,\cdots$ must also belong to $\learnableSet$.

For the special case of $d=2$, the input $\xsmall$ corresponds to an angle $\theta\in[-\pi,\pi]$, and the above-mentioned harmonics become Fourier series $\sin(k\theta)$ and $\cos(k\theta)$, $k=0,1,\cdots$. We can then get similar results that frequencies of $k\in\{0,1,2,4,6,\cdots\}$ are learnable (while others are not), which explains the learnable and not-learnable functions in Fig.~\ref{fig.compare}. Details can be found in Supplementary Material, Appendix~\ref{app.special_case_d2}.


\begin{remark}\label{remark.bias_enlarge}
We caution that the above claim on non-learnable functions critically depends on the network architecture. That is, we assume throughout this paper that the ReLU activation has no bias. It is known from an expressiveness point of view that, using ReLU without bias, a shallow network can only approximate the sum of linear functions and even functions \citep{ghorbani2019linearized}. Thus, it is not surprising that other odd-power (but non-linear) polynomials cannot be learned. In contrast, by adding a bias, a shallow network using ReLU becomes a universal approximator \citep{ji2019neural}. The recent work of \cite{Srikant21} shows that polynomials with all powers can be learned by the corresponding 2-layer NTK model. These results are consistent with ours because a ReLU activation function operating on $\tilde{\xsmall} \in \mathds{R}^{d-1}$ with a bias can be equivalently viewed as one operating on a $d$-dimension input (with the last-dimension being fixed at $1/\sqrt{d}$) but with no bias. Even though only a subset of functions are learnable in the $d$-dimension space, when projected into a $(d-1)$-dimension subspace, they may already span all functions. For example, one could write $(\xsmall^T \bm{a})^3$ as a linear combination of ($\begin{bsmallmatrix}\tilde{\xsmall}\\1/\sqrt{d}\end{bsmallmatrix}^T \bm{b}_i)^{l_i}$, where $i\in\{1,2,\cdots, 5\}$, $[l_1,\cdots,l_5]=[4,4,2,1,0]$, and $\bm{b}_i\in\mathds{R}^d$ depends only on $\bm{a}$. (See Supplementary Material, Appendix~\ref{proof.bias_enlarge} for details.) It remains an interesting question whether similar difference arises for other network architectures (e.g., with more than 2 layers).


\end{remark}

\section{Proof Sketch of Theorem~\ref{th.combine}}\label{sec.proof_combine}
In this section, we sketch the key steps to prove Theorem~\ref{th.combine}.
Starting from Eq.~\eqref{eq.DVl}, we have
\begin{align}\label{eq.temp_110903}
    \DVl =\Ht^T(\Ht\Ht^T)^{-1}\left(\F(\XX)+\esmall\right).
\end{align}
For the learned model $\fl(\xsmall)=\hx \DVl$ given in Eq.~\eqref{def.fl}, the error for any test input $\xsmall$ is then
\begin{align}\label{eq.bias_and_variance_decomposition}
    \fl(\xsmall)-\f(\xsmall)=&\left(\hx\Ht^T(\Ht\Ht^T)^{-1}\F(\XX)-\f(\xsmall)\right)\nonumber\\
    &+\hx\Ht^T(\Ht\Ht^T)^{-1}\esmall.
\end{align}
In the classical ``bias-variance'' analysis with respect to MSE \citep{belkin2018reconciling}, the first term on the right-hand-side of Eq.~\eqref{eq.bias_and_variance_decomposition} contributes to the bias and the second term contributes to the variance. We first quantify the second term (i.e., the variance) in the following proposition.
\begin{proposition}\label{prop.noise_large_p}
For any $n\geq 2$, $m\in\left[1,\ \frac{\ln n}{\ln \frac{\pi}{2}}\right]$, $d\leq n^4$, if $p\geq 6J_m(n,d)\ln\left(4n^{1+\frac{1}{m}}\right)$,
we must have $\prob\limits_{\XX,\Vzero}\big\{|\hx\Ht^T(\Ht\Ht^T)^{-1}\esmall|\leq \sqrt{J_m(n,d)n}\|\esmall\|_2,\text{ for all }\esmall\in \mathds{R}^n\big\}
    \geq 1-\frac{2}{\sqrt[m]{n}}$.
\end{proposition}
The proof is in Supplementary Material  Appendix~\ref{app.noise_large_p}. Proposition~\ref{prop.noise_large_p} implies that, for fixed $n$ and $d$, when $p\to\infty$, with high probability the variance will not exceed a certain factor of the noise $\|\esmall\|_2$. In other words, the variance will not go to infinity when $p\to\infty$.
The main step in the proof is to lower bound $\min\eig\left(\Ht\Ht^T\right)/p$, which is given by $1/(J_m(n,d)n)$. 
Note that this is the main place where we used the assumption that $\xsmall$ is uniformly distributed. We expect that our main proof techniques can be generalized to other distributions (with a different expression of $J_m(n,d)$), which we leave for future work.

\begin{remark}\label{remark.h_inf}
In the upper bound in \cite{arora2019fine} (i.e., Eq.~\eqref{eq.arora}), any noise added to $\ysmall$ will at least contribute to the generalization upper bound Eq.~\eqref{eq.arora}  by a positive term $\esmall^T(\mathbf{H}^{\infty})^{-1}\esmall/n$. Thus, their upper bound may also grow as $\min\eig (\mathbf{H}^{\infty})$ decreases. One of the contribution of Proposition~\ref{prop.noise_large_p} is to characterize this minimum eigenvalue.
\end{remark}


We now bound the bias part. We first study the class of ground-truth functions that can be learned with fixed $\Vzero$. We refer to them as \emph{pseudo ground-truth}, to differentiate them with the set $\learnableSet$ of learnable functions for random $\Vzero$. They are defined with respect to the same $g(\cdot)$ function, so that we can later extend to the ``real'' ground-truth functions in $\learnableSet$ when considering the randomness of $\Vzero$.


\begin{definition}\label{def.fv}
Given $\Vzero$, for any learnable ground-truth function $\fg\in\learnableSet$ with the corresponding function $g(\cdot)$, define the corresponding \textbf{pseudo ground-truth} as
\begin{align*}
    \fv(\xsmall)\defeq\int_{\sd}\xsmall^T\bm{z}\frac{|\Czx|}{p}g(\bm{z}) d\mu(\bm{z}).
\end{align*}
\end{definition}
The reason that this class of functions may be the learnable functions for fixed $\Vzero$ is similar to the discussions in Eq.~\eqref{eq.fl_before_limit} and Eq.~\eqref{eq.card_c}. Indeed, using the same choice of $g(\bm{z})$ in Eq.~\eqref{eq.g_as_delta}, the learned function $\fl$ in Eq.~\eqref{eq.fl_before_limit} at fixed $\Vzero$ is always of the form in Definition~\ref{def.fv}.

The following proposition gives an upper bound of the generalization performance when the data model is based on the pseudo ground-truth and the NTK model uses exactly the same $\Vzero$.

\begin{proposition}\label{th.fixed_Vzero}
Assume fixed $\Vzero$ (thus $p$ and $d$ are also fixed), there is no noise. If the ground-truth function is $\f=\fv$ in Definition~\ref{def.fv} and $\|g\|_\infty<\infty$, then for any $\xsmall\in\sd$ and $q\in[1,\ \infty)$, we have $\prob_{\XX}\big\{|\fl(\xsmall)-f(\xsmall)|\leq n^{-\frac{1}{2}\left(1-\frac{1}{q}\right)}\big\} \geq1-2e^2\exp\left(-\frac{\sqrt[q]{n}}{8\|g\|_\infty^2}\right)$.
\end{proposition}
The proof is in Supplementary Material, Appendix~\ref{app.proof_fixed_Vzero}. Note that both the threshold of the probability event and the upper bound coincide with Term~1 and Term~4, respectively, in Eq.~\eqref{eq.main_conclusion}.
Here we sketch the proof of Proposition~\ref{th.fixed_Vzero}. Based on the definition of the pseudo ground-truth, we can rewrite $\fv$ as $\fv(\xsmall)=\hx \DVs$, where $\DVs\in\mathds{R}^{dp}$ is given by, for all $j\in\{1,2,\cdots,p\}$, $\DVs[j]=\int_{ \sd}\bm{1}_{\{\bm{z}^T\Vzeroj>0\}}\bm{z}\frac{g(\bm{z})}{p}d\mu(\bm{z})$.
From Eq.~\eqref{eq.DVl} and Eq.~\eqref{def.fl}, we can see that the learned model is $\fl(\xsmall)=\hx\PH\DVs$ where $\PH\defeq \Ht^T(\Ht\Ht^T)^{-1}\Ht$. Note that $\PH$ is an orthogonal projection to the row-space of $\Ht$.
Further, it is easy to show that $\|\hx\|_2\leq \sqrt{p}$. Thus, we have $|\fl(\xsmall)-\fv(\xsmall)|=|\hx(\PH-\mathbf{I})\DVs|\leq \sqrt{p}\|(\PH-\mathbf{I})\DVs\|_2$. The term $(\PH-\mathbf{I})\DVs$ can be interpreted as the distance from $\DVs$ to the row-space of $\Ht$. Note that this distance is no greater than the distance between $\DVs$ and any point in the row-space of $\Ht$. Thus, in order to get an upper bound on $\|(\PH-\mathbf{I})\DVs\|_2$, we only need to find a vector $\bm{a}\in \mathds{R}^n$ that makes $\|\DVs-\Ht^T\bm{a}\|_2$ as small as possible, especially when $n$ is large. 
Our proof uses the vector $\bm{a}$ such that its $i$-th element is $\bm{a}_i\defeq \frac{g(\Xii)}{np}$. See Supplementary Material, Appendix~\ref{app.proof_fixed_Vzero} for the rest of the details.


The final step is to allow $\Vzero$ to be random.
Given any random $\Vzero$, any function $\fg\in\learnableSet$ can be viewed as the summation of a pseudo ground-truth function (with the same $g(\cdot)$) and a difference term. 
This difference can be viewed as a special form of ``noise'', and thus we can use Proposition~\ref{prop.noise_large_p} to quantify its impact. Further, the magnitude of this ``noise'' should decrease with $p$ (because of Eq.~\eqref{eq.large_p_converge}). Combining this argument with Proposition~\ref{th.fixed_Vzero}, we can then prove Theorem~\ref{th.combine}. See Supplementary Material, Appendix~\ref{app.proof_large_p} for details.
\section{Conclusions}
In this paper, we studied the generalization performance of the min $\ell_2$-norm overfitting solution for a two-layer NTK model. We provide a precise characterization of the learnable ground-truth functions for such models, by providing a generalization upper bound for  all functions in $\learnableSet$, and a generalization lower bound for all functions not in $\overline{\learnableSet}$. We show that,
while the test error of the overfitted NTK model also exhibits descent in the overparameterized regime, the descent behavior can be quite different from the double descent of linear models with simple features. 

There are several interesting directions for future work. First, based on Fig.~\ref{fig.noise_effect}(b), our estimation of the effect of noise could be further improved. Second, 
it would be interesting to explore whether the methodology can be extended to NTK model for other neural networks, e.g., with different activation functions and with more than two layers. 
\section*{Acknowledgements}

This work is partially supported by an NSF sub-award via Duke University (IIS-1932630), by NSF grants CNS-1717493, CNS-1901057, and CNS-2007231, and by  Office of Naval Research under Grant N00014-17-1-241.
The authors would like to thank Professor R. Srikant at the University of Illinois at Urbana-Champaign and anonymous reviewers for their valuable comments and suggestions.

\bibliographystyle{icml2021}
\bibliography{ref}
\newpage
\onecolumn
\appendix
\section{Extra Notations}\label{app.extra_notations}
In addition to the notations that we have introduced in the main body of this paper, we need some extra notations that are used in the following appendices.
The distribution of the initial weights $\Vzeroj$ is denoted by the probability density $\vdensity(\cdot)$ on $\mathds{R}^{d}$, and the directions of the initial weights (i.e., the normalized initial weights $\frac{\Vzeroj}{\|\Vzeroj\|_2}$) follows the probability density $\nvdensity(\cdot)$ on $\sd$. Let $\lambda_{a}(\cdot)$ be the Lebesgue measure on $\mathds{R}^{a}$ where the dimension $a$ can be, e.g., $(d-1)$ and $(d-2)$.

Let $\bino(a,b)$ denote the binomial distribution, where $a$ is the number of trials and $b$ is the success probability. Let $I_{\cdot}(\cdot,\cdot)$ denote the regularized incomplete beta function \cite{dutka1981incomplete}. 
Let $B(\cdot, \cdot)$ denote the beta function \cite{chaudhry1997extension}. Specifically,
\begin{align}
    &B(x,y)\defeq\int_0^1 t^{x-1}(1-t)^{y-1}dt,\label{eq.def_betaFunction}\\
    &I_x(a,b)\defeq\frac{\int_0^x t^{a-1}(1-t)^{b-1}dt}{B(a,b)}.\label{eq.def_reg_incomplete_beta}
\end{align}
Define a cap on a unit hyper-sphere $\sd$ as the intersection of $\sd$ with an open ball in $\mathds{R}^d$ centered at $\bm{v}_*$ with radius $r$, i.e.,
\begin{align}\label{eq.temp_100101}
    \capsr\defeq \left\{\bm{v}\in\sd\ |\ \|\bm{v}-\bm{v}_*\|_2< r\right\}.
\end{align}
\begin{remark}\label{remark.no_star}
For ease of exposition, we will sometimes neglect the subscript $\bm{v}_*$ of $\capsr$ and use $\capr$ instead, when the quantity that we are estimating only depends on $r$ but not $\bm{v}_*$. For example, where we are interested in the area of $\capsr$, it only depends on $r$ but not $\bm{v}_*$. Thus, we write $\lambda_{d-1}(\capr)$ instead.
\end{remark}

For any $\xsmall\in \mathds{R}^d$ such that $\xsmall^T\bm{v}_*=0$, define two halves of the cap $\capsr$ as
\begin{align}\label{eq.temp_110205}
    \capsrpx\defeq\left\{\bm{v}\in\capsr\ |\ \xsmall^T\bm{v}>0\right\},\quad \capsrnx\defeq\left\{\bm{v}\in\capsr\ |\ \xsmall^T\bm{v}<0\right\}.
\end{align}
Define the set of directions of the initial weights $\Vzero[j]$'s as
\begin{align}\label{eq.temp_110204}
    \mathcal{A}_{\Vzero}\defeq\left\{\frac{\Vzeroj}{\|\Vzeroj\|_2}\ \Bigg|\ j\in\{1,2,\cdots,p\}\right\}.
\end{align}
\section{GD (gradient descent) Converges to Min \texorpdfstring{$\ell_2$}{l2}-Norm Solutions}\label{ap.GD_same_l2}
We assume that the GD algorithm for minimizing the training MSE is given by
\begin{align}\label{eq.temp_102901}
    \DVg_{k+1}=\DVg_k-\gamma_k\sum_{i=1}^n(\Hti\DVg_k-\yii)\Hti^T,
\end{align}
where $\DVgk$ denotes the solution in the $k$-th GD iteration ($\DVg_0=\bm{0}$), and $\gamma_k$ denotes the step size of the $k$-th iteration.
\begin{lemma}\label{le.GD_same_l2}
If $\DVl$ exists and GD in Eq.~\eqref{eq.temp_102901} converges to zero-training loss (i.e., $\Ht\DVg_{\infty}=\ysmall$), then $\DVg_{\infty}=\DVl$.
\end{lemma}
\begin{proof}
Because $\DVg_0=\bm{0}$ and Eq.~\eqref{eq.temp_102901}, we know that $\DVgk$ is in the row space of $\Ht$ for any $k$. Thus, we can let $\DVg_{\infty}=\Ht^T\bm{a}$ where $\bm{a}\in \mathds{R}^{n}$. When GD converges to zero training loss, we have $\Ht\DVg_{\infty}=\ysmall$. Thus, we have $\Ht\Ht^T\bm{a}=\ysmall$, which implies $\bm{a}=(\Ht\Ht^T)^{-1}\ysmall$. Therefore, we must have $\DVg_{\infty}=\Ht^T\bm{a}=\Ht^T(\Ht\Ht^T)^{-1}\ysmall=\DVl$.
\end{proof}
\section{Assumptions and Justifications}\label{ap.justify_assump}
Because $\hatf_{\DV, \Vzero}(a\xsmall)=a\cdot\hatf_{\DV, \Vzero}(\xsmall)$ for any $a\in\mathds{R}$, we can always do preprocessing to normalize the input $\xsmall$. For simplicity, we focus on the simplest situation that the randomness for the inputs and the initial weights are uniform. Nonetheless, methods and results of this paper can be readily generalized to other continuous random variable distributions, which we leave for future work. We thus make the following Assumption~\ref{as.normalize}.
\begin{assumption}\label{as.normalize}\label{as.uniform}
The input $\xsmall$ are uniformly distributed in $\sd$. The initial weights $\Vzeroj$'s are uniform in all directions. In other words, $\xdensity(\cdot)$ and $\nvdensity(\cdot)$ are both $\mathsf{unif}(\sd)$.
\end{assumption}
We study the overparameterized and overfitted setting, so in this paper we always assume $p\geq n/d$, i.e., the number of parameters $pd$ is larger than or equal to the number of training samples $n$. The situation of $d=1$ is relatively trivial, so we only consider the case $d\geq 2$.
We then make Assumption~\ref{as.pnd}.
\begin{assumption}\label{as.pnd}
$p\geq n/d$ and $d\geq 2$.
\end{assumption}
If the input is a continuous random vector, then for any $i\neq j$, we have $\prob\{\Xii=\Xjj\}=0$ and $\prob\{\Xii=- \Xjj\}=0$ (because the probability that a continuous random variable equals to a given value is zero). Thus, $\prob\{\Xii\parallel\Xjj\}=0$, and $\prob\{\Xii\nparallel\Xjj\}=1$. Similarly, we can show that $\prob\{\Vzerok\nparallel\Vzerol\}=1$. We thus make Assumption~\ref{as.nparallel}.
\begin{assumption}\label{as.nparallel}
$\Xii\nparallel \Xjj$ for any $i\neq j$, and $\Vzerok\nparallel \Vzerol$ for any $k\neq l$.
\end{assumption}


With these assumptions, the following lemma says that when $p$ is large enough, with high probability $\Ht$ has full row-rank (and thus $\DVl$ exists).
\begin{lemma}\label{le.full_rank}
$\lim_{p\to\infty}\underset{\Vzero}{\prob} \left\{\rank (\Ht)=n\ |\ \XX\right\}=1$.
\end{lemma}
\begin{proof}
See Appendix~\ref{app.full_rank}.
\end{proof}

\section{Some Useful Supporting Results}\label{app.useful_lemmas}

Here we collect some useful lemmas that are needed for proofs in other appendices, many of which are estimations of certain quantities that we will use later.

\subsection{Quantities related to the area of a cap on a hyper-sphere}
The following lemma is introduced by \cite{li2011concise}, which gives the area of a cap on a hyper-sphere with respect to the colatitude angle.
\begin{lemma}\label{le.original_cap}
Let $\phi\in[0,\ \frac{\pi}{2}]$ denote the colatitude angle of the smaller cap on $\sd$, then the area (in the measure of $\musd$) of this hyper-spherical cap is
\begin{align*}
    \frac{1}{2}\musd(\sd)I_{\sin^2\phi}\left(\frac{d-1}{2},\frac{1}{2}\right).
\end{align*}
\end{lemma}

The following lemma is another representation of the area of the cap with respect to the radius $r$ (recall the definition of $\capr$ in Eq.~\eqref{eq.temp_100101} and Remark~\ref{remark.no_star}).
\begin{lemma}\label{le.cap_area}
If $r\leq \sqrt{2}$, then we have
\begin{align*}
    \musd(\capr)=
    \frac{1}{2}\musd(\sd)I_{r^2(1-\frac{r^2}{4})}\left(\frac{d-1}{2},\frac{1}{2}\right).
\end{align*}
\end{lemma}
\begin{proof}
Let $\phi$ denote the colatitude angle. By the law of cosines, we have
\begin{align*}
    \cos\phi = 1-\frac{r^2}{2}.
\end{align*}
Thus, we have
\begin{align*}
    \sin^2\phi=1-\cos^2\phi=1-\left(1-\frac{r^2}{2}\right)^2=r^2\left(1-\frac{r^2}{4}\right).
\end{align*}
By Lemma~\ref{le.original_cap}, the result of this lemma thus follows. Notice that we require $r\leq \sqrt{2}$ to make sure that $\phi\in[0,\ \frac{\pi}{2}]$, which is required by Lemma~\ref{le.original_cap}.
\end{proof}

The area of a cap can be interpreted as the probability of the event that a uniformly-distributed random vector falls into that cap. We have the following lemma.
\begin{lemma}\label{le.innerProd_I}
Suppose that a random vector $\bm{b}\in\sd$ follows uniform distribution in all directions. Given any $\bm{a}\in\sd$ and for any $c\in(0,1)$, we have
\begin{align*}
    \prob_{\bm{b}}\left\{|\bm{a}^T\bm{b}|>c\right\}=I_{1-c^2}\left(\frac{d-1}{2},\frac{1}{2}\right).
\end{align*}
\end{lemma}
\begin{proof}
Notice that $\left\{\bm{b}\ \big|\  \bm{a}^T\bm{b}>c\right\}$ is a hyper-spherical cap. Define its colatitude angle as $\phi$. We have $\cos\phi=\bm{a}^T\bm{b}=c$. Thus, we have $\sin^2\phi=1-c^2$. By Lemma~\ref{le.original_cap}, we then have
\begin{align*}
    \musd\left(\left\{\bm{b}\ \big|\  \bm{a}^T\bm{b}>c\right\}\right)=\frac{1}{2}\musd(\sd)I_{1-c^2}\left(\frac{d-1}{2},\frac{1}{2}\right).
\end{align*}
Further, by symmetry, we have
\begin{align*}
    \musd\left(\left\{\bm{b}\ \big|\  |\bm{a}^T\bm{b}|>c\right\}\right)=2\musd\left(\left\{\bm{b}\ \big|\  \bm{a}^T\bm{b}>c\right\}\right)=\musd(\sd)I_{1-c^2}\left(\frac{d-1}{2},\frac{1}{2}\right).
\end{align*}
Because $\bm{b}$ follows uniform distribution in all directions, we have
\begin{align*}
    \prob_{\bm{b}}\left\{|\bm{a}^T\bm{b}|>c\right\}=\frac{\musd\left(\left\{\bm{b}\ \big|\  |\bm{a}^T\bm{b}|>c\right\}\right)}{\musd(\sd)}=I_{1-c^2}\left(\frac{d-1}{2},\frac{1}{2}\right).
\end{align*}
\end{proof}

\subsection{Estimation of certain norms}
In this subsection, we will show $\|\hx\|_2\leq \sqrt{p}$ in Lemma~\ref{le.h_p}.  We also upper bound the norm of the product of two matrices by the product of their norms in Lemma~\ref{le.matrix_norm}. At last, Lemma~\ref{le.l2_diff} states that if two vector differ a lot, then the sum of their norm cannot be too small.
\begin{lemma}\label{le.h_p}
$\|\hx\|_2\leq \sqrt{p}$ for any $\xsmall\in \sd$.
\end{lemma}
\begin{proof}
This follows because the input $\xsmall$ is normalized. Specifically, by Eq.~\eqref{eq.def_hx}, we have
\begin{align}\label{eq.temp_100601}
    \|\hx\|_2=\sqrt{\sum_{j=1}^p\left\|\bm{1}_{\{\xsmall^T\Vzeroj>0\}}\cdot\xsmall^T\right\|_2^2}\leq \sqrt{p}.
\end{align}
\end{proof}

\begin{lemma}\label{le.matrix_norm}

If $\mathbf{C}=\mathbf{A}\mathbf{B}$, then $\|\mathbf{C}\|_2\leq \|\mathbf{A}\|_2\cdot\|\mathbf{B}\|_2$. Here $\mathbf{A}$, $\mathbf{B}$, and $\mathbf{C}$ could be scalars, vectors, or matrices.
\end{lemma}
\begin{proof}
This lemma directly follows the definition of matrix norm.
\end{proof}
\begin{remark}
Note that the ($\ell_2$) matrix-norm (i.e., spectral norm) of a vector is exactly its $\ell_2$ vector-norm (i.e., Euclidean norm)\footnote{To see this, consider a (row or column) vector $\bm{a}$. The matrix norm of $\bm{a}$ is \begin{align*}
    &\max_{|x|=1}\|\bm{a}x\|_2\text{ (when $\bm{a}$ is a column vector)},\\
    \text{or }&\max_{\|\xsmall\|_2=1}\|\bm{a}\xsmall\|_2\text{ (when $\bm{a}$ is a row vector)}.
\end{align*}
In both cases, the value of the matrix-norm equals to $\sqrt{\sum a_i^2}$, which is exactly the $\ell_2$-norm (Euclidean norm) of $\bm{a}$.
}. Therefore, when applying Lemma~\ref{le.matrix_norm}, we do not need to worry about whether $\mathbf{A}$, $\mathbf{B}$, and $\mathbf{C}$ are matrices or vectors.
\end{remark}


\begin{lemma}\label{le.l2_diff}
For any $\bm{v}_1,\bm{v}_2\in\mathds{R}^d$, we have
\begin{align*}
    \|\bm{v}_1\|_2^2+\|\bm{v}_2\|_2^2\geq \frac{1}{2}\|\bm{v}_1-\bm{v}_2\|_2^2.
\end{align*}
\end{lemma}
\begin{proof}
It is easy to prove that $\|\cdot\|_2^2$ is convex. Thus, we have
\begin{align*}
    \|\bm{v}_1\|_2^2+\|\bm{v}_2\|_2^2&=\|\bm{v}_1\|_2^2+\|-\bm{v}_2\|_2^2\\
    &\geq 2\left\|\frac{\bm{v}_1-\bm{v}_2}{2}\right\|_2^2\text{ (apply Jensen's inequality on the convex function $\|\cdot\|_2^2$)}\\
    &=\frac{1}{2}\|\bm{v}_1-\bm{v}_2\|_2^2.
\end{align*}
\end{proof}

\subsection{Estimates of certain tail probabilities}
The following is the (restated) Corollary 5 of \cite{goemans2015chernoff}.
\begin{lemma}\label{le.bino}
If the random variable $X$ follows $\bino(a, b)$, then for all $0<\delta<1$, we have
\begin{align*}
    \prob\{|X-ab|>\delta ab\}\leq 2e^{-ab\delta^2/3}.
\end{align*}
\end{lemma}

The following lemma is the (restated) Theorem~1.8 of \cite{hayes2005large}.
\begin{lemma}[Azuma–Hoeffding inequality for random vectors]\label{le.hoeffding}
Let $X_1,X_2,\cdots,X_k$ be \emph{i.i.d.} random vectors with zero mean (of the same dimension) in a real Euclidean space such that $\|X_i\|_2\leq 1$ for all $i=1,2,\cdots,k$. Then, for every $a>0$,
\begin{align*}
    \prob\left\{\left\|\sum_{i=1}^k X_i\right\|_2\geq a\right\}<2e^2\exp\left(-\frac{a^2}{2k}\right).
\end{align*}
\end{lemma}

In the following lemma, we use Azuma–Hoeffding inequality to upper bound the deviation of the empirical mean value of a bounded random vector from its expectation.
\begin{lemma}\label{le.large_deviation}
Let $X_1,X_2,\cdots,X_k$ be \emph{i.i.d.} random vectors (of the same dimension) in a real Euclidean space such that $\|X_i\|_2\leq U$ for all $i=1,2,\cdots,k$. Then, for any $q\in[1,\ \infty)$,
\begin{align*}
    \prob\left\{\left\|\left(\frac{1}{k}\sum_{i=1}^k X_i\right)-\expectation X_1\right\|_2\geq k^{\frac{1}{2q}-\frac{1}{2}}\right\}<2e^2\exp\left(-\frac{\sqrt[q]{k}}{8U^2}\right).
\end{align*}
\end{lemma}
\begin{proof}
Because $\|X_i\|_2\leq U$, we have $\expectation\|X_i\|_2\leq U$. By triangle inequality, we have $\|X_i-\expectation X_i\|_2\leq \|X_i\|_2+\expectation\|X_i\|_2\leq 2U$, i.e.,
\begin{align}\label{eq.temp_120308}
    \left\|\frac{X_i-\expectation X_i}{2U}\right\|_2\leq 1.
\end{align}
We also have
\begin{align}\label{eq.temp_022601}
    \expectation \left[ \frac{X_i-\expectation X_i}{2U}\right]=\frac{\expectation X_i-\expectation X_i}{2U}=\bm{0}.
\end{align}
We then have
\begin{align*}
    &\prob\left\{\left\|\left(\frac{1}{k}\sum_{i=1}^k X_i\right)-\expectation X_1\right\|_2\geq k^{\frac{1}{2q}-\frac{1}{2}}\right\}\\
    =&\prob\left\{\left\|\sum_{i=1}^k\left( X_i-\expectation X_i\right)\right\|_2\geq k^{\frac{1}{2q}+\frac{1}{2}}\right\}\\
    =&\prob\left\{\left\|\sum_{i=1}^k\left( \frac{X_i-\expectation X_i}{2U}\right)\right\|_2\geq \frac{k^{\frac{1}{2q}+\frac{1}{2}}}{2U}\right\}\\
    <&2e^2\exp\left(-\frac{\sqrt[q]{k}}{8U^2}\right)\text{ (by Eqs.~\eqref{eq.temp_120308}\eqref{eq.temp_022601} and letting $a=\frac{k^{\frac{1}{2q}+\frac{1}{2}}}{2U}$ in Lemma~\ref{le.hoeffding})}.
\end{align*}
\end{proof}

\subsection{Calculation of certain integrals}

The following lemma calculates the ratio between the intersection area of two hyper-hemispheres and the area of the whole hyper-sphere.
\begin{lemma}\label{le.spherePortion}
\begin{align}\label{eq.temp_111001}
    \int_{\sd} \bm{1}_{\{\bm{z}^T\bm{v}>0,\ \xsmall^T\bm{v}>0\}} d\nvdensity(\bm{v})=\frac{\pi-\arccos (\xsmall^T\bm{z})}{2\pi}.
\end{align}
(Recall that $\nvdensity(\cdot)$ denotes the distribution of the normalized version of $\Vzeroj$ on $\sd$ and is assumed to be uniform in all directions.)
\end{lemma}
Before we give the proof of Lemma~\ref{le.spherePortion}, we give its geometric explanation.

\emph{Geometric explanation of Eq.~\eqref{eq.temp_111001}}: Indeed, since $\nvdensity$ is uniform on $\sd$, the integral on the left-hand-side of Eq.~\eqref{eq.temp_111001} represents the probability that a random point falls into the intersection of two hyper-hemispheres that are represented by $\{\bm{v}\in\sd\ |\ \bm{z}^T\bm{v}>0\}$ and $\{\bm{v}\in\sd\ |\ \bm{x}^T\bm{v}>0\}$, respectively. We can calculate that probability by
\begin{align}\label{eq.temp_111002}
    \frac{\text{measure of a hyper-spherical lune with angle $\pi-\theta(\bm{z},\xsmall)$}}{\text{measure of a unit hyper-sphere}}=\frac{\pi-\arccos (\xsmall^T\bm{z})}{2\pi},
\end{align}
where $\theta(\cdot,\cdot)$ denote the angle (in radians) between two vectors, which would lead to Eq.~\eqref{eq.temp_111001}.
\begin{figure}[t]
\centering
\includegraphics[width=4cm]{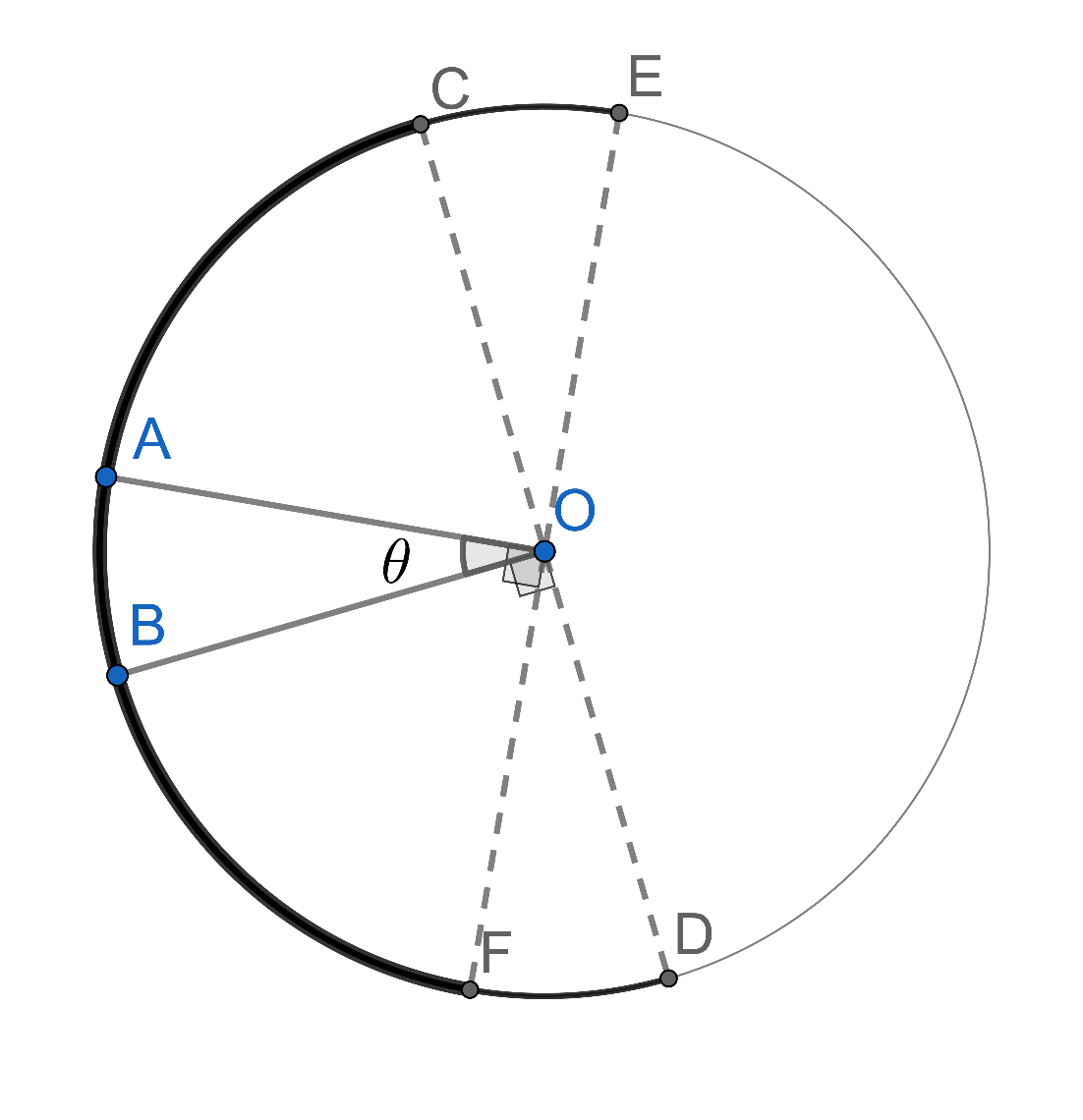}
\caption{The arc $\stackrel{\frown}{\mathrm{CBF}}$ is $\frac{\pi-\theta}{2\pi}$ of the perimeter of the circle $\mathrm{O}$.}
\label{fig.2D}
\end{figure}
\begin{figure}[t]
\centering
\includegraphics[width=8cm]{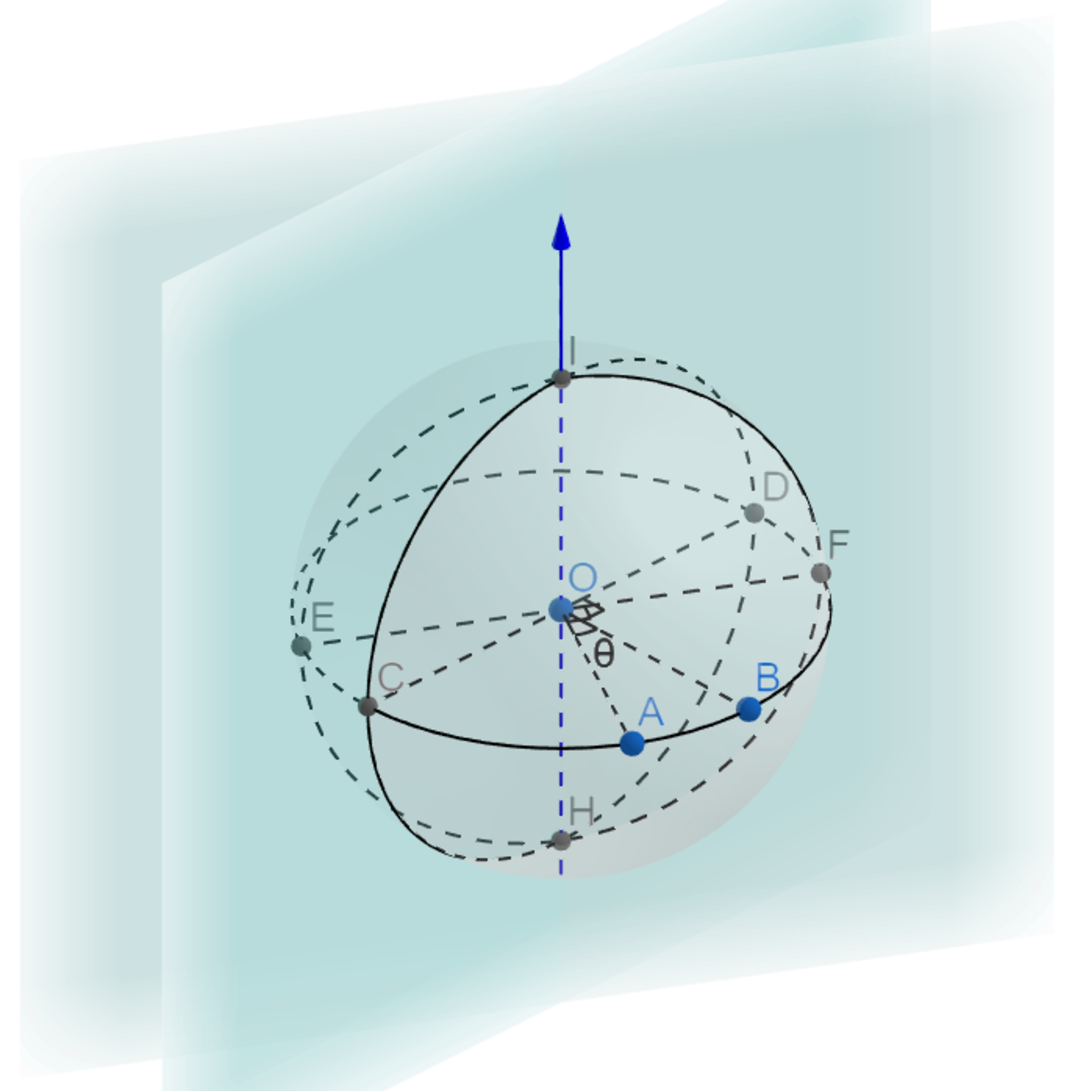}
\caption{The area of the spherical lune $\mathrm{ICHF}$ is $\frac{\pi-\theta}{2\pi}$ of the area of the whole sphere.}
\label{fig.3D}
\end{figure}
To help readers understand Eq.~\eqref{eq.temp_111002}, we give examples for 2D and 3D in Fig.~\ref{fig.2D} and Fig.~\ref{fig.3D}, respectively. In the 2D case depicted in Fig.~\ref{fig.2D},  $\overrightarrow{\mathrm{OA}}$ denotes $\bm{z}$, $\overrightarrow{\mathrm{OB}}$ denotes $\bm{x}$. Thus, the arc $\stackrel{\frown}{\mathrm{EAF}}$ denotes $\{\bm{v}\ |\ \bm{z}^T\bm{v}>0\}$, and the arc $\stackrel{\frown}{\mathrm{CBD}}$ denotes $\{\bm{v}\ |\ \bm{x}^T\bm{v}>0\}$. The intersection of $\stackrel{\frown}{\mathrm{EAF}}$ and $\stackrel{\frown}{\mathrm{CBD}}$, i.e., the arc $\stackrel{\frown}{\mathrm{CBF}}$, represents $\{\bm{v}\ |\ \bm{z}^T\bm{v}>0, \bm{x}^T\bm{v}>0\}$. Notice that the angle of $\stackrel{\frown}{\mathrm{CBF}}$ equals $\pi-\theta$, where $\theta$ denotes the angle between $\bm{z}$ and $\bm{x}$. Therefore, ratio of the length of $\stackrel{\frown}{\mathrm{CBF}}$ to the perimeter of the circle equals to $\frac{\angle\mathrm{COF}}{2\pi}=\frac{\pi-\theta}{2\pi}$. Similarly, in the 3D case depicted in Fig.~\ref{fig.3D}, the spherical lune $\mathrm{ICHF}$ denotes the intersection of the  semi-sphere in the direction of $\overrightarrow{\mathrm{OA}}$ and the semi-sphere in the direction of $\overrightarrow{\mathrm{OB}}$. We can see that the area of the spherical lune $\mathrm{ICHF}$ is still proportional to the angle $\angle \mathrm{COF}$. Thus, we still have the result that the area of the spherical lune $\mathrm{ICHF}$ is $\frac{\pi-\theta}{2\pi}$ of the area of the whole sphere. The proof below, on the other hand, applies to arbitrary dimensions.
\begin{proof}


Due to symmetry, we know that the integral of Eq.~\eqref{eq.temp_111001} only depends on the angle between $\xsmall$ and $\bm{z}$. Thus, without loss of generality, we let
\begin{align*}
    \xsmall =[\xsmall_1\ \xsmall_2\ \cdots\ \xsmall_d]= [0\ 0\ \cdots\ 0\ 1\ 0]^T,\ \bm{z}=[0\ 0\ \cdots\ 0\ \cos\theta\ \sin\theta]^T,
\end{align*}
where 
\begin{align}\label{eq.temp_020901}
    \theta = \arccos(\xsmall^T\bm{z})\in [0,\ \pi].
\end{align}
Thus, for any $\bm{v}=[\bm{v}_1\ \bm{v}_2\ \cdots\ \bm{v}_{d}]^T$ that makes $\bm{z}^T\bm{v}>0$ and $\xsmall^T\bm{v}>0$, it only needs to satisfy
\begin{align}\label{eq.temp_122901}
    [\cos\theta\ \sin\theta]\begin{bmatrix}\bm{v}_{d-1}\\\bm{v}_d\end{bmatrix}>0,\quad [1\ 0]\begin{bmatrix}\bm{v}_{d-1}\\\bm{v}_d\end{bmatrix}>0.
\end{align}
We compute the spherical coordinates $\bm{\varphi}_{\xsmall}=[\varphi_1^{\xsmall}\ \varphi_2^{\xsmall}\ \cdots\ \varphi_{d-1}^{\xsmall}]^T$ where $\varphi_1^{\xsmall},\cdots,\varphi_{d-2}^{\xsmall}\in [0,\pi]$ and $\varphi_{d-1}^{\xsmall}\in[0,2\pi)$ with the convention that
\begin{align*}
    &\bm{x}_1=\cos(\varphi_1^{\xsmall}),\\
    &\bm{x}_2=\sin(\varphi_1^{\xsmall})\cos(\varphi_2^{\xsmall}),\\
    &\bm{x}_3=\sin(\varphi_1^{\xsmall})\sin(\varphi_2^{\xsmall})\cos(\varphi_3^{\xsmall}),\\
    &\vdots\\
    &\bm{x}_{d-1}=\sin(\varphi_1^{\xsmall})\sin(\varphi_2^{\xsmall})\cdots\sin(\varphi_{d-2}^{\xsmall})\cos(\varphi_{d-1}^{\xsmall}),\\
    &\bm{x}_d=\sin(\varphi_1^{\xsmall})\sin(\varphi_2^{\xsmall})\cdots\sin(\varphi_{d-2}^{\xsmall})\sin(\varphi_{d-1}^{\xsmall}).
\end{align*}
Thus, we have $\bm{\varphi}_{\xsmall}=[\pi/2\ \pi/2\ \cdots\ \pi/2\ 0]^T$. Similarly, the spherical coordinates for $\bm{z}$ is $\bm{\varphi}_{\bm{z}}=[\pi/2\ \pi/2\ \cdots \pi/2\ \theta]^T$. Let the spherical coordinates for $\bm{v}$ be $\bm{\varphi}_{\bm{v}}=[\varphi_1^{\bm{v}}\ \varphi_2^{\bm{v}}\ \cdots\ \varphi_{d-1}^{\bm{v}}]^T$. Thus, Eq.~\eqref{eq.temp_122901} is equivalent to
\begin{align}
    &\sin(\varphi_1^{\bm{v}})\sin(\varphi_2^{\bm{v}})\cdots\sin(\varphi_{d-2}^{\bm{v}})\left(\cos\theta \cos(\varphi_{d-1}^{\bm{v}})+\sin\theta\sin(\varphi_{d-1}^{\bm{v}})\right)>0,\label{eq.temp_020902}\\
    &\sin(\varphi_1^{\bm{v}})\sin(\varphi_2^{\bm{v}})\cdots\sin(\varphi_{d-2}^{\bm{v}})\cos(\varphi_{d-1}^{\bm{v}})>0.\label{eq.temp_020903}
\end{align}
Because $\varphi_1^{\bm{v}},\cdots,\varphi_{d-2}^{\bm{v}}\in[0,\pi]$ (by the convention of spherical coordinates), we have
\begin{align*}
    \sin(\varphi_1^{\bm{v}})\sin(\varphi_2^{\bm{v}})\cdots\sin(\varphi_{d-2}^{\bm{v}})\geq 0.
\end{align*}
Thus, for Eq.~\eqref{eq.temp_020902} and Eq.~\eqref{eq.temp_020903}, we have
\begin{align*}
    \cos(\theta-\varphi_{d-1}^{\bm{v}})>0,\quad \cos(\varphi_{d-1}^{\bm{v}})>0,
\end{align*}
i.e., $\varphi_{d-1}^{\bm{v}}\in (-\pi/2,\ \pi/2)\cap (\theta-\pi/2,\ \theta+\pi/2)\pmod{2\pi}$.
We have
\begin{align*}
    &\int_{\sd} \bm{1}_{\{\bm{z}^T\bm{v}>0,\ \xsmall^T\bm{v}>0\}} d\nvdensity(\bm{v})\\
    =&\frac{\int_{(-\frac{\pi}{2},\ \frac{\pi}{2})\cap (\theta-\frac{\pi}{2},\ \theta+\frac{\pi}{2})}\int_0^\pi\cdots\int_0^\pi \sin^{d-2}(\varphi_1)\sin^{d-3}(\varphi_2)\cdots \sin(\varphi_{d-2})\ d\varphi_1\ d\varphi_2\cdots d\varphi_{d-1}}{\int_0^{2\pi}\int_0^\pi\cdots\int_0^\pi \sin^{d-2}(\varphi_1)\sin^{d-3}(\varphi_2)\cdots \sin(\varphi_{d-2})\ d\varphi_1\ d\varphi_2\cdots d\varphi_{d-1}}\\
    =&\frac{\int_{(-\frac{\pi}{2},\ \frac{\pi}{2})\cap (\theta-\frac{\pi}{2},\ \theta+\frac{\pi}{2})}A\cdot d\varphi_{d-1}}{\int_0^{2\pi}A\cdot d\varphi_{d-1}}\\
    &\text{ (by defining $A\defeq \int_0^\pi\cdots\int_0^\pi \sin^{d-2}(\varphi_1)\sin^{d-3}(\varphi_2)\cdots \sin(\varphi_{d-2})\ d\varphi_1\ d\varphi_2$)}\\
    =&\frac{\text{length of the interval }(-\frac{\pi}{2},\ \frac{\pi}{2})\cap (\theta-\frac{\pi}{2},\ \theta+\frac{\pi}{2})}{2\pi}\\
    =&\frac{\pi-\theta}{2\pi}\text{ (because $\theta\in[0,\pi]$ by Eq.~\eqref{eq.temp_020901})}\\
    =&\frac{\pi- \arccos(\bm{x}^T\bm{z})}{2\pi}\text{ (by Eq.~\eqref{eq.temp_020901})}.
\end{align*}
The result of this lemma thus follows.
\end{proof}

\subsection{Limits of \texorpdfstring{${|\Czx|}/{p}$}{|C|/p} when \texorpdfstring{$p\to\infty$}{p approaches infinity}}\label{app.c_divide_p_convergence}

We introduce some notions given by \cite{uniformConvergence}.

\textbf{Glivenko-Cantelli class}. Let $\mathscr{F}$ be a class of integrable real-valued functions with domain $\mathcal{X}$, and let $X_1^k=\{X_1,\cdots,X_k\}$ be a collection of \emph{i.i.d.} samples from some distribution $\mathbb{P}$ over $\mathcal{X}$. Consider the random variable
\begin{align*}
    \|\mathbb{P}_k-\mathbb{P}\|_{\mathscr{F}}\defeq \sup_{\ft\in\mathscr{F}} \left|\frac{1}{k}\sum_{i=1}^k \ft(X_k)-\expectation[\ft]\right|,
\end{align*}
which measures the maximum deviation (over the class $\mathscr{F}$) between the sample average $\frac{1}{k}\sum_{i=1}^k\ft(X_i)$ and the population average $\expectation[\ft]=\expectation[\ft(X)]$. We say that $\mathscr{F}$ is a \emph{Glivenko-Cantelli} class for $\mathbb{P}$ if $\|\mathbb{P}_k-\mathbb{P}\|_{\mathscr{F}}$ converges to zero in probability as $k\to\infty$.

\textbf{Polynomial discrimination}. A class $\mathscr{F}$ of functions with domain $\mathcal{X}$ has polynomial discrimination of order $\nu\geq 1$ if for each positive integer $k$ and collection $X_1^k=\{X_1,\cdots,X_k\}$ of $k$ points in $\mathcal{X}$, the set $\mathscr{F}(X_1^k)$ has cardinality upper bounded by 
\begin{align*}
    \text{card}(\mathscr{F}(X_1^k))\leq (k+1)^{\nu}.
\end{align*}

The following lemma is shown in Page 108 of \cite{uniformConvergence}.
\begin{lemma}\label{le.uniformConvergence}
Any bounded function class with polynomial discrimination is Glivenko-Cantelli.
\end{lemma}

For our case, we care about the following value.
\begin{align*}
    \left|\frac{|\Czx|}{p}-\frac{\pi-\arccos (\xsmall^T\bm{z})}{2\pi}\right|=\left|\frac{1}{p}\sum_{j=1}^p \bm{1}_{\{\xsmall^T\Vzeroj>0,\bm{z}^T\Vzeroj>0\}}-\expectation_{\bm{v}\sim \nvdensity(\cdot)}[\bm{1}_{\{\xsmall^T\bm{v}>0,\bm{z}^T\bm{v}>0\}}]\right|\text{ (by Lemma~\ref{le.spherePortion})}.
\end{align*}
In the language of Glivenko-Cantelli class, the function class $\mathscr{F}_*$ consists of functions $\bm{1}_{\{\xsmall^T\bm{v}>0,\bm{z}^T\bm{v}>0\}}$ that map $\bm{v}\in\sd$ to $0$ or $1$, where every $\bm{x}\in\sd$ and $\bm{z}\in\sd$ corresponds to a distinct function in $\mathscr{F}_*$. 
According to Lemma~\ref{le.uniformConvergence}, we need to calculate the order of the polynomial discrimination for this $\mathscr{F}_*$. 
Towards this end, we need the following lemma, which can be derived from the quantity $Q_{n,N}$ in \cite{wendel1962problem} (which is the quantity $Q_{d,k}$ in the following lemma).
\begin{lemma}\label{le.how_many_different_parts}
Given $\bm{v}_1,\bm{v}_2,\cdots,\bm{v}_k\in\sd$, the number of different values (i.e., the cardinality) of the set $\left\{\left(\bm{1}_{\{\xsmall^T \bm{v}_1>0\}}, \bm{1}_{\{\xsmall^T \bm{v}_2>0\}},\cdots, \bm{1}_{\{\xsmall^T \bm{v}_k>0\}}\right)\ \big|\ \xsmall\in\sd\right\}$ is at most $Q_{d,k}$, where
\begin{align*}
    Q_{d,k}\defeq\begin{cases}
    2\sum_{i=0}^{d-1}\binom{k-1}{i},&\text{ if }k>d,\\
    2^k,&\text{ if }k\leq d.
    \end{cases}
\end{align*}
\end{lemma}
Intuitively, Lemma~\ref{le.how_many_different_parts} states the number of different regions that $k$ hyper-planes through the origin (i.e., the kernel of the inner product with each $\bm{v}_i$) can cut $\sd$ into, because all $\xsmall$ in one region corresponds to the same value of the tuple $\left(\bm{1}_{\{\xsmall^T \bm{v}_1>0\}}, \bm{1}_{\{\xsmall^T \bm{v}_2>0\}},\cdots, \bm{1}_{\{\xsmall^T \bm{v}_k>0\}}\right)$. For example, in the 2D case (i.e., $d=2$), $k$ diameters of a circle can at most cut the whole circle into $2k$ (which equals to $Q_{2,k}$) parts. Notice that if some $\bm{v}_i$'s are parallel (thus some diameters are overlapped), then the total number of different parts can only be smaller. That is why Lemma~\ref{le.how_many_different_parts} states that the cardinality is ``at most'' $Q_{d,k}$.

The following lemma shows that the cardinality in Lemma~\ref{le.how_many_different_parts} is polynomial in $k$.
\begin{lemma}\label{le.estiamte_order_polynomial_discrimination}
Recall the definition $Q_{d,k}$ in Lemma~\ref{le.how_many_different_parts}. For any integer $k\geq 1$ and $d\geq 2$, we must have $Q_{d,k}\leq (k+1)^{d+1}$.
\end{lemma}
\begin{proof}
When $k>d$, because $\binom{k-1}{i}\leq (k-1)^{d-1}$ when $i\leq d-1$, we have $Q_{d,k}=2\sum\limits_{i=0}^{d-1}\binom{k-1}{i}\leq 2d(k+1)^{d-1}\leq (k+1)^{d+1}$ (the last step uses $k\geq 1$ and $k>d$).
When $k\leq d$, because $k\geq 1$, we have $Q_{d,k}=2^k\leq (k+1)^k\leq (k+1)^d$. 
In summary, for any integer $k\geq 1$ and $d\geq 2$, the result $Q_{d,k}\leq (k+1)^{d+1}$ always holds.
\end{proof}

We can now calculate the order of the polynomial discrimination for the function class $\mathscr{F}_*$.
Because
\begin{align*}
    &\text{card}\left(\left\{ \left(\bm{1}_{\{\xsmall^T\bm{v}_1>0,\bm{z}^T\bm{v}_1>0\}}, \bm{1}_{\{\xsmall^T \bm{v}_2>0,\bm{z}^T\bm{v}_2>0\}},\cdots, \bm{1}_{\{\xsmall^T \bm{v}_k>0,\bm{z}^T\bm{v}_k0\}}\right)\ \big|\ \xsmall\in\sd, \bm{z}\in\sd\right\}\right)\\
    \leq & \text{card}\left(\left\{\left(\bm{1}_{\{\xsmall^T \bm{v}_1>0\}}, \bm{1}_{\{\xsmall^T \bm{v}_2>0\}},\cdots, \bm{1}_{\{\xsmall^T \bm{v}_k>0\}}\right)\ \big|\ \xsmall\in\sd\right\}\right)\\
    &\cdot \text{card}\left(\left\{\left(\bm{1}_{\{\bm{z}^T \bm{v}_1>0\}}, \bm{1}_{\{\bm{z}^T \bm{v}_2>0\}},\cdots, \bm{1}_{\{\bm{z}^T \bm{v}_k>0\}}\right)\ \big|\ \bm{z}\in\sd\right\}\right),
\end{align*}
by Lemma~\ref{le.how_many_different_parts} and Lemma~\ref{le.estiamte_order_polynomial_discrimination}, we know that
\begin{align*}
    \text{card}(\mathscr{F}_*(X_1^k))\leq \left(Q_{d,k}\right)^2\leq (k+1)^{ 2(d+1)}.
\end{align*}
(Here $X_1^k$ means $\{\Vzero[1],\cdots,\Vzero[k]\}$.)

Thus, $\mathscr{F}_*$ has polynomial discrimination with order at most $2(d+1)$. Notice that all functions in $\mathscr{F}_*$ is bounded because their outputs can only be $0$ or $1$. Therefore, by Lemma~\ref{le.uniformConvergence} (i.e., any bounded function class with polynomial discrimination is Glivenko-Cantelli), we know that $\mathscr{F}_*$ is Glivenko-Cantelli. In other words, we have shown the following lemma.
\begin{lemma}\label{le.c_divide_p_convergence}
\begin{align}\label{eq.c_p_convergence}
    \sup_{\xsmall,\bm{z}\in\sd}\left|\frac{|\Czx|}{p}-\frac{\pi-\arccos (\xsmall^T\bm{z})}{2\pi}\right|\stackrel{\text{P}}{\rightarrow}0, \text{ as }p\to\infty.
\end{align}
\end{lemma}

\section{Proof of Lemma~\ref{le.full_rank} (\texorpdfstring{$\Ht$}{H} has full row-rank with high probability as \texorpdfstring{$p\to\infty$}{p approaches infinity})}\label{app.full_rank}
In this section, we prove Lemma~\ref{le.full_rank}, i.e., the matrix $\Ht$ has full row-rank with high probability when $p\to\infty$.
We first introduce two useful lemmas as follows.

The following lemma states that, given $\XX$ (that satisfies Assumption~\ref{as.nparallel}) and $k\in\{1,2,\cdots,n\}$, there always exists a vector $\bm{v}\in\sd$ that is only orthogonal to one training input $\Xkk$ but not orthogonal to other training inputs $\Xii$ for all $i\neq k$. An intuitive explanation is that, because no training inputs are parallel (as stated in Assumption~\ref{as.nparallel}), the total set of vectors that are orthogonal to at least two training inputs is too small. That gives us many options to pick such a vector $\bm{v}$ that is only orthogonal to one input but not others.
\begin{lemma}\label{le.nonEmpty}
For all $k\in\{1,2,\cdots,n\}$ we have
\begin{align*}
    \mathcal{T}_k\defeq \left\{\bm{v}\in\sd\ |\ \bm{v}^T\Xkk=0,\bm{v}^T\Xii\neq 0,\text{for all }i\in \{1,2,\cdots,n\}\setminus\{k\}\right\}\neq \varnothing.
\end{align*}
\end{lemma}
\begin{proof}
We have
\begin{align*}
    \mathcal{T}_k&=\sd\cap\ker(\Xkk)\setminus\left(\bigcup_{i\in \{1,2,\cdots,n\}\setminus\{k\}}\ker(\Xii)\right)\\
    &=\sd\cap\ker(\Xkk)\setminus\left(\bigcup_{i\in \{1,2,\cdots,n\}\setminus\{k\}}\left(\sd\cap\ker(\Xkk)\cap \ker(\Xii)\right)\right).
\end{align*}
Because
\begin{align}
    &\dim(\sd\cap\ker(\Xkk))=d-2,\nonumber\\
    &\dim(\sd\cap\ker(\Xkk)\cap \ker(\Xii))=d-3\text{ for all }i\in\{1,2,\cdots,n\}\setminus\{k\}\text{ (because $\Xii\nparallel\Xkk$)},\label{eq.temp_093002}
\end{align}
we have
\begin{align}
    &\lambda_{d-2}(\sd\cap\ker(\Xkk))=\lambda_{d-2}(\mathcal{S}^{d-2})>0,\nonumber\\
    &\lambda_{d-2}\left(\sd\cap\ker(\Xkk)\cap \ker(\Xii)\right)=0\text{ for all }i\in\{1,2,\cdots,n\}\setminus\{k\}.\label{eq.temp_093001}
\end{align}
(When $d=2$, the set in Eq.~\eqref{eq.temp_093002} is not defined. Nonetheless, Eq.~\eqref{eq.temp_093001} still holds when $d=2$.)
Thus, we have
\begin{align*}
    \lambda_{d-2}(\mathcal{T}_k)&=\lambda_{d-2}\left(\sd\cap\ker(\Xkk)\right)-\lambda_{d-2}\left(\bigcup_{i\in \{1,2,\cdots,n\}\setminus\{k\}}\left(\sd\cap\ker(\Xkk)\cap \ker(\Xii)\right)\right)\\
    &\geq \lambda_{d-2}\left(\sd\cap\ker(\Xkk)\right)-\sum_{i\in \{1,2,\cdots,n\}\setminus\{k\}}\lambda_{d-2}\left(\sd\cap\ker(\Xkk)\cap \ker(\Xii)\right)\\
    &=\lambda_{d-2}(\mathcal{S}^{d-2})\\
    &>0.
\end{align*}
Therefore, $\mathcal{T}_k\neq \varnothing$.
\end{proof}

\begin{figure}[t]
\centering
\includegraphics[width=8cm]{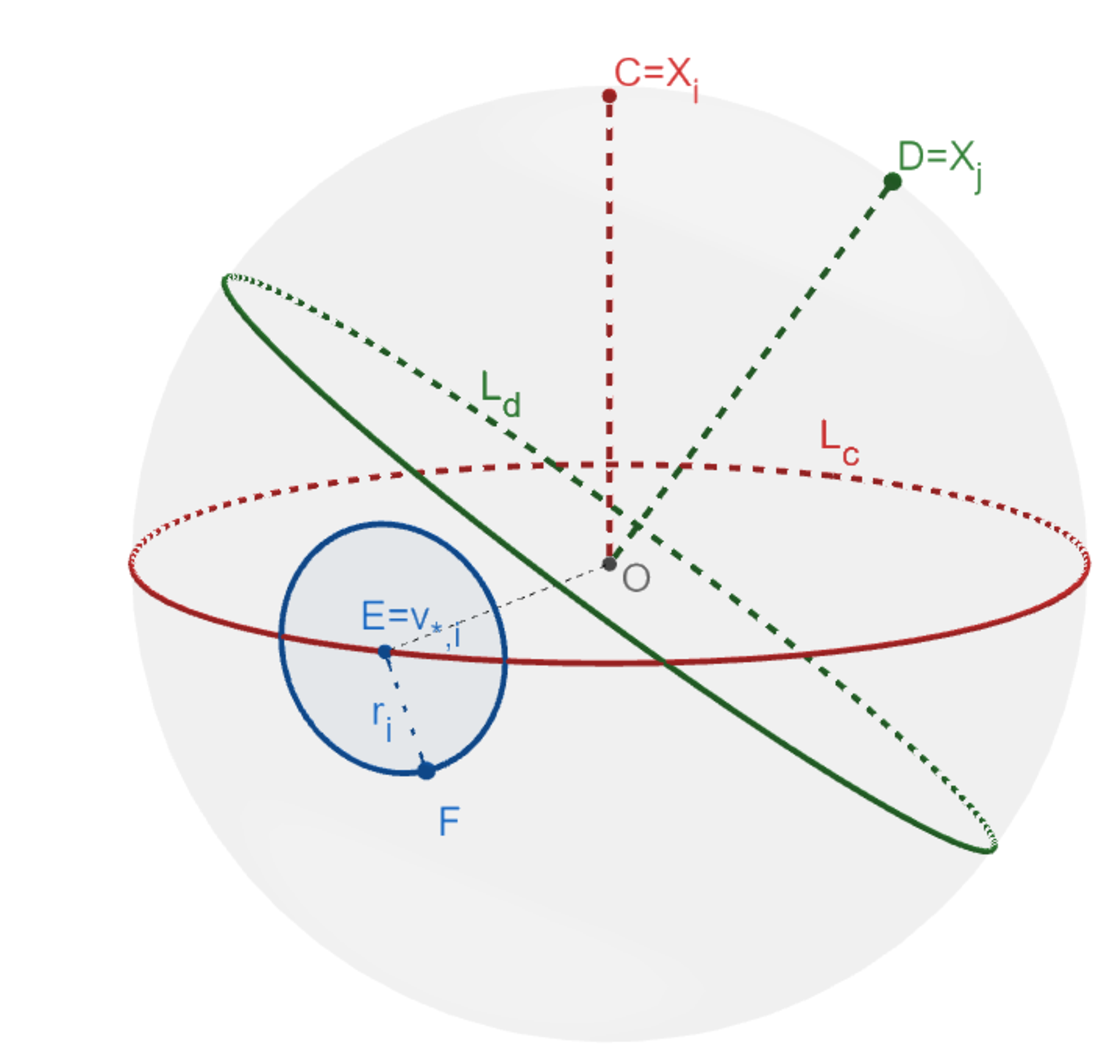}
\caption{Geometric interpretation of $\capsrpX$ and $\capsrnX$ on a sphere (i.e., $\mathcal{S}^2$).}
\label{fig.split_cap}
\end{figure}
The following lemma plays an important role in answering whether $\Ht$ has full row-rank. Further, it is also closely related to our estimation on the $\min\eig(\Ht\Ht^T)$ later in Appendix~\ref{app.noise_large_p}.
\begin{lemma}\label{le.split_half}
Consider any $i\in\{1,2,\cdots,n\}$. For any $\vstari\in \sd$ satisfying $\vstari^T\Xii=0$, we define
\begin{align}\label{eq.temp_020501}
    r_i\defeq\min_{j\in \{1,2,\cdots,n\}\setminus \{i\}}\left|\vstari^T\Xjj\right|.
\end{align}
 If there exist $k,l\in\{1,\cdots,p\}$ such that
\begin{align}\label{eq.temp_110206}
    \frac{\Vzerok}{\|\Vzerok\|_2}\in \capsrpX,\quad
    \frac{\Vzerol}{\|\Vzerol\|_2}\in \capsrnX,
\end{align}
then we must have
\begin{align}
&\Htj[k]=\Htj[l],\ \text{for all } j\in\{1,2,\cdots,n\}\setminus\{i\},\label{eq.temp_020701}\\
&\Hti[k]=\Xii^T,\label{eq.temp_020702}\\
&\Hti[l]=\bm{0}.\label{eq.temp_020703}
\end{align}
(Notice that Eq.~\eqref{eq.temp_110206} implies $r_i>0$.)
\end{lemma}

\begin{remark}\label{remark.cap}
We first give an intuitive geometric interpretation of Lemma~\ref{le.split_half}. 
In Fig.~\ref{fig.split_cap}, the sphere centered at $\mathrm{O}$ denotes $\sd$, the vector $\overrightarrow{\mathrm{OC}}$ denotes $\Xii$, the vector $\overrightarrow{\mathrm{OD}}$ denotes one of other $\Xjj$'s, the vector $\overrightarrow{\mathrm{OE}}$ denotes $\vstari$, which is perpendicular to $\Xii$ (i.e., $\Xii^T\vstari=0$). 
The upper half of the cap $\mathrm{E}$ denotes $\capsrpX$, the lower half of the cap $\mathrm{E}$ denotes $\capsrnX$.
The great circle $\mathrm{L_c}$ cuts the sphere into two semi-spheres. The semi-sphere in the direction of $\overrightarrow{\mathrm{OC}}$ corresponds to all vectors $\bm{v}$ on the sphere that have positive inner product with $\Xii$ (i.e., $\vstari^T\Xii>0$), and the semi-sphere in the opposite direction of $\overrightarrow{\mathrm{OC}}$ corresponds to all vectors $\bm{v}$ on the sphere that have negative inner product with $\Xii$ (i.e., $\bm{v}^T\Xii<0$). The great circle $\mathrm{L_d}$ is similar to the great circle $\mathrm{L_c}$, but is perpendicular to the direction $\overrightarrow{\mathrm{OD}}$ (i.e., $\Xjj$). By choosing the radius of the cap $\mathrm{E}$ in Eq.~\eqref{eq.temp_020501}, we can ensure that all great circles that are perpendicular to other $\Xjj$'s do not pass the cap $\mathrm{E}$. 
In other words, for the two semi-spheres cut by the great circle perpendicular to $\Xjj$, $j\neq i$, the cap $\mathrm{E}$ must be contained in one of them.
Therefore, vectors on the upper half of the cap $\mathrm{E}$ and the vectors on the lower half of the cap $\mathrm{E}$ must have the same sign when calculating the inner product with all $\Xjj$'s, for all $j\neq i$.

Now, let us consider the meaning of Eq.~\eqref{eq.temp_110206} in this geometric setup depicted in Fig.~\ref{fig.split_cap}. The expression $\frac{\Vzerok}{\|\Vzerok\|_2}\in \capsrpX$ means that the direction of $\Vzerok$ is in the upper half of the cap $\mathrm{E}$. By the definition of $\Hti=\hsmall_{\Vzero,\Xii}$ in Eq.~\eqref{eq.def_hx}, we must then have $\Hti[k]=\Xii^T$. Similarly, the expression $\frac{\Vzerol}{\|\Vzerol\|_2}\in \capsrnX$ means that the direction of $\Vzerol$ is in the lower half of the cap $\mathrm{E}$, and thus $\Hti[l]=\bm{0}$. Then, based on the discussions in the previous paragraph, we know that $\Vzerok$ and $\Vzerol$ has the same activation pattern under ReLU for all $\Xjj$'s that $j\neq i$, which implies that $\Htj[k]=\Htj[l]$. These are precisely the conclusions in Eqs.~\eqref{eq.temp_020701}\eqref{eq.temp_020702}\eqref{eq.temp_020703}.

Later in Appendix~\ref{app.noise_large_p}, Lemma~\ref{le.split_half} plays an important role in estimating $\min_{\bm{a}\in\sn}\|\Ht^T\bm{a}\|_2^2$. To see this, let $a_j$ denotes the $j$-th element of $\bm{a}$. By Eq.~\eqref{eq.temp_020701}, we have $\sum_{j\in\{1,2,\cdots,n\}\setminus\{i\}}((\Ht^T a_j)[k]-(\Ht^T a_j)[l])=\bm{0}$. By Eq.~\eqref{eq.temp_020702} and Eq.~\eqref{eq.temp_020703}, we have $(\Ht^T a_{i})[k]-(\Ht^T a_{i})[l]=\Xii$. Combining them together, we have $(\Ht^T\bm{a})[k]-(\Ht^T\bm{a})[l]=a_{i}\Xii$. As long as $a_{i}$ is not zero, then regardless values of other elements in $\bm{a}$, we always obtain that $\Ht^T\bm{a}$ is a non-zero vector. This implies $\|\Ht^T\bm{a}\|_2>0$, which will be useful for estimating $\min\eig(\Ht\Ht^T)/p$ in Appendix~\ref{app.noise_large_p}.

\end{remark}

\begin{proof}
By the definition of $r_i$, we have
\begin{align}\label{eq.temp_093003}
    |\vstari^T\Xjj|-r_i\geq 0\text{, for all }j\in\{1,2,\cdots,n\}\setminus\{i\}.
\end{align}
For any $j\in\{1,2,\cdots,n\}\setminus \{i\}$ and any $\bm{v}\in \capsrX$, since $\|\bm{v}-\vstari\|_2<r_i$, we have
\begin{align*}
    (\bm{v}^T\Xjj)(\vstari^T\Xjj)&=\left((\bm{v}-\vstari)^T\Xjj+\vstari^T\Xjj\right)(\vstari^T\Xjj)\\
    &=(\vstari^T\Xjj)^2+(\vstari^T\Xjj)\left((\bm{v}-\vstari)^T\Xjj\right)\\
    &\geq (\vstari^T\Xjj)^2 - \left|\vstari^T\Xjj\right|\cdot\left|(\bm{v}-\vstari)^T\Xjj\right|\\
    &\geq (\vstari^T\Xjj)^2 - \left|\vstari^T\Xjj\right|\cdot\left\|\bm{v}-\vstari\right\|_2\left\|\Xjj\right\|_2\\
    &> (\vstari^T\Xjj)^2 - \left|\vstari^T\Xjj\right|\cdot r_i\text{ (by Eq.~\eqref{eq.temp_100101})}\\
    &= |\vstari^T\Xjj|(|\vstari^T\Xjj|-r_i)\\
    &\geq 0\text{ (by Eq.~\eqref{eq.temp_093003})}.
\end{align*}
Thus, for any $\bm{v}_1\in\capsrX,\ \bm{v}_2\in\capsrX,\ j\in\{1,2,\cdots,n\}\setminus\{i\}$, we have $(\bm{v}_1^T\Xjj)(\vstari^T\Xjj)>0$ and $(\bm{v}_2^T\Xjj)(\vstari^T\Xjj)>0$. It implies that
\begin{align}\label{eq.temp_100104}
    \text{sign}(\bm{v}_1^T\Xjj)=\text{sign}(\vstari^T\Xjj)=\text{sign}(\bm{v}_2^T\Xjj).
\end{align}
By Eq.~\eqref{eq.temp_110206}, we know that both $\Vzerok$ and $\Vzerol$ are in $\capsrX$. Applying Eq.~\eqref{eq.temp_100104}, we have
\begin{align*}
    \text{sign}(\Xjj^T\Vzerok)=\text{sign}(\Xjj^T\Vzerol),\ \text{for all } j\in\{1,2,\cdots,n\}\setminus\{i\}.
\end{align*}
Thus, by Eq.~\eqref{eq.def_hx}, we have
\begin{align*}
    \Htj[k]=\bm{1}_{\{\Xjj^T\Vzerok>0\}}\Xjj^T=\bm{1}_{\{\Xjj^T\Vzerol>0\}}\Xjj^T=\Htj[l],\ \text{for all } j\in\{1,2,\cdots,n\}\setminus\{i\}.
\end{align*}
By Eq.~\eqref{eq.temp_110205}, we have
\begin{align*}
    \Xii^T\Vzerok>0,\ \Xii^T\Vzerol<0.
\end{align*}
Thus, by Eq.~\eqref{eq.def_hx}, we have
\begin{align*}
    \Hti[k]=\bm{1}_{\{\Xii^T\Vzerok>0\}}\Xii^T=\Xii^T,\quad \Hti[l]=\bm{1}_{\{\Xii^T\Vzerol>0\}}\Xii^T=\bm{0}.
\end{align*}
\end{proof}

Now, we are ready to prove Lemma~\ref{le.full_rank}.
\begin{proof}
We prove by contradiction. Suppose on the contrary that with some nonzero probability, the design matrix is not full row-rank as $p\to \infty$. Note that when the design matrix is not full row-rank, there exists a set of indices $\mathcal{I}\subseteq\{1,\cdots,n\}$ such that
\begin{align}\label{eq.temp_100108}
    \sum_{i\in \mathcal{I}}b_i\Hti=\bm{0}, \  b_i\neq 0\text{ for all } i\in\mathcal{I}.
\end{align}
The proof will be finished by two steps: 1) find an event $\mathcal{J}$ that happens almost surely when $p\to\infty$; 2) prove this event $\mathcal{J}$ contradicts Eq.~\eqref{eq.temp_100108}.

\noindent\textbf{Step 1}:

Consider each $i\in\{1, 2,\cdots, n\}$. By Lemma~\ref{le.nonEmpty}, we know that there exists a $\vstari\in\sd$ such that
\begin{align}\label{eq.temp_030401}
    \vstari^T\Xii=0,\ \vstari^T\Xjj\neq 0,\text{ for all } j\in \{1,2,\cdots,n\}\setminus \{i\}.
\end{align}
Define
\begin{align}\label{eq.temp_100103}
    r_i=\min_{j\in \{1,2,\cdots,n\}\setminus \{i\}}\left|\vstari^T\Xjj\right|>0.
\end{align}

For all $i=1,2,\cdots,n$, we define several events as follows.
\begin{align*}
    &\mathcal{J}_i\defeq \left\{\mathcal{A}_{\Vzero}\cap\capsrpX\neq\varnothing,\ \mathcal{A}_{\Vzero}\cap\capsrnX\neq\varnothing\right\},\\
    &\mathcal{J}_{i,+}=\left\{\mathcal{A}_{\Vzero}\cap\capsrpX\neq\varnothing\right\},\\
    &\mathcal{J}_{i,-}=\left\{\mathcal{A}_{\Vzero}\cap\capsrnX\neq\varnothing\right\},\\
    &\mathcal{J}\defeq \bigcap_{i=1}^n \mathcal{J}_i.
\end{align*}

(Recall the geometric interpretation in Remark~\ref{remark.cap}. The events $\mathcal{J}_{i,+}$ and $\mathcal{J}_{i,-}$ mean that there exists $\Vzeroj/\|\Vzeroj\|_2$ in the upper half and the lower half of the cap $\mathrm{E}$, respectively. The event $\mathcal{J}_i=\mathcal{J}_{i,+}\cap \mathcal{J}_{i,-}$ means that there exist $\Vzeroj/\|\Vzeroj\|_2$ in both halves of the cap $\mathrm{E}$.
Finally, the event $\mathcal{J}$ occurs when $\mathcal{J}_i$ occurs for all $i$, although the vector $\Vzeroj/\|\Vzeroj\|_2$ that falls into the two halves may differ across $i$. As we will show later, whenever the event $\mathcal{J}$ occurs, the matrix $\Ht$ will have the full row-rank, which is why we are interesting in the probability of the event $\mathcal{J}$.)

Those definitions implies that
\begin{align}
    &\mathcal{J}_i^c=\mathcal{J}_{i,+}^c\cup \mathcal{J}_{i,-}^c\text{ for all }i=1,2,\cdots,n,\label{eq.temp_112302}\\
    &\mathcal{J}^c=\bigcup_{i=1}^n \mathcal{J}_i^c.\label{eq.temp_112301}
\end{align}
Thus, we have
\begin{align}
    \prob_{\Vzero}[\mathcal{J}]=&1-\prob_{\Vzero}[\mathcal{J}^c]\nonumber\\
    \geq& 1-\sum_{i=1}^n\prob_{\Vzero}[\mathcal{J}_i^c]\text{ (by Eq.~\eqref{eq.temp_112301} and the union bound)}.\label{eq.temp_112303}
\end{align}

For a fixed $i$, recall that by Eq.~\eqref{eq.temp_100103}, we have $r_i>0$.
Because $\capsrpX$ and  $\capsrnX$ are two halves of $\capsrX$, we have
\begin{align}\label{eq.temp_100102}
    \musd(\capsrpX)=\musd(\capsrnX)=\frac{1}{2}\musd(\capsrX).
\end{align}
Therefore, we have
\begin{align*}
    \prob_{\Vzero}[\mathcal{J}_i^c]\leq& \prob_{\Vzero}[\mathcal{J}_{i,+}^c]+\prob_{\Vzero}[\mathcal{J}_{i,-}^c]\text{ (by Eq.~\eqref{eq.temp_112302} and the union bound)}\\
    =&\left(1-\frac{\musd(\capsrpX)}{\musd(\sd)}\right)^p+\left(1-\frac{\musd(\capsrnX)}{\musd(\sd)}\right)^p\\
    &\text{ (all $\Vzeroi$'s are independent and Assumption~\ref{as.uniform})}\\
    =&2\left(1-\frac{\musd(\capsrX)}{2\musd(\sd)}\right)^p \text{ (by Eq.~\eqref{eq.temp_100102})}.
\end{align*}
Notice that $r_i$ is determined only by $\XX$, and is independent of $\Vzero$ and $p$. Therefore, we have
\begin{align}\label{eq.temp_112304}
    \lim_{p\to\infty}\prob_{\Vzero}[\mathcal{J}_i^c]=0.
\end{align}
Plugging Eq.~\eqref{eq.temp_112304} into Eq.~\eqref{eq.temp_112303}, we have
\begin{align*}
    \lim_{p\to\infty}\prob_{\Vzero} [\mathcal{J}]=1 \text{ (because $n$ is finite)}.
\end{align*}

\noindent\textbf{Step 2}:

To complete the proof, it remains to show that the event $\mathcal{J}$ contradicts Eq.~\eqref{eq.temp_100108}. Towards this end, we assume the event $\mathcal{J}$ happens. By Eq.~\eqref{eq.temp_100108}, we can pick one $i\in\mathcal{I}$. Further, by the definition of $\mathcal{J}$, there exists $r_i$ such that $\mathcal{A}_{\Vzero}\cap\capsrpX\neq\varnothing$ and $ \mathcal{A}_{\Vzero}\cap\capsrnX\neq\varnothing$. In other words, there must exist $k,l\in\{1,\cdots,p\}$ such that
\begin{align*}
    \frac{\Vzerok}{\|\Vzerok\|_2}\in \capsrpX,\quad
    \frac{\Vzerol}{\|\Vzerol\|_2}\in \capsrnX.
\end{align*}
By Lemma~\ref{le.split_half}, we have
\begin{align}
    &\mathbf{H}_{j}[k]=\mathbf{H}_{j}[l],\ \text{for all } j\in\{1,2,\cdots,n\}\setminus\{i\},\label{eq.temp_1000106}\\
    &\Hti[k]=\Xii^T,\quad \Hti[l]=\bm{0}.\label{eq.temp_100107}
\end{align}
We now show that $\Ht$ restricted to the columns corresponding to $k$ and $l$ cannot be linearly dependent. Specifically, we have
\begin{align*}
    \sum_{j\in\mathcal{I}}b_j\Htj[k]&=b_i\Hti[k]+\sum_{j\in\mathcal{I}\setminus\{i\}}b_j\Htj[k]\text{ (as we have picked $i\in\mathcal{I}$)}\\
    &=b_i\Hti[k]-b_j\Hti[l]+\sum_{j\in\mathcal{I}}b_j\Htj[l]\text{ (by Eq.~\eqref{eq.temp_1000106})}\\
    &=b_i\Xii^T+\sum_{j\in\mathcal{I}}b_j\Htj[l]\text{ (by Eq.~\eqref{eq.temp_100107})}\\
    &\neq \sum_{j\in\mathcal{I}}b_j\Htj[l]\text{ (because $b_i\neq 0$)}.
\end{align*}
This contradicts the assumption Eq.~\eqref{eq.temp_100108} that
\begin{align*}
    \sum_{j\in\mathcal{I}}b_j\Htj[k]=\sum_{j\in\mathcal{I}}b_j\Htj[l]=\bm{0}.
\end{align*}
The result thus follows.
\end{proof}

\section{Proof of Proposition~\ref{prop.noise_large_p} (the upper bound of the variance)}\label{app.noise_large_p}
The following lemma shows the relationship between the variance term and $\min\eig\left(\Ht\Ht^T\right)/p$.
\begin{lemma}\label{le.temp_122501}
\begin{align*}
    |\hx\Ht^T(\Ht\Ht^T)^{-1}\esmall|\leq \frac{\sqrt{p}\|\esmall\|_2}{\sqrt{\min\eig(\Ht\Ht^T)}}.
\end{align*}
\end{lemma}
\begin{proof}
We have
\begin{align}\label{eq.temp_022801}
    \|\Ht^T(\Ht\Ht^T)^{-1}\esmall\|_2=\sqrt{(\Ht^T(\Ht\Ht^T)^{-1}\esmall)^T\Ht^T(\Ht\Ht^T)^{-1}\esmall}=\sqrt{\esmall^T(\Ht\Ht^T)^{-1}\esmall}\leq \frac{\|\esmall\|_2}{\sqrt{\min\eig(\Ht\Ht^T)}}.
\end{align}

Thus, we have
\begin{align*}
    &|\hx\Ht^T(\Ht\Ht^T)^{-1}\esmall|\\
    =&\|\hx\Ht^T(\Ht\Ht^T)^{-1}\esmall\|_2\text{ ($\ell_2$-norm of a number equals to its absolute value)}\\
    \leq &\|\hx\|_2\cdot\|\Ht^T(\Ht\Ht^T)^{-1}\esmall\|_2\text{ (by Lemma~\ref{le.matrix_norm})}\\
    \leq& \frac{\sqrt{p}\|\esmall\|_2}{\sqrt{\min\eig(\Ht\Ht^T)}}\text{ (by Lemma~\ref{le.h_p} and Eq.~\eqref{eq.temp_022801})}.
\end{align*}
\end{proof}

The following lemma shows our estimation on $\min\eig\left(\Ht\Ht^T\right)/p$.

\begin{lemma}\label{le.lower_eig}
For any $n\geq 2$, $m\in\left[1,\ \frac{\ln n}{\ln \frac{\pi}{2}}\right]$, $d\leq n^4$, if $p\geq 6J_m(n,d)\ln\left(4n^{1+\frac{1}{m}}\right)$,
we must have
\begin{align*}
    \prob_{\XX,\Vzero}\Big\{\frac{\min\eig\left(\Ht\Ht^T\right)}{p}\geq \frac{1}{J_m(n,d)n}\Big\}\geq 1-\frac{2}{\sqrt[m]{n}}.
\end{align*}
\end{lemma}

Proposition~\ref{prop.noise_large_p} directly follows from Lemma~\ref{le.lower_eig} and Lemma~\ref{le.temp_122501}. \footnote{We can see that the key part during the proof of Proposition~\ref{prop.noise_large_p} is to estimate ${\min\eig\left(\Ht\Ht^T\right)}/{p}$. Lemma~\ref{le.lower_eig} shows a lower bound of ${\min\eig\left(\Ht\Ht^T\right)}/{p}$ which is almost $\Omega(n^{1-2d})$ when $p$ is large. However, our estimation of this value may be loose. We will show a upper bound which is $O(n^{-\frac{1}{d-1}})$ (see Appendix~\ref{app.upper_bound_eig}).}

In rest of this section, we will show how to prove Lemma~\ref{le.lower_eig}.
The following lemma shows that, to estimate $\min\eig\left(\Ht\Ht^T\right)/p$, it is equivalent to estimate $\min_{\bm{a}\in\sn}\|\Ht^T\bm{a}\|_2^2/p$.
\begin{lemma}\label{le.min_eig}
\begin{align*}
    \min\eig\left(\Ht\Ht^T\right)=\min_{\bm{a}\in\sn}\|\Ht^T\bm{a}\|_2^2.
\end{align*}
\end{lemma}
\begin{proof}
Do the singular value decomposition (SVD) of $\Ht^T$ as $\Ht^T=\mathbf{U}\mathbf{\Sigma}\mathbf{W}^T$, where 
\begin{align*}
    \mathbf{\Sigma}\in\mathds{R}^{(dp)\times n}=\diag(\Sigma_1,\Sigma_2,\cdots,\Sigma_n).
\end{align*}
By properties of singular values, we have
\begin{align*}
    \min_{\bm{a}\in\sn}\|\Ht^T\bm{a}\|_2^2=\min_{i\in\{1,2,\cdots,n\}}\Sigma_i^2.
\end{align*}
We also have
\begin{align*}
    \Ht\Ht^T&=\mathbf{W}\mathbf{\Sigma}^T\mathbf{U}^T\mathbf{U}\mathbf{\Sigma}\mathbf{W}^T\\
    &=\mathbf{W}\mathbf{\Sigma}^T\mathbf{\Sigma}\mathbf{W}^T\text{ (because $\mathbf{U}^T\mathbf{U}=\mathbf{I}$)}\\
    &=\mathbf{W}\diag(\Sigma_1^2,\Sigma_2^2,\cdots,\Sigma_n^2)\mathbf{W}^T.
\end{align*}
This equation is indeed the  eigenvalue decomposition of $\Ht\Ht^T$, which implies that its eigenvalues are $\Sigma_1^2,\Sigma_2^2,\cdots,\Sigma_n^2$.
Thus, we have
\begin{align*}
    \min\eig\left(\Ht\Ht^T\right)=\min_{i\in\{1,2,\cdots,n\}}\Sigma_i^2=\min_{\bm{a}\in\sn}\|\Ht^T\bm{a}\|_2^2.
\end{align*}
\end{proof}

Therefore, to finish the proof of Proposition~\ref{prop.noise_large_p}, it only remains to estimate $\min_{\bm{a}\in\sn}\|\Ht^T\bm{a}\|_2^2$.

By Lemma~\ref{le.full_rank} and its proof in Appendix~\ref{app.full_rank}, we have already shown that $\Ht^T\bm{a}$ is not likely to be zero (i.e. $\min_{\bm{a}\in\sn}\|\Ht^T\bm{a}\|_2^2>0$) when $p\to\infty$. Here, we basically use the similar method as in Appendix~\ref{app.full_rank}, but with more precise quantification.

Recall the definitions in Eqs.~\eqref{eq.temp_100101}\eqref{eq.temp_110205}\eqref{eq.temp_110204}.
For any $i\in\{1,2,\cdots,n\}$, we choose one
\begin{align}\label{eq.temp_030403}
    \vstari\in\sd\text{ independently of $\Xjj$, $j\neq i$, such that }\vstari^T\Xii=0.
\end{align}
(Note that here, unlike in Eq.~\eqref{eq.temp_030401}, we do not require $\vstari^T\Xjj\neq 0$ for all $j\neq i$. This is important as we would like $\Xjj$ to be independent of $\vstari$ for all $j\neq i$.) Further, for any $0\leq r_0\leq 1$, we define
\begin{align}\label{eq.temp_110301}
    c_{r_0}^{i}\defeq\min\left\{|\mathcal{A}_{\Vzero}\cap \capsrzpX|,\ |\mathcal{A}_{\Vzero}\cap \capsrznX|\right\}.
\end{align}
Then, we define
\begin{align}
    &r_i\defeq \min_{j\in\{1,2,\cdots,n\}\setminus\{i\}}\left|\vstari^T\Xjj\right|,\label{eq.temp_110201}\\
    &\hat{r}\defeq \min_{i\in\{1,2,\cdots,n\}}r_i.\label{eq.def_hat_r}
\end{align}
(Note that here $r_i$ or $\hat{r}$ may be zero. Later we will show that they can be lower bounded with high probability.)
Define
\begin{align}\label{eq.def_DX}
    \DX\defeq \frac{\musd(\capsrh)}{8n\musd(\sd)}.
\end{align}
Similar to Remark~\ref{remark.cap}, these definitions have their geometric interpretation in Fig.~\ref{fig.split_cap}. The value $c_{r_0}^{i}$ denotes the number of distinct pairs $\left(\frac{\Vzerok}{\|\Vzerok\|_2}, \frac{\Vzerol}{\|\Vzerol\|_2}\right)$\footnote{Here, ``distinct'' means that any normalized version of $\Vzeroj$ can appear at most in one pair.} such that $\frac{\Vzerok}{\|\Vzerok\|_2}$ is in the upper half of the cap $\mathrm{E}$, and $\frac{\Vzerol}{\|\Vzerol\|_2}$ is in the lower half of the cap $\mathrm{E}$. The quantities $r_0$, $r_i$, and $\hat{r}$ can all be used as the radius of the cap $\mathrm{E}$. The ratio $\DX$ is proportional to the area of the cap $\mathrm{E}$ with radius $\hat{r}$ (or equivalently, the probability that the normalized $\Vzeroj$ falls in the cap $\mathrm{E}$).

The following lemma gives an estimation on $\|\Ht^T\bm{a}\|_2^2/p$ when $\XX$ is given. We put its proof in Appendix~\ref{subsec.proof_Ha_V0}.
\begin{lemma}\label{le.Ha_V0}
Given $\XX$, we have
\begin{align*}
    \prob_{\Vzero}\left\{\|\Ht^T\bm{a}\|_2^2\geq p \DX,\ \text{for all }\bm{a}\in\sn\right\}\geq 1-4ne^{-np\DX/6}.
\end{align*}
\end{lemma}

Notice that $\DX$ only depends on $\XX$ and it may even be zero if $\hat{r}$ is zero. However, after we introduce the randomness of $\XX$,
we can show that $\hat{r}$ is lower bounded with high probability. We can then obtain the following lemma. We put its proof in Appendix~\ref{subsec.ha}.

Define 
\begin{align}
    &C_d\defeq\frac{2\sqrt{2}}{B(\frac{d-1}{2},\frac{1}{2})}\label{eq.def_Cd},\\
    &D(n,d,\delta)\defeq \frac{1}{16n}I_{\frac{\delta^2}{n^4C_d^2}\left(1-\frac{\delta^2}{4n^4C_d^2}\right)}\left(\frac{d-1}{2},\frac{1}{2}\right).\label{eq.temp_100404}
\end{align}

\begin{lemma}\label{le.Ha}
For any $\delta\in\left(0,\frac{2}{\pi}\right]$, we have
\begin{align*}
    \prob_{\XX,\Vzero}\left\{\|\Ht^T\bm{a}\|_2^2\geq p D(n,d,\delta),\ \text{for all }\bm{a}\in\sn\right\}\geq 1-4ne^{-npD(n,d,\delta)/6}-\delta.
\end{align*}
\end{lemma}

Notice that Lemma~\ref{le.Ha} is already very close to Lemma~\ref{le.lower_eig}, and we put the final steps of the proof of Lemma~\ref{le.lower_eig} in Appendix~\ref{subsec.le_lower_eig}.

\subsection{Proof of Lemma~\ref{le.Ha_V0}}\label{subsec.proof_Ha_V0}

\begin{proof}
Define events as follows.
\begin{align*}
    &\mathcal{J}\defeq \left\{\|\Ht^T\bm{a}\|_2^2\geq p \DX,\ \text{for all }\bm{a}\in\sn\right\},\\
    &\mathcal{J}_i\defeq \left\{\text{there exists }\bm{a}\in\sn \text{ that }i\in\argmax_{j\in\{1,2,\cdots,n\}} |a_j|,\ \text{and }\|\Ht^T\bm{a}\|_2^2\leq p \DX\right\},\\
    &\mathcal{K}_i\defeq \left\{c_{r_i}^i\leq 2np \DX\right\},\text{ for }i=1,2,\cdots,n.
\end{align*}
Those definitions directly imply that
\begin{align}\label{eq.temp_112405}
    \mathcal{J}^c=\bigcup_{i=1}^n\mathcal{J}_i.
\end{align}

\noindent\textbf{Step 1: prove }$\mathcal{J}_i\subseteq \mathcal{K}_i$

To show $\mathcal{J}_i\subseteq \mathcal{K}_i$, we only need to prove that $\mathcal{J}_i$ implies $\mathcal{K}_i$. To that end, it suffices to show $\|\Ht^T\bm{a}\|_2^2\geq \frac{c_{r_i}^i}{2n}$ for the vector $\bm{a}$ defined in $\mathcal{J}_i$.
Because $i\in\argmax_{j=1}^n |a_j|$ and $\|\bm{a}\|_2=1$, we have
\begin{align}\label{eq.temp_100401}
    |a_{i}|\geq \frac{1}{\sqrt{n}}.
\end{align}
By Eq.~\eqref{eq.temp_110301}, we can construct $c_{r_i}^i$ pairs $(k_j,l_j)$ for $j=1,2,\cdots,c_{r_i}^i$ (all $k_j$'s are different and all $l_j$'s are different),  such that
\begin{align*}
    \frac{\Vzero[k_j]}{\|\Vzero[k_j]\|_2}\in \capsrpX,\quad
    \frac{\Vzero[l_j]}{\|\Vzero[l_j]\|_2}\in \capsrnX.
\end{align*}
Thus, we have
\begin{align*}
    (\Ht^T\bm{a})[k_j]-(\Ht^T\bm{a})[l_j]=&\sum_{k=1}^n a_k \left(\Ht_k[k_j]-\Ht_k[l_j]\right)\\
    =&a_{i}\left(\Ht_{i}[k_j]-\Ht_{i}[l_j]\right)+\sum_{k\in\{1,2,\cdots,n\}\setminus\{i\}}a_k \left(\Ht_k[k_j]-\Ht_k[l_j]\right)\\
    =&a_{i}\Xii\text{ (by Lemma~\ref{le.split_half})}.
\end{align*}
We then have
\begin{align*}
    \|(\Ht^T\bm{a})[k_j]\|_2^2+\|(\Ht^T\bm{a})[l_j]\|_2^2&\geq \frac{1}{2}\|a_{i}\Xii\|_2^2\text{ (by Lemma~\ref{le.l2_diff})}\\
    &\geq \frac{1}{2n}\text{ (by Eq.~\eqref{eq.temp_100401})}.
\end{align*}
Further, we have
\begin{align}\label{eq.temp_021001}
    \|\Ht^T\bm{a}\|_2^2&=\sum_{j=1}^p\|(\Ht^T\bm{a})[j]\|_2^2\geq \sum_{j=1}^{c_{r_i}^i} \|(\Ht^T\bm{a})[k_j]\|_2^2+\|(\Ht^T\bm{a})[l_j]\|_2^2=\frac{c_{r_i}^i}{2n}.
\end{align}
Clearly, if the event $\mathcal{J}_i$ occurs, then $\|\Ht\bm{a}\|_2^2 \leq p\DX$. Combining with Eq.~\eqref{eq.temp_021001}, we then have $c_{r_i}^i \leq 2np\DX $. In other words, the event $\mathcal{K}_i$ must occur. Hence, we have shown that $\mathcal{J}_i\subseteq\mathcal{K}_i$. 

\noindent\textbf{Step 2: estimate the probability of }$\mathcal{K}_i$

For all $j\in\{1,2,\cdots,p\}$, because $\Vzero[j]$ is uniformly distributed in all directions, for any fixed $0\leq r_0\leq 1$, we have
\begin{align*}
    \prob_{\Vzero}\left\{\frac{\Vzeroj}{\|\Vzeroj\|_2}\in \capsrzpX\Bigg| \ i\right\}=\frac{\musd(\capsrz)}{2\musd(\sd)}.
\end{align*}
Thus, $|\mathcal{A}_{\Vzero}\cap \capsrzpX|$ follows the distribution $\bino\left(p,\  \frac{\musd(\capsrz)}{2\musd(\sd)}\right)$ given $i\text{ and }\XX$. By Lemma~\ref{le.bino} (with $\delta=\frac{1}{2}$), we have
\begin{align}\label{eq.temp_110302}
    \prob_{\Vzero}\left\{|\mathcal{A}_{\Vzero}\cap \capsrzpX|<\frac{p\musd(\capsrz)}{4\musd(\sd)}\Big|\ i\right\}\leq 2\exp\left(-\frac{p\musd(\capsrz)}{48\musd(\sd)}\right).
\end{align}
Similarly, we have
\begin{align}\label{eq.temp_110303}
    \prob_{\Vzero}\left\{|\mathcal{A}_{\Vzero}\cap \capsrznX|<\frac{p\musd(\capsrz)}{4\musd(\sd)}\Big|\ i\right\}\leq 2\exp\left(-\frac{p\musd(\capsrz)}{48\musd(\sd)}\right).
\end{align}
By plugging Eq.~\eqref{eq.temp_110302} and Eq.~\eqref{eq.temp_110303} into Eq.~\eqref{eq.temp_110301} and applying the union bound, we have
\begin{align*}
    \prob_{\Vzero}\left\{c_{r_0}^{i}<\frac{p\musd(\capsrz)}{4\musd(\sd)}\Big|\ i\right\}\leq 4\exp\left(-\frac{p\musd(\capsrz)}{48\musd(\sd)}\right).
\end{align*}
By letting $r_0=\hat{r}$ and by Eq.~\eqref{eq.def_DX}, we thus have
\begin{align*}
    \prob_{\Vzero}\left\{c_{r_i}^i\leq 2np \DX\Big|\ i\right\} \leq 4\exp\left(-\frac{1}{6}np\DX\right),
\end{align*}
i.e.,
\begin{align}\label{eq.temp_020301}
    \prob_{\Vzero}[\mathcal{K}_i]\leq 4\exp\left(-\frac{1}{6}np\DX\right), \text{ for all }i=1,2,\cdots,n.
\end{align}

\noindent\textbf{Step3: estimate the probability of $\mathcal{J}$}

We have
\begin{align*}
    \prob_{\Vzero}[\mathcal{J}^c]\leq& \sum_{i=1}^n\prob_{\Vzero}[\mathcal{J}_i]\text{ (by Eq.~\eqref{eq.temp_112405} and the union bound)}\\
    \leq &\sum_{i=1}^n\prob_{\Vzero}[\mathcal{K}_i]\text{ (by $\mathcal{J}_i\subseteq\mathcal{K}_i$ proven in Step 1)}\\
    \leq &4n\exp\left(-\frac{1}{6}np\DX\right)\text{ (by Eq.~\eqref{eq.temp_020301})}.
\end{align*}
Thus, we have
\begin{align*}
    \prob_{\Vzero}[\mathcal{J}]=1-\prob_{\Vzero}[\mathcal{J}^c]\geq 1-4n\exp\left(-\frac{1}{6}np\DX\right).
\end{align*}
The result of this lemma thus follows.
\end{proof}

\subsection{Proof of Lemma~\ref{le.Ha}}\label{subsec.ha}
Based on Lemma~\ref{le.Ha_V0}, it remains to estimate $\hat{r}$, which will then allow us to bound $\DX$. Towards this end, we need a few lemmas to estimate $B\left(\frac{d-1}{2},\ \frac{1}{2}\right)$ and  $I_x\left(\frac{d-1}{2},\frac{1}{2}\right)$.

\begin{lemma}\label{le.lnx}
For any $x\in\mathds{R}$, we must have $x+1\leq e^x$.
\end{lemma}
\begin{proof}
Consider a function $g(x)=e^x-x-1$. It remains to show that $g(x)\geq 0$ for all $x$. We have $g'(x)=e^x-1$. In other words, $g'(x)\leq 0$ when $x\leq 0$, and $g'(x)\geq 0$ when $x\geq 0$. Thus, $g(x)$ is monotone decreasing when $x\leq 0$, and is monotone increasing when $x\geq 0$. Hence, we know that $g(x)$ achieves its minimum value at $x=0$, i.e., $g(x)\geq g(0)=0$ for any $x$. The conclusion of this lemma thus follows.
\end{proof}

\begin{lemma}\label{le.de}
For any $d\geq 5$, we have
\begin{align*}
    \left(1-\frac{1}{d-3}\right)^{d-3}\geq \frac{1}{e^2}.
\end{align*}
\end{lemma}
\begin{proof}
By letting $x=\frac{1}{d-4}$ in Lemma~\ref{le.lnx}, we have
\begin{align*}
    \frac{d-3}{d-4}=\frac{1}{d-4}+1\leq \exp\left(\frac{1}{d-4}\right),
\end{align*}
i.e.,
\begin{align}\label{eq.temp_112002}
    \frac{d-4}{d-3}\geq \exp\left(-\frac{1}{d-4}\right).
\end{align}
Thus, we have
\begin{align*}
    \left(1-\frac{1}{d-3}\right)^{d-3}&=\left(\frac{d-4}{d-3}\right)^{d-3}\\
    &\geq \exp\left(-\frac{d-3}{d-4}\right)\\
    &= \exp\left(-1-\frac{1}{d-4}\right)\\
    &\geq \exp(-2)\text{ (because $\exp(\cdot)$ is monotone increasing and $d\geq 5$)}.
\end{align*}
\end{proof}

\begin{lemma}\label{le.dm3}
For any $d\geq 5$, we must have
\begin{align*}
    \frac{2}{e}\sqrt{\frac{1}{d-3}}\geq\frac{2}{e} \frac{1}{\sqrt{d}}
\end{align*}
\end{lemma}
\begin{proof}
The result directly follows by $d-3\leq d$ when $d\geq 5$.
\end{proof}

\begin{lemma}\label{le.bound_B}
\begin{align*}
    B\left(\frac{d-1}{2},\ \frac{1}{2}\right)\in \left[\frac{2}{e}\frac{1}{\sqrt{d}},\ \pi\right].
\end{align*}
Further, if $d\geq 5$, we have
\begin{align*}
    B\left(\frac{d-1}{2},\ \frac{1}{2}\right)\in \left[\frac{2}{e}\frac{1}{\sqrt{d}},\ \frac{4}{\sqrt{d-3}}\right].
\end{align*}
\end{lemma}
\begin{proof}
When $d=2$, we have $B\left(\frac{d-1}{2},\ \frac{1}{2}\right)=\pi$. When $d=3$, we have $B\left(\frac{d-1}{2},\ \frac{1}{2}\right)=2$. When $d=4$, we have $B\left(\frac{d-1}{2},\ \frac{1}{2}\right)\approx 1.57$. It is easy to verify that the statement of the lemma holds for $d=2$, $3$, and $4$.  It remains to validate the case of $d\geq 5$. 
We first prove the lower bound. For any $m\in(0, 1)$, we have
\begin{align*}
    B\left(\frac{d-1}{2},\ \frac{1}{2}\right)=&\int_0^1 t^{\frac{d-3}{2}}(1-t)^{-\frac{1}{2}}dt\\
    \geq & \int_{m}^1 t^{\frac{d-3}{2}}(1-t)^{-\frac{1}{2}}dt\text{ (because $t^{\frac{d-3}{2}}(1-t)^{-\frac{1}{2}}\geq 0$)}\\
    \geq & m^{\frac{d-3}{2}}\int_m^1(1-t)^{-\frac{1}{2}}dt\\
    &\text{ (because $t^{\frac{d-3}{2}}$ is monotone increasing with respect to $t$ when $d\geq 5$)}\\
    =&m^{\frac{d-3}{2}}\left(-2\sqrt{1-t}\ \bigg|_m^1\right)\\
    =&m^{\frac{d-3}{2}}\cdot 2\sqrt{1-m}.
\end{align*}
By letting $m=1-\frac{1}{d-3}$, we thus have
\begin{align*}
    B\left(\frac{d-1}{2},\ \frac{1}{2}\right)\geq \left(1-\frac{1}{d-3}\right)^{\frac{d-3}{2}}\cdot 2\sqrt{\frac{1}{d-3}}.
\end{align*}
Then, applying Lemma~\ref{le.de}, we have
\begin{align*}
    B\left(\frac{d-1}{2},\ \frac{1}{2}\right)\geq \frac{2}{e}\cdot\sqrt{\frac{1}{d-3}}.
\end{align*}
Thus, by Lemma~\ref{le.dm3}, we have
\begin{align*}
    B\left(\frac{d-1}{2},\ \frac{1}{2}\right)\geq \frac{2}{e}\frac{1}{\sqrt{d}}.
\end{align*}
Now we prove the upper bound. For any $m\in(0, 1)$, we have
\begin{align*}
    B\left(\frac{d-1}{2},\ \frac{1}{2}\right)=&\int_0^1 t^{\frac{d-3}{2}}(1-t)^{-\frac{1}{2}}dt\\
    =&\int_0^m t^{\frac{d-3}{2}}(1-t)^{-\frac{1}{2}}dt+\int_m^1 t^{\frac{d-3}{2}}(1-t)^{-\frac{1}{2}}dt\\
    \leq &\int_0^m t^{\frac{d-3}{2}}(1-m)^{-\frac{1}{2}}dt+\int_m^1 (1-t)^{-\frac{1}{2}}dt\\
    =&\frac{2}{d-1}m^{\frac{d-1}{2}}(1-m)^{-\frac{1}{2}}+2\sqrt{1-m}\\
    \leq &\frac{2}{d-1}(1-m)^{-\frac{1}{2}}+2\sqrt{1-m}\text{ (because $m<1$ and $d\geq 5$)}.
\end{align*}
By letting $m=1-\frac{1}{d-3}$, we thus have
\begin{align*}
    B\left(\frac{d-1}{2},\ \frac{1}{2}\right)
    \leq&\frac{2\sqrt{d-3}}{d-1}+\frac{2}{\sqrt{d-3}}\\
    \leq &\frac{4}{\sqrt{d-3}}.
\end{align*}
Notice that $\frac{4}{\sqrt{5-3}}=2\sqrt{2}<\pi$. The result of this lemma thus follows.

\end{proof}

\begin{lemma}\label{le.temp_112401}
Recall $C_d$ is defined in Eq.~\eqref{eq.def_Cd}. If $d\leq n^4$ and $\delta\leq 1$, then
\begin{align*}
    \left(1-\frac{\delta^2}{4n^4 C_d^2}\right)^{\frac{d-1}{2}}\geq \frac{1}{2}.
\end{align*}
\end{lemma}
\begin{proof}
We have
\begin{align*}
    \left(1-\frac{\delta^2}{4n^4 C_d^2}\right)^{\frac{d-1}{2}}\geq &\left(1-\frac{\delta^2}{4n^4 C_d^2}\right)^{d-1}\\
    \geq & 1-\frac{(d-1)\delta^2}{4n^4 C_d^2}\text{ (by Bernoulli's inequality $(1+x)^a\geq 1+ax$)}\\
    = & 1-\frac{(d-1)\left(B\left(\frac{d-1}{2},\ \frac{1}{2}\right)\right)^2}{4n^4 \cdot 8} \text{ (by $\delta\leq 1$ and Eq.~\eqref{eq.def_Cd})}\\
    \geq & 1-\frac{(d-1)\pi^2}{32n^4}\text{ (by Lemma~\ref{le.bound_B})}\\
    \geq & 1-\frac{d}{n^4}\cdot\frac{\pi^2}{32}\\
    \geq &\frac{1}{2}\text{ (because $n^4\geq d$ and $\pi\leq 4$)}.
\end{align*}
\end{proof}

\begin{lemma}\label{le.bound_b}
For any $\delta\in\left(0,\ \frac{2}{\pi}\right]$, we must have $\frac{\delta}{n^2 C_d}\leq \frac{1}{\sqrt{2}}$.
\end{lemma}
\begin{proof}
Because Eq.~\eqref{eq.def_Cd}, $\delta\leq \frac{2}{\pi}$, and $n\geq 1$, this lemma directly follows by Lemma~\ref{le.bound_B}.
\end{proof}

\begin{lemma}\label{le.estimate_Ix}
For any $x\in[0,\ 1]$, we must have
\begin{align*}
    I_x\left(\frac{d-1}{2},\frac{1}{2}\right)\geq \frac{C_d}{\sqrt{2}(d-1)}x^{\frac{d-1}{2}},
\end{align*}
and
\begin{align*}
    \lim_{x\to 0}\frac{I_x\left(\frac{d-1}{2},\frac{1}{2}\right)}{x^{\frac{d-1}{2}}}=\frac{C_d}{\sqrt{2}(d-1)}.
\end{align*}
\end{lemma}
\begin{proof}
we have
\begin{align*}
    I_x\left(\frac{d-1}{2},\frac{1}{2}\right)=&\frac{\int_0^x t^{\frac{d-3}{2}}(1-t)^{-\frac{1}{2}}dt}{B\left(\frac{d-1}{2},\frac{1}{2}\right)}\nonumber\\
    =&\frac{C_d}{2\sqrt{2}}\int_0^x t^{\frac{d-3}{2}}(1-t)^{-\frac{1}{2}}dt\text{ (by Eq.~\eqref{eq.def_Cd})}\\
    \in &\left[\frac{C_d}{2\sqrt{2}}\int_0^x t^{\frac{d-3}{2}}dt,\ \frac{C_d}{2\sqrt{2}\sqrt{1-x}}\int_0^x t^{\frac{d-3}{2}}dt\right]\\
    &\text{ (because $(1-t)^{-1/2}\in\left[1, \frac{1}{\sqrt{1-x}}\right]$)}\nonumber\\
    \in&\left[\frac{C_d}{\sqrt{2}(d-1)}x^{\frac{d-1}{2}},\ \frac{C_d}{\sqrt{2}(d-1)\sqrt{1-x}}x^{\frac{d-1}{2}}\right].
\end{align*}
Thus, we have
\begin{align*}
    \frac{I_x\left(\frac{d-1}{2},\frac{1}{2}\right)}{x^{\frac{d-1}{2}}}\in\left[\frac{C_d}{\sqrt{2}(d-1)},\ \frac{C_d}{\sqrt{2}(d-1)\sqrt{1-x}}\right],
\end{align*}
which implies
\begin{align*}
    \lim_{x\to 0}\frac{I_x\left(\frac{d-1}{2},\frac{1}{2}\right)}{x^{\frac{d-1}{2}}}=\frac{C_d}{\sqrt{2}(d-1)}.
\end{align*}
\end{proof}


\begin{lemma}\label{le.beta_bound}
For any $x\in \left[\frac{1}{2},\ 1\right)$ and for any $d\in\{2,3,\cdots\}$, we have
\begin{align*}
    I_{x}\left(\frac{d-1}{2},\frac{1}{2}\right)\geq 1-\frac{2\sqrt{2(1-x)}}{B\left(\frac{d-1}{2},\frac{1}{2}\right)}.
\end{align*}
We also have
\begin{align*}
    \lim_{(1-x)\to 0^+}\frac{1-I_{x}\left(\frac{d-1}{2},\frac{1}{2}\right)}{\sqrt{1-x}}=\frac{2}{B\left(\frac{d-1}{2},\frac{1}{2}\right)}.
\end{align*}
\end{lemma}
\begin{proof}
By the definition of regularized incomplete beta function in Eq.~\eqref{eq.def_reg_incomplete_beta}, we have
\begin{align*}
    I_x\left(\frac{d-1}{2},\frac{1}{2}\right)=\frac{\int_0^x t^{\frac{d-1}{2}-1}(1-t)^{-\frac{1}{2}}dt}{B\left(\frac{d-1}{2},\frac{1}{2}\right)}=1-\frac{\int_x^1 t^{\frac{d-3}{2}}(1-t)^{-\frac{1}{2}}dt}{B\left(\frac{d-1}{2},\frac{1}{2}\right)}.
\end{align*}
Thus, it remains to show that
\begin{align}
    &\int_x^1 t^{\frac{d-3}{2}}(1-t)^{-\frac{1}{2}}dt\leq 2\sqrt{2(1-x)},\text{ and}\label{eq.temp_112403}\\
    &\lim_{(1-x)\to 0^+}\frac{\int_x^1 t^{\frac{d-3}{2}}(1-t)^{-\frac{1}{2}}dt}{\sqrt{1-x}}=2.\label{eq.temp_112404}
\end{align}
First, we prove Eq.~\eqref{eq.temp_112403}.
Case 1: $d=2$. We have
\begin{align*}
    &\int_x^1 t^{\frac{d-3}{2}}(1-t)^{-\frac{1}{2}}dt\\
    =&\int_x^1 t^{-\frac{1}{2}}(1-t)^{-\frac{1}{2}}dt\\
    \leq & \frac{1}{\sqrt{x}}\int_x^1 (1-t)^{-\frac{1}{2}}dt\text{ (because $t^{-\frac{1}{2}}$ is monotone decreasing in $[x,\ 1]$)}\\
    =&2\sqrt{\frac{1-x}{x}}\\
    \leq & 2\sqrt{2(1-x)}\text{ (because $x\geq\frac{1}{2}$)}.
\end{align*}
Case 2: $d\geq 3$. Then $t^{\frac{d-3}{2}}$ is monotone increasing in $[x,\ 1]$. Thus, we have
\begin{align*}
    \int_x^1 t^{\frac{d-3}{2}}(1-t)^{-\frac{1}{2}}dt
    \leq \int_x^1 (1-t)^{-\frac{1}{2}}dt=2\sqrt{1-x}\leq 2\sqrt{2(1-x)}.
\end{align*}
To conclude, for all $d\in\{2,3,\cdots\}$, Eq.~\eqref{eq.temp_112403} holds. 

Second, we prove Eq.~\eqref{eq.temp_112404}. We have
\begin{align*}
    \frac{\int_x^1 t^{\frac{d-3}{2}}(1-t)^{-\frac{1}{2}}dt}{\sqrt{1-x}}\in&\left[\frac{\min\{1,\ x^{\frac{d-3}{2}}\}\int_x^1 (1-t)^{-\frac{1}{2}}dt}{\sqrt{1-x}},\ \frac{\max\{1,\ x^{\frac{d-3}{2}}\}\int_x^1 (1-t)^{-\frac{1}{2}}dt}{\sqrt{1-x}}\right]\\
    =&\left[2\min\{1,\ x^{\frac{d-3}{2}}\},\ 2\max\{1,\ x^{\frac{d-3}{2}}\}\right].
\end{align*}
Since $\lim_{x\to 1}x^{\frac{d-3}{2}}=1$, Eq.~\eqref{eq.temp_112404} thus follows.
\end{proof}

Now we are ready to prove Lemma~\ref{le.Ha}.

Recall $\vstari$ defined in Eq.~\eqref{eq.temp_030403}. For any $b\in \left(0, \frac{1}{\sqrt{2}}\right]$, we have, for $\xsmall$ independent of $\vstari$ and with distribution $\xdensity$,
\begin{align}
    \prob_{\xsmall\sim \xdensity}\left\{|\vstari^T\xsmall|\geq b\right\}&=I_{1-b^2}\left(\frac{d-1}{2},\frac{1}{2}\right)\text{ (because Lemma~\ref{le.innerProd_I})}\nonumber\\
    &\geq1-\frac{2\sqrt{2\left(1-(1-b^2)\right)}}{B\left(\frac{d-1}{2},\frac{1}{2}\right)}\text{ (by Lemma~\ref{le.beta_bound})}\nonumber\\
    &= 1-C_d b\text{ (by the definition of $C_d$ in Eq.~\eqref{eq.def_Cd})}.\label{eq.temp_110202}
\end{align}
Since each of the $\Xjj$, $j\neq i$, is independent of $\vstari$, we have
\begin{align*}
    &\prob_{\XX}\left\{\min_{j\in\{1,2,\cdots,n\}\setminus\{i\}}|\vstari^T\Xjj|\geq b\right\}\\
    =&\left(\prob_{\xsmall\sim \xdensity}\left\{|\vstari^T\xsmall|\geq b\right\}\right)^{n-1}\text{ (because each $\Xjj$, $j\neq i$,  is \emph{i.i.d.} and independent of $\vstari$)}\\
    \geq& \left(1-C_d b\right)^{n-1}\text{ (by Eq.~\eqref{eq.temp_110202})}\\
    \geq & 1-(n-1)C_d b\text{ (by Bernoulli's inequality)}\\
    \geq & 1- nC_d b.
\end{align*}
Or, equivalently,
\begin{align}\label{eq.temp_110203}
    \prob_{\XX}\left\{\min_{i\in\{1,2,\cdots,n\}\setminus\{i\}}|\vstari^T\Xii|< b\right\}\leq nC_d b.
\end{align}
Recall the definition of $r_i$ and $\hat{r}$ in Eqs.~\eqref{eq.temp_110201}\eqref{eq.def_hat_r}.
Thus, we then have
\begin{align}
    &\prob_{\XX,\Vzero}\left\{\hat{r}< \frac{\delta}{n^2C_d}\right\}\nonumber\\
    \leq &n\prob_{\XX,\Vzero}\left\{r_i< \frac{\delta}{n^2C_d}\right\}\text{ (by Eq.~\eqref{eq.def_hat_r} and the union bound)}\nonumber\\
   =&n\prob_{\XX}\left\{r_i< \frac{\delta}{n^2C_d}\right\}\text{ (because $r$ is independent of $\Vzero$)}\nonumber\\
    =&n\prob_{\XX}\left\{\min_{j\in\{1,2,\cdots,n\}\setminus\{i\}}\left|\vstari^T\Xjj\right|< \frac{\delta}{n^2C_d}\right\}\text{ (by Eq.~\eqref{eq.temp_110201})}\nonumber\\
    \leq &n\cdot n C_d\cdot \frac{\delta}{n^2C_d}\text{ (by letting $b=\frac{\delta}{n^2C_d}$ in Eq.~\eqref{eq.temp_110203} and $b\leq \frac{1}{\sqrt{2}}$ because of Lemma~\ref{le.bound_b})}\nonumber\\
    =&\delta.\label{eq.temp_100403}
\end{align}
By Lemma~\ref{le.cap_area} and Eq.~\eqref{eq.temp_100404}, we have
\begin{align*}
    \musd(\capsdelta)=\frac{1}{2}\musd(\sd)I_{\frac{\delta^2}{n^4C_d^2}(1-\frac{\delta^2}{4n^4C_d^2})}\left(\frac{d-1}{2},\frac{1}{2}\right)=8n\musd(\sd)D(n,d,\delta).
\end{align*}
Thus, we have
\begin{align}\label{eq.temp_110307}
    D(n,d,\delta)=\frac{\musd(\capsdelta)}{8n\musd(\sd)}.
\end{align}
By Eq.~\eqref{eq.def_DX} and Eq.~\eqref{eq.temp_110307}, we have
\begin{align*}
    \DX\geq D(n,d,\delta),\text{ when }\hat{r}\geq \frac{\delta}{n^2C_d}.
\end{align*}
Notice that $\hat{r}$ only depends on $\XX$ and is independent of $\Vzero$. By Lemma~\ref{le.Ha_V0}, for any $\XX$ that makes $\hat{r}\geq \frac{\delta}{n^2C_d}$, we must have
\begin{align*}
    \prob_{\Vzero}\left\{\|\Ht^T\bm{a}\|_2^2\geq p D(n,d,\delta),\ \text{for all }\bm{a}\in\sn\right\}\geq 1-4ne^{-npD(n,d,\delta)/6}.
\end{align*}
In other words,
\begin{align*}
    \prob_{\Vzero}\left\{\|\Ht^T\bm{a}\|_2^2\geq p D(n,d,\delta),\ \text{for all }\bm{a}\in\sn\ \bigg| \ \text{any given }\XX \text{ such that }\hat{r}\geq \frac{\delta}{n^2C_d}\right\}\geq 1-4ne^{-npD(n,d,\delta)/6}.
\end{align*}
We thus have
\begin{align}\label{eq.temp_020302}
    \prob_{\XX,\Vzero}\left\{\|\Ht^T\bm{a}\|_2^2\geq p D(n,d,\delta),\ \text{for all }\bm{a}\in\sn\ \bigg| \ \hat{r}\geq \frac{\delta}{n^2C_d}\right\}\geq 1-4ne^{-npD(n,d,\delta)/6}.
\end{align}
Thus, we have
\begin{align*}
    &\prob_{\XX,\Vzero}\left\{\|\Ht^T\bm{a}\|_2^2\geq p D(n,d,\delta),\ \text{for all }\bm{a}\in\sn\right\}\\
    \geq &\prob_{\XX,\Vzero}\left\{\hat{r}\geq \frac{\delta}{n^2C_d},\text{ and }\|\Ht^T\bm{a}\|_2^2\geq p D(n,d,\delta),\ \text{for all }\bm{a}\in\sn\right\}\\
    =&\prob_{\XX,\Vzero}\left\{\|\Ht^T\bm{a}\|_2^2\geq pD(n,d,\delta),\ \text{for all }\bm{a}\in\sn\ \bigg|\ \hat{r}\geq \frac{\delta}{n^2C_d}\right\}\cdot \prob_{\XX,\Vzero}\left\{\hat{r}\geq \frac{\delta}{n^2C_d}\right\}\\
    \geq & (1 -4ne^{-npD(n,d,\delta)/6})(1- \delta)\text{ (by Eq.~\eqref{eq.temp_100403} and Eq.~\eqref{eq.temp_020302})}\\
    \geq & 1-4ne^{-npD(n,d,\delta)/6}-\delta.
\end{align*}
The result of this lemma thus follows.

\subsection{Proof of Lemma~\ref{le.lower_eig}}\label{subsec.le_lower_eig}
Based on Lemma~\ref{le.Ha}, it only remains to estimate $D(n,d,\delta)$. We start with a lemma.
\begin{lemma}\label{le.bound_D}
If $\delta\leq 1$ and $d\leq n^4$, we must have
\begin{align}\label{eq.temp_120901}
    D(n,d,\delta)\geq2^{-2d-5.5}d^{-0.5d}n^{-2d+1}\delta^{d-1}.
\end{align}
For any given $\delta\geq 0$ and $d$, we must have
\begin{align*}
    \lim_{n\to\infty}\frac{D(n,d,\delta)}{n^{2d-1}}=2^{-1.5d-1.5}\left(B\left(\frac{d-1}{2},\ \frac{1}{2}\right)\right)^{d-2}\frac{1}{d-1}\delta^{d-1}.
\end{align*}
\end{lemma}
\begin{proof}
We start from
\begin{align}
    \frac{1}{(d-1)C_d^{d-2}}=&\frac{\left(B\left(\frac{d-1}{2},\ \frac{1}{2}\right)\right)^{d-2}}{(d-1)(2\sqrt{2})^{d-2}}\text{ (by Eq.~\eqref{eq.def_Cd})}\nonumber\\
    \geq & \frac{1}{(d-1)d^{\frac{d}{2}-1}(e\sqrt{2})^{d-2}}\text{ (by Lemma~\ref{le.bound_B})}\nonumber\\
    \geq &\frac{1}{d^{\frac{d}{2}}(e\sqrt{2})^d}\nonumber\\
    =&(16d)^{-\frac{d}{2}}\text{  (since $2e^2\approx 14.778\leq 16$)}.\label{eq.temp_112402}
\end{align}
Thus, we have
\begin{align*}
    D(n,d,\delta)\geq& \frac{1}{16n}\frac{C_d}{\sqrt{2}(d-1)}\left(\frac{\delta^2}{n^4C_d^2}\left(1-\frac{\delta^2}{4n^4C_d^2}\right)\right)^{\frac{d-1}{2}}\text{ (by Eq.~\eqref{eq.temp_100404} and Lemma~\ref{le.estimate_Ix})}\\
    =&\frac{1}{16\sqrt{2}}\frac{1}{(d-1)C_d^{d-2}}\left(1-\frac{\delta^2}{4n^4C_d^2}\right)^{\frac{d-1}{2}}\frac{\delta^{d-1}}{n^{2d-1}}\\
    \geq &\frac{1}{32\sqrt{2}}(16d)^{-\frac{d}{2}}\frac{\delta^{d-1}}{n^{2d-1}}\text{ (by Lemma~\ref{le.temp_112401} and Eq.~\eqref{eq.temp_112402})}\\
    =&2^{-2d-5.5}d^{-0.5d}n^{-2d+1}\delta^{d-1}.
\end{align*}
For any given $d$ and $\delta\geq 0$, we have
\begin{align*}
    \lim_{n\to\infty}\frac{D(n,d,\delta)}{n^{2d-1}}=&\lim_{n\to\infty}\frac{1}{16n^{2d-2}}I_{\frac{\delta^2}{n^4C_d^2}\left(1-\frac{\delta^2}{4n^4C_d^2}\right)}\left(\frac{d-1}{2},\frac{1}{2}\right)\text{ (by Eq.~\eqref{eq.temp_100404})}\\
    =&\lim_{n\to\infty}\frac{\left(\frac{\delta^2}{n^4C_d^2}\left(1-\frac{\delta^2}{4n^4C_d^2}\right)\right)^{\frac{d-1}{2}}}{16n^{2d-2}}\cdot\frac{I_{\frac{\delta^2}{n^4C_d^2}\left(1-\frac{\delta^2}{4n^4C_d^2}\right)}\left(\frac{d-1}{2},\frac{1}{2}\right)}{\left(\frac{\delta^2}{n^4C_d^2}\left(1-\frac{\delta^2}{4n^4C_d^2}\right)\right)^{\frac{d-1}{2}}}\\
    =&\frac{1}{16}\lim_{n\to\infty}\left(\frac{\delta^2}{C_d^2}\left(1-\frac{\delta^2}{4n^4C_d^2}\right)\right)^{\frac{d-1}{2}}\cdot\frac{I_{\frac{\delta^2}{n^4C_d^2}\left(1-\frac{\delta^2}{4n^4C_d^2}\right)}\left(\frac{d-1}{2},\frac{1}{2}\right)}{\left(\frac{\delta^2}{n^4C_d^2}\left(1-\frac{\delta^2}{4n^4C_d^2}\right)\right)^{\frac{d-1}{2}}}\\
    =&\frac{1}{16}\lim_{n\to\infty}\left(\frac{\delta^2}{C_d^2}\left(1-\frac{\delta^2}{4n^4C_d^2}\right)\right)^{\frac{d-1}{2}}\cdot\lim_{n\to\infty}\frac{I_{\frac{\delta^2}{n^4C_d^2}\left(1-\frac{\delta^2}{4n^4C_d^2}\right)}\left(\frac{d-1}{2},\frac{1}{2}\right)}{\left(\frac{\delta^2}{n^4C_d^2}\left(1-\frac{\delta^2}{4n^4C_d^2}\right)\right)^{\frac{d-1}{2}}}\\
    =&\frac{1}{16}\frac{\delta^{d-1}}{C_d^{d-1}}\frac{C_d}{\sqrt{2}(d-1)}\text{ (by Lemma~\ref{le.estimate_Ix})}\\
    =&2^{-1.5d-1.5}\left(B\left(\frac{d-1}{2},\ \frac{1}{2}\right)\right)^{d-2}\frac{1}{d-1}\delta^{d-1}\text{ (by Eq.~\eqref{eq.def_Cd})}.
\end{align*}
\end{proof}


Now we are ready to finish our proof of Lemma~\ref{le.lower_eig}.

We have
\begin{align*}
    D(n,d,\delta)\bigg|_{\delta=\frac{1}{\sqrt[m]{n}}}\geq&\frac{1}{2^{2d+5.5}d^{0.5d}n^{2d-1}n^{\frac{d-1}{m}}}\text{ (by Eq.~\eqref{eq.temp_120901})}\\
    =&\frac{1}{2^{2d+5.5}d^{0.5d}n^{\left(2+\frac{1}{m}\right)(d-1)}n}\\
    =&\frac{1}{J_m(n,d)n}\text{ (by Eq.~\eqref{eq.def_Jmnd})}.
\end{align*}
Thus, when $p\geq 6J_m(n,d)\ln\left(4n^{1+\frac{1}{m}}\right)$, we have
\begin{align*}
    1-4ne^{-npD(n,d,\delta)/6}-\delta\bigg|_{\delta=\frac{1}{\sqrt[m]{n}}}\geq & 1-\frac{2}{\sqrt[m]{n}}.
\end{align*}

Then, we have
\begin{align*}
    &m\in\left[1,\ \frac{\ln n}{\ln \frac{\pi}{2}}\right]
    \implies  \left(\frac{\pi}{2}\right)^{m}\leq n
    \implies  n^{\frac{1}{m}}\geq \frac{\pi}{2}
    \implies  \frac{1}{\sqrt[m]{n}} \leq \frac{2}{\pi}
    \implies  \delta \leq \frac{2}{\pi}.
\end{align*}

By Lemma~\ref{le.min_eig} and Lemma~\ref{le.Ha}, the conclusion of Lemma~\ref{le.lower_eig} thus follows.

\section{Upper bound of \texorpdfstring{${\min\eig\left(\Ht\Ht^T\right)}/{p}$}{min-eigenvalue} }\label{app.upper_bound_eig}

By Lemma~\ref{le.min_eig}, to get an upper bound of ${\min\eig\left(\Ht\Ht^T\right)}/{p}$, it is equivalent to get an upper bound of $\min_{\bm{a}\in\sn}\|\Ht^T\bm{a}\|_2^2/p$. To that end, we only need to construct a vector $\bm{a}$ and calculate the value of $\|\Ht^T\bm{a}\|_2^2/p$, which automatically becomes an upper bound $\min_{\bm{a}\in\sn}\|\Ht^T\bm{a}\|_2^2/p$.

The following lemma shows that, for given $\XX$, if two input training data $\Xii$ and $\Xkk$ are close to each other, then $\min_{\bm{a}\in\sn}\|\Ht^T\bm{a}\|_2^2/p$ is unlikely to be large.
\begin{lemma}\label{le.temp_122402}
If there exist $\Xii$ and $\Xkk$ such that $i\neq k$ and $\theta \defeq\arccos(\Xii^T\Xkk)$, then
\begin{align*}
    \prob_{\Vzero}\left\{\min_{\bm{a}\in\sn}\|\Ht^T\bm{a}\|_2^2\geq \frac{3p\theta^2}{8}+\frac{3p\theta}{4\pi}\right\}\leq 2\exp \left(-\frac{p}{24}\right)+2\exp\left(-\frac{p\theta}{12}\right).
\end{align*}
\end{lemma}
Intuitively, Lemma~\ref{le.temp_122402} is true because, when $\Xii$ and $\Xkk$ are similar, $\Ht_i$ and $\Ht_k$ (the $i$-th and $k$-th row of $\Ht$, respectively) will also likely be similar, i.e., $\|\Ht_i-\Ht_k\|_2$ is not likely to be large. Thus, we can construct $\bm{a}$ such that $\Ht^T\bm{a}$ is proportional to $\Ht_i-\Ht_k$, which will lead to the result of Lemma~\ref{le.temp_122402}. We put the proof of Lemma~\ref{le.temp_122402} in Appendix~\ref{subsec.upper_fixed_V0}.

The next step is to estimate such difference between $\Xii$ and $\Xkk$ (or equivalently, the angle $\theta$ between them). We have the following lemma.
\begin{lemma}\label{le.temp_122401}
When $n\geq \pi(d-1)$, there must exist two different $\Xii$'s such that the angle between them is at most
\begin{align*}
    \theta = \pi \left((d-1)B(\frac{d-1}{2},\frac{1}{2})\right)^{\frac{1}{d-1}}n^{-\frac{1}{d-1}}.
\end{align*}
\end{lemma}
Lemma~\ref{le.temp_122401} is intuitive because $\sd$ has limited area. When there are many $\Xii$'s on $\sd$, there must exist at least two $\Xii$'s that are relatively close. We put the proof of Lemma~\ref{le.temp_122401} in Appendix~\ref{subsec.pigeonhole}.

Finally, we have the following lemma.
\begin{lemma}\label{le.temp_122403}
When $n\geq \pi(d-1)$, we have
\begin{align*}
    &\prob_{\Vzero,\XX}\Big\{\frac{\min\eig (\Ht\Ht^T)}{p}\leq \frac{3\pi^2}{8}\left((d-1)B(\frac{d-1}{2},\frac{1}{2})\right)^{\frac{2}{d-1}}n^{-\frac{2}{d-1}}\\
    &\qquad +\frac{3}{4}\left((d-1)B(\frac{d-1}{2},\frac{1}{2})\right)^{\frac{1}{d-1}}n^{-\frac{1}{d-1}}\Big\}\\
    \geq& 1 - 2\exp\left(-\frac{p}{24}\right)-2\exp\left(-\frac{p}{12}\pi \left((d-1)B(\frac{d-1}{2},\frac{1}{2})\right)^{\frac{1}{d-1}}n^{-\frac{1}{d-1}}\right).
\end{align*}
\end{lemma}
\begin{proof}
This lemma directly follows by combining Lemma~\ref{le.min_eig}, Lemma~\ref{le.temp_122402}, and Lemma~\ref{le.temp_122401}.
\end{proof}
By Lemma~\ref{le.temp_122403}, we can conclude that when $p$ is much larger than $n^{\frac{1}{d-1}}$, $\frac{\min\eig (\Ht\Ht^T)}{p}=O(n^{-\frac{1}{d-1}})$ with high probability.

\subsection{Proof of Lemma~\ref{le.temp_122402}}\label{subsec.upper_fixed_V0}
We first prove a useful lemma.
\begin{lemma}\label{le.sin}
For any $\varphi \in[0,2\pi]$, we must have $\sin\varphi \leq \varphi$. For any $\varphi \in[0,\pi/2]$, we must have $\varphi \leq \frac{\pi}{2}\sin\varphi$.
\end{lemma}
\begin{proof}
To prove the first part of the lemma, note that
\begin{align*}
    \frac{d(\varphi-\sin\varphi)}{d\varphi}=1-\cos\varphi \geq 0.
\end{align*}
Thus, the function $(\varphi-\sin\varphi)$ is monotone increasing with respect to  $\varphi \in[0,2\pi]$. Thus, we have
\begin{align*}
    \min_{\varphi \in[0,2\pi]}(\varphi-\sin\varphi)=(\varphi-\sin\varphi)\big|_{\varphi=0}=0.
\end{align*}
In other words, we have $\sin\varphi \leq \varphi$ for any $\varphi \in[0,2\pi]$.

To prove the second part of the lemma, note that when $\varphi\in[0,\pi/2]$, we have
\begin{align*}
    \frac{d^2(\varphi-\frac{\pi}{2}\sin\varphi)}{d\varphi^2}=\frac{\pi}{2}\sin\varphi\geq 0.
\end{align*}
Thus, the function $\varphi-\frac{\pi}{2}\sin\varphi$ is convex with respect to $\varphi\in[0, \pi/2]$. Because the maximum of a convex function must be attained at the endpoint of the domain interval, we have
\begin{align*}
    \max_{\varphi \in[0,\pi/2]}(\varphi-\frac{\pi}{2}\sin\varphi)=\max_{\varphi\in\{0,\pi/2\}}(\varphi-\frac{\pi}{2}\sin\varphi)= 0.
\end{align*}
Thus, we have $\varphi\leq \frac{\pi}{2}\sin\varphi$ for any $\varphi \in[0,\pi/2]$.
\end{proof}
Now we are ready to prove Lemma~\ref{le.temp_122402}.
\begin{proof}
Through the proof, we fix $\Xii$ and $\Xkk$, and only consider the randomness of $\Vzero$.
Because $\theta$ is the angle between $\Xii$ and $\Xkk$ and because of Assumption~\ref{as.normalize}, we have
\begin{align}
    \|\Xii-\Xkk\|_2=&2\sin\frac{\theta}{2}\nonumber\\
    \leq& 2\cdot \frac{\theta}{2}\text{ (by Lemma~\ref{le.sin})}\nonumber\\
    =&\theta.\label{eq.temp_122201}
\end{align}
Let $\bm{a}=\frac{1}{\sqrt{2}}(\bm{e}_i-\bm{e}_k)$, where $\bm{e}_q$ denotes the $q$-th standard basis vector, $q=1,2,\cdots,n$. Then, we have
\begin{align}
    \|\Ht^T\bm{a}\|_2^2=&\frac{1}{2}\|\Ht_i^T-\Ht_k^T\|_2^2\nonumber\\
    =&\frac{1}{2}\sum_{j=1}^p\left\|\bm{1}_{\{\Xii^T\Vzeroj>0\}}\Xii-\bm{1}_{\{\Xkk^T\Vzeroj>0\}}\Xkk \right\|_2^2\text{ (by Eq.~\eqref{eq.def_hx})}\nonumber\\
    =&\frac{1}{2}\sum_{j=1}^p \left(\bm{1}_{\{\Xii^T\Vzeroj>0,\ \Xkk^T\Vzeroj>0\}}\|\Xii-\Xkk\|_2^2+\bm{1}_{\{(\Xii^T\Vzeroj)(\Xkk^T\Vzeroj)<0\}}\right)\text{ (by $\|\Xii\|_2^2=\|\Xkk\|_2^2=1$)}\nonumber\\
    \leq &\frac{1}{2}\sum_{j=1}^p \left(\bm{1}_{\{\Xii^T\Vzeroj>0,\ \Xkk^T\Vzeroj>0\}}\theta^2+\bm{1}_{\{(\Xii^T\Vzeroj)(\Xkk^T\Vzeroj)<0\}}\right)\text{ (by Eq.~\eqref{eq.temp_122201})}\nonumber\\
    \leq &\frac{\theta^2}{2}\sum_{j=1}^p\bm{1}_{\{\Xii^T\Vzeroj>0\}}+\frac{1}{2}\sum_{j=1}^p\bm{1}_{\{(\Xii^T\Vzeroj)(\Xkk^T\Vzeroj)<0\}}.\label{eq.temp_122202}
\end{align}
Since $\Xii$ is fixed and the direction of $\Vzeroj$ is uniformly distributed, we have $\prob_{\Vzero}\{\Xii^T\Vzeroj>0\}=\frac{1}{2}$ and
\begin{align*}
    \prob_{\Vzero}\{(\Xii^T\Vzeroj)(\Xkk^T\Vzeroj)<0\}=&2\prob_{\Vzero}\{\Xii^T\Vzeroj>0,\ \Xkk^T\Vzeroj<0\}\\
    =&2\prob_{\Vzero}\{\Xii^T\Vzeroj>0,\ -\Xkk^T\Vzeroj>0\}\\
    =&2\int_{\sd}\bm{1}_{\{\Xii^T\bm{v}>0,\ -\Xkk^T\bm{v}>0\}}d\nvdensity(\bm{v})\\
    =&2\cdot\frac{\pi-(\pi-\theta)}{2\pi}\text{ (by Lemma~\ref{le.spherePortion})}\\
    =&\frac{\theta}{\pi}.
\end{align*}
Thus, based on the randomness of $\Vzero$, when $\XX$ are given, we have
\begin{align*}
    &\sum_{j=1}^p\bm{1}_{\{\Xii^T\Vzeroj>0\}}\sim \bino\left(p,\ \frac{1}{2}\right),\\
    &\sum_{j=1}^p\bm{1}_{\{(\Xii^T\Vzeroj)(\Xkk^T\Vzeroj)<0\}}\sim \bino\left(p,\ \frac{\theta}{\pi}\right).
\end{align*}
By letting $\delta=\frac{1}{2}$, $a=p$, $b=\frac{1}{2}$ in Lemma~\ref{le.bino}, we then have
\begin{align}
    &\prob_{\Vzero}\left\{\sum_{j=1}^p\bm{1}_{\{\Xii^T\Vzeroj>0\}}\geq\frac{3p}{4}\right\}\leq 2\exp \left(-\frac{p}{24}\right),\label{eq.temp_122203}\\
    &\prob_{\Vzero}\left\{\sum_{j=1}^p\bm{1}_{\{(\Xii^T\Vzeroj)(\Xkk^T\Vzeroj)<0\}}\geq \frac{3p\theta}{2\pi}\right\}\leq 2\exp\left(-\frac{p\theta}{12\pi}\right).\label{eq.temp_122204}
\end{align}
Thus, we have
\begin{align*}
    &\prob_{\Vzero}\left\{\|\Ht^T\bm{a}\|_2^2\geq \frac{3p\theta^2}{8}+\frac{3p\theta}{4\pi}\right\}\\
    \leq& \prob_{\Vzero}\left\{\frac{\theta^2}{2}\sum_{j=1}^p\bm{1}_{\{\Xii^T\Vzeroj>0\}}+\frac{1}{2}\sum_{j=1}^p\bm{1}_{\{(\Xii^T\Vzeroj)(\Xkk^T\Vzeroj)<0\}}\geq \frac{3p\theta^2}{8}+\frac{3p\theta}{4\pi}\right\}\\
    &\text{ (by Eq.~\eqref{eq.temp_122202})}\\
    \leq & \prob_{\Vzero}\left\{\left\{\sum_{j=1}^p\bm{1}_{\{\Xii^T\Vzeroj>0\}}>\frac{3p}{4}\right\}\cup\left\{\sum_{j=1}^p\bm{1}_{\{(\Xii^T\Vzeroj)(\Xkk^T\Vzeroj)<0\}}\geq \frac{3p\theta}{2\pi}\right\}\right\}\\
    \leq &\prob_{\Vzero}\left\{\sum_{j=1}^p\bm{1}_{\{\Xii^T\Vzeroj>0\}}>\frac{3p}{4}\right\} +\prob_{\Vzero}\left\{\sum_{j=1}^p\bm{1}_{\{(\Xii^T\Vzeroj)(\Xkk^T\Vzeroj)<0\}}\geq \frac{3p\theta}{2\pi}\right\}\\
    &\text{ (by the union bound)}\\
    \leq &2\exp \left(-\frac{p}{24}\right)+2\exp\left(-\frac{p\theta}{12}\right) \text{ (by Eq.~\eqref{eq.temp_122203} and Eq.~\eqref{eq.temp_122204})}.
\end{align*}
The result of Lemma~\ref{le.temp_122402} thus follows.
\end{proof}

\subsection{Proof of Lemma~\ref{le.temp_122401}}\label{subsec.pigeonhole}
We first prove a useful lemma.
Recall the definition of $C_d$ in Eq.~\eqref{eq.def_Cd}.
\begin{lemma}\label{le.theta_range}
We have
\begin{align*}
    \frac{2\sqrt{2}(d-1)}{nC_d}\in\left[\frac{2}{e}\frac{d-1}{n\sqrt{d}},\ \frac{\pi(d-1)}{n}\right].
\end{align*}
\end{lemma}
\begin{proof}
By Lemma~\ref{le.bound_B} and Eq.~\eqref{eq.def_Cd}, we have
\begin{align*}
    C_d\in\left[\frac{2\sqrt{2}}{\pi},\ e\sqrt{2d}\right].
\end{align*}
Thus, we have
\begin{align*}
    \frac{2\sqrt{2}(d-1)}{nC_d}\in\left[\frac{2}{e}\frac{d-1}{n\sqrt{d}},\ \frac{\pi(d-1)}{n}\right].
\end{align*}
\end{proof}

Now we are ready to proof Lemma~\ref{le.temp_122401}.
\begin{proof}
Recall the definition of $\theta$ in Lemma~\ref{le.temp_122401}.
Draw $n$ caps on $\sd$ centered at $\XX_1,\ \XX_2,\cdots,\XX_n$ with the colatitude angle $\varphi$ where
\begin{align}\label{eq.temp_122401}
    \varphi =\frac{\theta}{2}= \frac{\pi}{2}\left(\frac{2\sqrt{2}(d-1)}{nC_d}\right)^{\frac{1}{d-1}}\text{ (by Eq.~\eqref{eq.def_Cd})}.
\end{align}
By Lemma~\ref{le.theta_range} and $n\geq \pi(d-1)$, we have $\varphi\in[0,\pi/2]$. Thus, by Lemma~\ref{le.sin}, we have
\begin{align}\label{eq.temp_122301}
    \sin\varphi \geq \frac{2\varphi}{\pi} = \left(\frac{2\sqrt{2}(d-1)}{nC_d}\right)^{\frac{1}{d-1}}.
\end{align}
By Lemma~\ref{le.original_cap}, the area of each cap is
\begin{align*}
    A=\frac{1}{2}\musd(\sd)I_{\sin^2\varphi}\left(\frac{d-1}{2},\frac{1}{2}\right).
\end{align*}
Applying Lemma~\ref{le.estimate_Ix} and  Eq.~\eqref{eq.temp_122301}, we thus have
\begin{align*}
    A\geq \frac{1}{2}\musd(\sd)\frac{C_d}{\sqrt{2}(d-1)}(\sin^2\varphi)^{\frac{d-1}{2}}=\frac{1}{n}\musd(\sd).
\end{align*}
In other words, we have
\begin{align*}
    \frac{\musd(\sd)}{A}\leq n.
\end{align*}
By the pigeonhole principle, we know there exist at least two different caps that overlap, i.e., the angle between them is at most $2\varphi$. The result of this lemma thus follows by Eq.~\eqref{eq.temp_122401}.
\end{proof}

\section{Proof of Proposition~\ref{th.fixed_Vzero}}\label{app.proof_fixed_Vzero}
We follow the sketch of proof in Section~\ref{sec.proof_combine}.
Recall the definition of the pseudo ground-truth function $\fv$ in Definition~\ref{def.fv}, and the corresponding $\DVs\in\mathds{R}^{dp}$ that
\begin{align}\label{eq.temp_0929}
    \DVs[j]=\int_{ \sd}\bm{1}_{\{\bm{z}^T\Vzeroj>0\}}\bm{z}\frac{g(\bm{z})}{p}d\mu(\bm{z}), \text{  for all $j\in\{1,2,\cdots,p\}$}.
\end{align}

We first show that the pseudo ground-truth can be written in a linear form.
\begin{lemma}\label{le.dvs}
$\hx\DVs=\fv(\xsmall)$ for all $\xsmall\in\sd$.
\end{lemma}
\begin{proof}
For all $\xsmall\in\sd$, we have
\begin{align*}
    \hx \DVs&=\sum_{j=1}^p \hx[j]\DVs[j]\\
    &=\sum_{j=1}^p \bm{1}_{\{\xsmall^T\Vzeroj>0\}}\cdot\xsmall^T \int_{ \sd}\bm{1}_{\{\bm{z}^T\Vzeroj>0\}}\bm{z}\frac{g(\bm{z})}{p}d\mu(\bm{z})\text{ (by Eq.~\eqref{eq.def_hx} and Eq.~\eqref{eq.temp_0929})}\\
    &=\int_{ \sd}\sum_{j=1}^p \bm{1}_{\{\xsmall^T\Vzeroj>0\}}\cdot\xsmall^T \bm{1}_{\{\bm{z}^T\Vzeroj>0\}}\bm{z}\frac{g(\bm{z})}{p}d\mu(\bm{z})\\
    &=\int_{\sd}\xsmall^T\bm{z}\frac{|\Czx|}{p}g(\bm{z}) d\mu(\bm{z})\text{ (by Eq.~\eqref{eq.card_c})}\\
    &=\fv(\xsmall)\text{ (by Definition~\ref{def.fv})}.
\end{align*}
\end{proof}

Let $\PH\defeq \Ht^T(\Ht\Ht^T)^{-1}\Ht$. Since $\PH^2=\PH$ and $\PH=\PH^T$, we know that $\PH$ is an orthogonal projection to the row-space of $\Ht$. Next, we give an expression for the test error. Note that even though Proposition~\ref{prop.noise_large_p} assumes no noise, below we state a more general version below with noise (which will be useful later).
\begin{lemma}\label{le.f_minus_f}
If the ground-truth is $\f(\xsmall)=\hx\DVs$ for all $\xsmall$, then we have
\begin{align*}
    \fl(\xsmall)-\f(\xsmall)=\hx(\PH-\identity)\DVs+\hx\Ht^T(\Ht\Ht^T)^{-1}\esmall,\text{ for all }\xsmall.
\end{align*}
\end{lemma}
\begin{proof}
Because $\f(\xsmall)=\hx\DVs$, we have $\ysmall=\Ht\DVs+\esmall$. Thus, we have
\begin{align*}
    \DVl&=\Ht^T(\Ht\Ht^T)^{-1}\ysmall\text{ (by Eq.~\eqref{eq.DVl})}\\
    &=\Ht^T(\Ht\Ht^T)^{-1}(\Ht\DVs+\esmall).
\end{align*}
Further, we have
\begin{align*}
    \DVl-\DVs=&\left(\Ht^T(\Ht\Ht^T)^{-1}\Ht-\identity\right)\DVs+\Ht^T(\Ht\Ht^T)^{-1}\esmall\\
    =&(\PH-\identity)\DVs+\Ht^T(\Ht\Ht^T)^{-1}\esmall.
\end{align*}
Finally, using Eq.~\eqref{def.fl}, we have
\begin{align*}
    \fl(\xsmall)-\f(\xsmall)=\hx\DVl-\hx\DVs=\hx(\PH-\identity)\DVs+\hx\Ht^T(\Ht\Ht^T)^{-1}\esmall.
\end{align*}
\end{proof}

When there is no noise, Lemma~\ref{le.f_minus_f} reduces to $\fl(\xsmall)-\f(\xsmall)=\hx(\PH-\identity)\DVs$. As we described in Section~\ref{sec.proof_combine}, $(\PH-\identity)\DVs$ has the interpretation of the distance from $\DVs$ to the row-space of $\Ht$. We then have the following.
\begin{lemma}\label{le.hxDV}
For all $\bm{a}\in \mathds{R}^n$, we have
\begin{align*}
    \left|\hx(\PH-\identity)\DVs\right|\leq \sqrt{p}\|\DVs-\Ht\bm{a}\|_2.
\end{align*}
\end{lemma}
\begin{proof}
Recall that $\PH=\Ht^T(\Ht\Ht^T)^{-1}\Ht$. Thus, we have
\begin{align}\label{eq.temp_110801}
    \PH\Ht^T=\Ht^T(\Ht\Ht^T)^{-1}\Ht\Ht^T=\Ht^T.
\end{align}
We then have
\begin{align*}
    \|(\PH-\identity)\DVs\|_2&=\|\PH\DVs-\DVs\|_2\\
    &=\|\PH(\Ht^T\bm{a}+\DVs-\Ht^T\bm{a})-\DVs\|_2\\
    &=\|\PH\Ht^T\bm{a}+\PH(\DVs-\Ht^T\bm{a})-\DVs\|_2\\
    &=\|\Ht^T\bm{a}+\PH(\DVs-\Ht^T\bm{a})-\DVs\|_2\text{ (by Eq.~\eqref{eq.temp_110801})}\\
    &=\|(\PH-\identity)(\DVs-\Ht^T\bm{a})\|_2\\
    &\leq \|\DVs-\Ht^T\bm{a}\|_2\text{ (because $\PH$ is an orthogonal projection)}.
\end{align*}
Therefore, we have
\begin{align*}
    \left|\hx(\PH-\identity)\DVs\right|=&\left\|\hx(\PH-\identity)\DVs\right\|_2\\
    \leq &\|\hx\|_2\cdot\|(\PH-\identity)\DVs\|_2\text{ (by Lemma~\ref{le.matrix_norm})}\\
    \leq& \sqrt{p}\|\DVs-\Ht\bm{a}\|_2\text{ (by Lemma~\ref{le.h_p})}.
\end{align*}
\end{proof}

Now we are ready to prove Proposition~\ref{th.fixed_Vzero}.
\begin{proof}
Because there is no noise, we have $\esmall=\bm{0}$. Thus, by Lemma~\ref{le.f_minus_f}, we have
\begin{align}
    \fl(\xsmall)-f(\xsmall)=\hx(\PH-\identity)\DVs.
\end{align}
We then have, for all $\bm{a}\in\mathds{R}^n$,
\begin{align}
    &\prob_{\XX}\left\{\left|\fl(\xsmall)-f(\xsmall)\right|\geq n^{-\frac{1}{2}\left(1-\frac{1}{q}\right)}\right\}\nonumber\\
    =&\prob_{\XX}\left\{\left|\hx(\PH-\identity)\DVs\right|\geq n^{-\frac{1}{2}\left(1-\frac{1}{q}\right)}\right\}\nonumber\\
    \leq& \prob_{\XX}\left\{\sqrt{p}\|\Ht^T\bm{a}-\DVs\|_2\geq n^{-\frac{1}{2}\left(1-\frac{1}{q}\right)}\right\}\text{ (by Lemma~\ref{le.hxDV})}.\label{eq.temp_120309}
\end{align}
It only remains to find the vector $\bm{a}$. Define $\mathbf{K}_i\in\mathds{R}^{dp}$ for $i=1,2,\cdots,n$ as
\begin{align*}
    \mathbf{K}_i[j]\defeq \bm{1}_{\{\bm{\Xii}^T\Vzeroj>0\}}\bm{\Xii}\frac{g(\bm{\Xii})}{p},\ j=1,2,\cdots,p.
\end{align*}
By Eq.~\eqref{eq.temp_0929}, for all $j=1,2,\cdots,p$, we have
\begin{align}\label{eq.temp_022800}
    \expectation_{\Xii}[\mathbf{K}_i[j]]=\DVs[j].
\end{align}
Because $\|\Xii\|_2=1$, we have
\begin{align*}
    \|\mathbf{K}_i[j]\|_2\leq \frac{\|g\|_{\infty}}{p}.
\end{align*}
Thus, we have
\begin{align*}
    \|\mathbf{K}_i\|_2=&\sqrt{\sum_{j=1}^p\|\mathbf{K}_i[j]\|_2^2}\leq \frac{\|g\|_\infty}{\sqrt{p}},
\end{align*}
i.e.,
\begin{align}\label{eq.temp_022802}
    \sqrt{p}\|\mathbf{K}_i\|_2\leq \|g\|_\infty.
\end{align}
We now construct the vector $\bm{a}$. Define $\bm{a}\in\mathds{R}^n$ whose $i$-th element is $\bm{a}_i=\frac{g(\Xii)}{np}$, $i=1,2,\cdots,n$. Notice that $\bm{a}$ is well-defined because $\|g\|_\infty<\infty$. Then, for all $j\in\{1,2,\cdots,p\}$, we have
\begin{align*}
    (\Ht^T\bm{a})[j]&=\sum_{i=1}^n\Hti^T[j]\bm{a}_i\\
    &=\sum_{i=1}^n \bm{1}_{\{\bm{\Xii}^T\Vzeroj>0\}}\bm{\Xii}\frac{g(\bm{\Xii})}{np}\\
    &=\frac{1}{n}\sum_{i=1}^n\mathbf{K}_i[j],
\end{align*}
i.e.,
\begin{align}\label{eq.temp_022803}
    \Ht^T\bm{a}=\frac{1}{n}\sum_{i=1}^n \mathbf{K}_i.
\end{align}
Thus, by Eq.~\eqref{eq.temp_022802} and  Lemma~\ref{le.large_deviation} (with $X_i=\sqrt{p}\mathbf{K}_i$, $U=\|g\|_\infty$, and $k=n$), we have
\begin{align*}
    \prob_{\XX}\left\{\sqrt{p}\left\|\left(\frac{1}{n}\sum_{i=1}^n\mathbf{K}_i\right)-\expectation_{\XX}\mathbf{K}_1\right\|_2\geq n^{-\frac{1}{2}\left(1-\frac{1}{q}\right)}\right\}\leq 2e^2\exp\left(-\frac{\sqrt[q]{n}}{8\|g\|_\infty^2}\right).
\end{align*}
Further, by  Eq.~\eqref{eq.temp_022803} and Eq.~\eqref{eq.temp_022800}, we have
\begin{align}\label{eq.temp_022804}
    \prob_{\XX}\left\{\sqrt{p}\|\Ht^T\bm{a}-\DVs\|_2\geq n^{-\frac{1}{2}\left(1-\frac{1}{q}\right)}\right\}\leq 2e^2\exp\left(-\frac{\sqrt[q]{n}}{8\|g\|_\infty^2}\right).
\end{align}
Plugging Eq.~\eqref{eq.temp_022804} into Eq.~\eqref{eq.temp_120309}, we thus have
\begin{align*}
    \prob_{\XX}\left\{\left|\fl(\xsmall)-f(\xsmall)\right|\geq n^{-\frac{1}{2}\left(1-\frac{1}{q}\right)}\right\}\leq2e^2\exp\left(-\frac{\sqrt[q]{n}}{8\|g\|_\infty^2}\right).
\end{align*}

\end{proof}

\section{Proof of Theorem~\ref{th.combine}}\label{app.proof_large_p}
We first prove a useful lemma.
\begin{lemma}\label{le.change_seq}
If $\|g\|_1<\infty$, then for any $\xsmall$, we must have
\begin{align*}
    \int_{\sd}\int_{\sd}\xsmall^T\bm{z} \bm{1}_{\{\bm{z}^T\bm{v}>0,\ \xsmall^T\bm{v}>0\}}g(\bm{z}) d\xdensity(\bm{z})d\nvdensity(\bm{v})=\int_{\sd}\xsmall^T\bm{z}\frac{\pi-\arccos (\xsmall^T\bm{z})}{2\pi}g(\bm{z}) d\xdensity(\bm{z}).
\end{align*}
\end{lemma}
\begin{proof}
This follows from Fubini's Theorem and by a change of order of the integral. Specifically, because $\|g\|_1<\infty$, we have
\begin{align*}
    \int_{\sd}|g(\bm{z})|d\xdensity(\bm{z})<\infty.
\end{align*}
Thus, we have
\begin{align*}
    \int_{\sd\times\sd}|g(\bm{z})|d\xdensity(\bm{z})\nvdensity(\bm{v})<\infty.
\end{align*}
Because $\left|\xsmall^T\bm{z} \bm{1}_{\{\bm{z}^T\bm{v}>0,\ \xsmall^T\bm{v}>0\}}\right|\leq 1$ when $\xsmall\in\sd$ and $\bm{z}\in\sd$, we have
\begin{align*}
    \int_{\sd\times \sd}\left|\xsmall^T\bm{z} \bm{1}_{\{\bm{z}^T\bm{v}>0,\ \xsmall^T\bm{v}>0\}}g(\bm{z})\right| d\xdensity(\bm{z})\nvdensity(\bm{v})\leq \int_{\sd\times\sd}|g(\bm{z})|d\xdensity(\bm{z})\nvdensity(\bm{v})<\infty.
\end{align*}
Thus, by Fubini's theorem, we can exchange the sequence of integral, i.e., we have
\begin{align*}
    &\int_{\sd}\int_{\sd}\xsmall^T\bm{z} \bm{1}_{\{\bm{z}^T\bm{v}>0,\ \xsmall^T\bm{v}>0\}}g(\bm{z}) d\xdensity(\bm{z})d\nvdensity(\bm{v})\\
    =&\int_{\sd}\int_{\sd}\xsmall^T\bm{z} \bm{1}_{\{\bm{z}^T\bm{v}>0,\ \xsmall^T\bm{v}>0\}}g(\bm{z}) d\nvdensity(\bm{v})d\xdensity(\bm{z})\\
    =&\int_{\sd}\left(\int_{\sd} \bm{1}_{\{\bm{z}^T\bm{v}>0,\ \xsmall^T\bm{v}>0\}} d\nvdensity(\bm{v})\right)\xsmall^T\bm{z}g(\bm{z})d\xdensity(\bm{z})\\
    =&\int_{\sd}\xsmall^T\bm{z}\frac{\pi-\arccos (\xsmall^T\bm{z})}{2\pi}g(\bm{z}) d\xdensity(\bm{z})\text{ (by Lemma~\ref{le.spherePortion})}.
\end{align*}
\end{proof}

The following proposition characterizes generalization performance when $\esmall=\bm{0}$, i.e., the bias term in Eq.~\eqref{eq.bias_and_variance_decomposition}.
\begin{proposition}\label{th.large_p}
Assume no noise ($\esmall=\bm{0}$), a ground truth $\f=\fg\in\learnableSet$ where $\|g\|_\infty<\infty$, $n\geq 2$, $m\in\left[1,\ \frac{\ln n}{\ln \frac{\pi}{2}}\right]$, $d\leq n^4$, and $ p\geq 6 J_m(n,d)\ln\left(4n^{1+\frac{1}{m}}\right)$.
Then, for any $q\in[1,\ \infty)$ and for almost every $\xsmall\in \sd$, we must have
\begin{align*}
    &\prob_{\Vzero,\XX}\big\{|\fl(\xsmall)-\f(\xsmall)|\geq n^{-\frac{1}{2}\left(1-\frac{1}{q}\right)}\\
    &+\left(1+\sqrt{J_m(n,d)n}\right)p^{-\frac{1}{2}\left(1-\frac{1}{q}\right)}\big\}\\
    &\leq 2e^2\bigg(\exp\left(-\frac{\sqrt[q]{n}}{8\|g\|_\infty^2}\right)+\exp\left(-\frac{\sqrt[q]{p}}{8\|g\|_1^2}\right)\\
    &+\exp\left(-\frac{\sqrt[q]{p}}{8n\|g\|_1^2}\right)\bigg)+\frac{2}{\sqrt[m]{n}}.
\end{align*}
\end{proposition}

\begin{proof}
We split the whole proof into 5 steps as follows.

\noindent\textbf{Step 1: use pseudo ground-truth as a ``intermediary''}

Recall Definition \ref{def.fv} where we define the pseudo ground-truth $\fv$. We then define the output of the pseudo ground-truth for training input as
\begin{align*}
    \Fv(\XX)\defeq[\fv(\XX_1)\ \fv(\XX_2)\ \cdots\ \fv(\XX_n)]^T.
\end{align*}
The rest of the proof will use the pseudo ground-truth as a ``intermediary'' to connect the ground-truth $\f$ and the model output $\fl$.
Specifically, we have
\begin{align}
    \fl(\xsmall)&=\hx\DVl\nonumber\\
    &=\hx\Ht^T(\Ht\Ht^T)^{-1}\F(\XX)\text{ (by Eq.~\eqref{eq.temp_110903} and $\esmall=\bm{0}$)}\nonumber\\
    &=\hx\Ht^T(\Ht\Ht^T)^{-1}\Fv(\XX)-\hx\Ht^T(\Ht\Ht^T)^{-1}\left(\Fv(\XX)-\F(\XX)\right).\label{eq.temp_110904}
\end{align}
Thus, we have
\begin{align}
    &|\fl(\xsmall)-\f(\xsmall)|\nonumber\\
    =&\left|\fl(\xsmall)-\fv(\xsmall)+\fv(\xsmall)-\f(\xsmall)\right|\nonumber\\
    =&\left|\hx\Ht^T(\Ht\Ht^T)^{-1}\Fv(\XX)-\fv(\xsmall)-\hx\Ht^T(\Ht\Ht^T)^{-1}\left(\Fv(\XX)-\F(\XX)\right)\right.\nonumber\\
    &\left.+\fv(\xsmall)-\f(\xsmall)\right|\text{ (by Eq.~\eqref{eq.temp_110904})}\nonumber\\
    \leq &\underbrace{\left|\hx\Ht^T(\Ht\Ht^T)^{-1}\Fv(\XX)-\fv(\xsmall)\right|}_{\text{term }A}+\underbrace{\left|\hx\Ht^T(\Ht\Ht^T)^{-1}\left(\Fv(\XX)-\F(\XX)\right)\right|}_{\text{term }B}\nonumber\\
    &+\underbrace{\left|\fv(\xsmall)-\f(\xsmall)\right|}_{\text{term }C}\label{eq.temp_100704}.
\end{align}
In Eq.~\eqref{eq.temp_100704}, we can see that the term $A$ corresponds to the test error of the pseudo ground-truth, the term $B$ corresponds to the impact of the difference between the pseudo ground-truth and the real ground-truth in the training data, and the term $C$ corresponds to the impact of the difference between pseudo ground-truth and real ground-truth in the test data. Using the terminology of bias-variance decomposition, we refer to term $A$ as the ``pseudo bias'' and term $B$ as the ``pseudo variance''.

\noindent\textbf{Step 2: estimate term $A$}

We have
\begin{align}
    \prob_{\XX,\Vzero}\left\{\text{term }A\geq n^{-\frac{1}{2}\left(1-\frac{1}{q}\right)}\right\}&=\int_{\Vzero\in\mathds{R}^{dp}}\prob_{\XX}\left\{\text{term }A\geq n^{-\frac{1}{2}\left(1-\frac{1}{q}\right)}\ \bigg|\ \Vzero\right\}\ d\vdensity(\Vzero)\nonumber\\
    &\leq \int_{\Vzero\in\mathds{R}^{dp}}2e^2\exp\left(-\frac{\sqrt[q]{n}}{8\|g\|_\infty^2}\right)\ d\vdensity(\Vzero)\text{ (by Proposition~\ref{th.fixed_Vzero})}\nonumber\\
    &=2e^2\exp\left(-\frac{\sqrt[q]{n}}{8\|g\|_\infty^2}\right).\label{eq.temp_111102}
\end{align}

\noindent\textbf{Step 3: estimate term $C$}

For all $j=1,2,\cdots,p$, define
\begin{align*}
    K_j^{\xsmall}\defeq\int_{\sd}\xsmall^T\bm{z} \bm{1}_{\{\bm{z}^T\Vzeroj>0,\ \xsmall^T\Vzeroj>0\}}g(\bm{z}) d\xdensity(\bm{z}).
\end{align*}
We now show that $K_j^{\xsmall}$ is bounded and with mean equal to $\fg$, where $\fg=\int_{\sd}\xsmall^T\bm{z}\frac{\pi-\arccos (\xsmall^T\bm{z})}{2\pi}g(\bm{z}) d\xdensity(\bm{z})$ defined by Definition~\ref{def.learnableSet}. Specifically, we have
\begin{align}\label{eq.temp_120402}
    \expectation_{\Vzero}K_j^{\xsmall}=&\expectation_{\bm{v}\sim \nvdensity}\left(\int_{\sd}\xsmall^T\bm{z} \bm{1}_{\{\bm{z}^T\bm{v}>0,\ \xsmall^T\bm{v}>0\}}g(\bm{z}) d\xdensity(\bm{z})\right)\nonumber\\
    =&\int_{\sd}\int_{\sd}\xsmall^T\bm{z} \bm{1}_{\{\bm{z}^T\bm{v}>0,\ \xsmall^T\bm{v}>0\}}g(\bm{z}) d\xdensity(\bm{z})d\nvdensity(\bm{v})\nonumber\\
    =&\int_{\sd}\xsmall^T\bm{z}\frac{\pi-\arccos (\xsmall^T\bm{z})}{2\pi}g(\bm{z}) d\xdensity(\bm{z})\text{ (by Lemma~\ref{le.change_seq})}\nonumber\\
    =&\fg(\xsmall)\text{ (by Definition~\ref{def.learnableSet})}.
\end{align}
From Definition~\ref{def.fv}, we have
\begin{align}\label{eq.temp_100702}
    \fv(\xsmall)=&\int_{\sd}\xsmall^T\bm{z}\frac{|\Czx|}{p}g(\bm{z}) d\xdensity(\bm{z})\text{ (by Definition~\ref{def.fv})}\nonumber\\
    =&\frac{1}{p}\sum_{j=1}^p\int_{\sd}\xsmall^T\bm{z} \bm{1}_{\{\bm{z}^T\Vzeroj>0,\ \xsmall^T\Vzeroj>0\}}g(\bm{z}) d\xdensity(\bm{z})\text{ (by Eq.~\eqref{eq.card_c})}\nonumber\\
    =&\frac{1}{p}\sum_{j=1}^p K_j^{\xsmall}.
\end{align}
Because $\Vzeroj$'s are \emph{i.i.d.}, $K_j^{\xsmall}$'s are also \emph{i.i.d.}.
Thus, we have
\begin{align}
    \expectation_{\Vzero}\fv(\xsmall)&=\fg(\xsmall).\label{eq.temp_100703}
\end{align}
Further, for any $j\in\{1,2,\cdots,p\}$, we have
\begin{align}
    \|K_j^{\xsmall}\|_2=&|K_j^{\xsmall}|\text{ (because $K_j^{\xsmall}$ is a scalar)}\nonumber\\
    =&\left|\int_{\sd}\xsmall^T\bm{z} \bm{1}_{\{\bm{z}^T\Vzeroj>0,\ \xsmall^T\Vzeroj>0\}}g(\bm{z}) d\xdensity(\bm{z})\right|\nonumber\\
    \leq &\int_{\sd}\left|\xsmall^T\bm{z} \bm{1}_{\{\bm{z}^T\Vzeroj>0,\ \xsmall^T\Vzeroj>0\}}g(\bm{z})\right| d\xdensity(\bm{z})\nonumber\\
    \leq &\int_{\sd}\left|\xsmall^T\bm{z} \bm{1}_{\{\bm{z}^T\Vzeroj>0,\ \xsmall^T\Vzeroj>0\}}\right|\cdot\left|g(\bm{z})\right| d\xdensity(\bm{z})\nonumber\\
    \leq &\int_{\sd}\left|g(\bm{z})\right| d\xdensity(\bm{z})\nonumber\\
    =&\|g\|_1\label{eq.temp_120411}.
\end{align}
Thus, by Lemma~\ref{le.large_deviation}, we have
\begin{align*}
    \prob_{\Vzero}\left\{\left\|\left(\frac{1}{p}\sum_{j=1}^p K_j^{\xsmall}\right)-\expectation_{\Vzero}K_1\right\|_2\geq p^{-\frac{1}{2}\left(1-\frac{1}{q}\right)}\right\}\leq 2e^2\exp\left(-\frac{\sqrt[q]{p}}{8\|g\|_1^2}\right).
\end{align*}
Further, by Eq.~\eqref{eq.temp_100702} and Eq.~\eqref{eq.temp_120402}, we have
\begin{align*}
    \prob_{\Vzero}\left\{\left|\fv(\xsmall)-\fg(\xsmall)\right|\geq p^{-\frac{1}{2}\left(1-\frac{1}{q}\right)}\right\}\leq 2e^2\exp\left(-\frac{\sqrt[q]{p}}{8\|g\|_1^2}\right).
\end{align*}
Because $\f\stackrel{\text{a.e.}}{=}\fg$, we have
\begin{align*}
    \prob_{\Vzero}\left\{\left|\fv(\xsmall)-\f(\xsmall)\right|\geq p^{-\frac{1}{2}\left(1-\frac{1}{q}\right)}\right\}\leq 2e^2\exp\left(-\frac{\sqrt[q]{p}}{8\|g\|_1^2}\right).
\end{align*}
Because $\fv$ does not change with $\XX$, we thus have
\begin{align}\label{eq.temp_120403}
    \prob_{\Vzero,\XX}\left\{\text{term }C\geq p^{-\frac{1}{2}\left(1-\frac{1}{q}\right)}\right\}\leq 2e^2\exp\left(-\frac{\sqrt[q]{p}}{8\|g\|_1^2}\right).
\end{align}

\noindent\textbf{Step 4: estimate term $B$}

Our idea is to treat $\Fv(\XX)-\F(\XX)$ as a special form of noise, and then apply Proposition~\ref{prop.noise_large_p}. We first bound the magnitude of this special noise.
For $j=1,2,\cdots,p$, we define
\begin{align*}
    \mathbf{K}_j\defeq [K_j^{\XX_1}\ K_j^{\XX_2}\ \cdots\ K_j^{\XX_n}]^T.
\end{align*}
Then, we have
\begin{align*}
    \|\mathbf{K}_j\|_2= \sqrt{\sum_{i=1}^n \|K_j^{\XX_i}\|_2^2}\leq \sqrt{n}\|g\|_1\text{ (by Eq.~\eqref{eq.temp_120411})}.
\end{align*}
Similar to how we get Eq.~\eqref{eq.temp_120403} in Step 3, we have
\begin{align}\label{eq.temp_120404}
    \prob_{\Vzero,\XX}\left\{\left\|\Fv(\XX)-\F(\XX)\right\|_2\geq p^{-\frac{1}{2}\left(1-\frac{1}{q}\right)}\right\}\leq 2e^2\exp\left(-\frac{\sqrt[q]{p}}{8n\|g\|_1^2}\right).
\end{align}
Thus, we have
\begin{align}
    &\prob_{\Vzero,\XX}\left\{\text{term }B\geq \sqrt{J_m(n,d)n}p^{-\frac{1}{2}\left(1-\frac{1}{q}\right)}\right\}\nonumber\\
    =&\prob_{\Vzero,\XX}\left\{\text{term }B\geq \sqrt{J_m(n,d)n}p^{-\frac{1}{2}\left(1-\frac{1}{q}\right)},\left\|\Fv(\XX)-\F(\XX)\right\|_2\geq p^{-\frac{1}{2}\left(1-\frac{1}{q}\right)}\right\}\nonumber\\
    &+\prob_{\Vzero,\XX}\left\{\text{term }B\geq \sqrt{J_m(n,d)n}p^{-\frac{1}{2}\left(1-\frac{1}{q}\right)},\left\|\Fv(\XX)-\F(\XX)\right\|_2< p^{-\frac{1}{2}\left(1-\frac{1}{q}\right)}\right\}\nonumber\\
    \leq& \prob_{\Vzero,\XX}\left\{\left\|\Fv(\XX)-\F(\XX)\right\|_2\geq p^{-\frac{1}{2}\left(1-\frac{1}{q}\right)}\right\}\nonumber\\
    &+\prob_{\Vzero,\XX}\left\{\text{term }B\geq \sqrt{J_m(n,d)n}\left\|\Fv(\XX)-\F(\XX)\right\|_2\right\}\nonumber\\
    \leq & 2e^2\exp\left(-\frac{\sqrt[q]{p}}{8n\|g\|_1^2}\right)+\frac{2}{\sqrt[m]{n}}
    \text{ (by  Eq.~\eqref{eq.temp_120404} and Proposition~\ref{prop.noise_large_p})}.\label{eq.temp_111104}
\end{align}

\noindent\textbf{Step 5: estimate $|\fl(\xsmall)-\f(\xsmall)|$}

We have
\begin{align*}
    &\prob_{\Vzero,\XX}\left\{|\fl(\xsmall)-\f(\xsmall)|\geq n^{-\frac{1}{2}\left(1-\frac{1}{q}\right)}+\frac{1+\sqrt{J_m(n,d)n}}{\sqrt[4]{p}}\right\}\\
    \leq &\prob_{\Vzero,\XX}\left\{\text{term }A+\text{term }B+\text{term }C\geq n^{-\frac{1}{2}\left(1-\frac{1}{q}\right)}+\frac{1+\sqrt{J_m(n,d)n}}{\sqrt[4]{p}}\right\}\text{ (by Eq.~\eqref{eq.temp_100704})}\\
    \leq &\prob_{\XX,\Vzero}\left\{\left\{\text{term }A\geq n^{-\frac{1}{2}\left(1-\frac{1}{q}\right)}\right\}\cup\left\{\ \text{term }B\geq \sqrt{J_m(n,d)n}p^{-\frac{1}{2}\left(1-\frac{1}{q}\right)}\right\}\right.\\
    &\left.\cup\left\{\ \text{term }C\geq p^{-\frac{1}{2}\left(1-\frac{1}{q}\right)}\right\}\right\}\\
    \leq &\prob_{\XX,\Vzero}\left\{\text{term }A\geq n^{-\frac{1}{2}\left(1-\frac{1}{q}\right)}\right\}+\prob_{\Vzero,\XX}\left\{\text{term }B\geq \sqrt{J_m(n,d)n}p^{-\frac{1}{2}\left(1-\frac{1}{q}\right)}\right\}\\
    &+\prob_{\Vzero,\XX}\left\{\text{term }C\geq p^{-\frac{1}{2}\left(1-\frac{1}{q}\right)}\right\}\text{(by the union bound)}\\
    \leq &2e^2\left(\exp\left(-\frac{\sqrt[q]{n}}{8\|g\|_\infty^2}\right)+\exp\left(-\frac{\sqrt[q]{p}}{8\|g\|_1^2}\right)+\exp\left(-\frac{\sqrt[q]{p}}{8n\|g\|_1^2}\right)\right)+\frac{2}{\sqrt[m]{n}}
    \\
    &\text{ (by Eqs.~\eqref{eq.temp_111102}\eqref{eq.temp_120403}\eqref{eq.temp_111104})}.
\end{align*}
The last step exactly gives the conclusion of this proposition.

\end{proof}
Theorem~\ref{th.combine} thus follows by Proposition~\ref{prop.noise_large_p}, Proposition~\ref{th.large_p}, Eq.~\eqref{eq.bias_and_variance_decomposition}, and the union bound.
\section{Proof of Proposition~\ref{th.equal_set} (lower bound for ground-truth functions outside \texorpdfstring{$\overline{\learnableSet}$}{the learnable set})}\label{ap.equal_set}
We first show what $\flinf$ looks like.
Define $\mathbf{H}^\infty\in\mathds{R}^{n\times n}$ where its $(i,j)$-th element is
\begin{align*}
    \mathbf{H}^\infty_{i,j}=\Xii^T\Xjj\frac{\pi-\arccos(\Xii^T\Xjj)}{2\pi}.
\end{align*}
Notice that
\begin{align*}
    \left(\frac{\Ht\Ht^T}{p}\right)_{i,j}=\frac{1}{p}\sum_{k=1}^p \Xii^T\Xjj \bm{1}_{\{\Xii^T \Vzerok >0, \Xjj^T \Vzerok>0\}}=\Xii^T\Xjj \frac{|\mathcal{C}^{\Vzero}_{\Xii,\Xjj}|}{p}.
\end{align*}
By Lemma~\ref{le.c_divide_p_convergence}, we have that $\left(\frac{\Ht\Ht^T}{p}\right)_{i,j}$ converges in probability to $(\Ht^\infty)_{i,j}$ as $p\to\infty$ uniformly in $i,j$. In other words,
\begin{align}\label{eq.H_inf_convergence}
    \max_{i,j}\left|\left(\frac{\Ht\Ht^T}{p}\right)_{i,j}-(\Ht^\infty)_{i,j}\right|\stackrel{\text{P}}{\rightarrow}0,\text{ as } p\to\infty.
\end{align}

Let $\{\bm{e}_i\ |\ 1\leq i\leq n\}$ denote the standard basis in $\mathds{R}^n$. For $i=1,2,\cdots,n$, define
\begin{align}\label{eq.def_gbi}
    &\gbi\defeq np\bm{e}_i^T(\Ht\Ht^T)^{-1}\ysmall,
\end{align}
which is a number. Further, define
\begin{align*}
    [\gb_{1,p}\ \gb_{2,p}\ \cdots\ \gb_{n,p}]^T=np(\Ht\Ht^T)^{-1}\ysmall.
\end{align*}
Further, define the number
\begin{align*}
    \gbinf\defeq n\bm{e}_i^T(\Ht^\infty)^{-1}\ysmall,
\end{align*}
and 
\begin{align*}
    [\gb_{1,\infty}\ \gb_{2,\infty}\ \cdots\ \gb_{n,\infty}]^T=n(\Ht^\infty)^{-1}\ysmall.
\end{align*}
Notice that $(\Ht^\infty)^{-1}$ exists because of Eq.~\eqref{eq.H_inf_convergence} and Lemma~\ref{le.full_rank}. 

By Eq.~\eqref{eq.H_inf_convergence}, we have
\begin{align}\label{eq.gbi_convergence}
    \max_{i\in\{1,2,\cdots,n\}}|\gbi-\gbinf|\stackrel{\text{P}}{\rightarrow}0,\text{ as } p\to\infty.
\end{align}

For any given $\XX$, we define $\flinf(\cdot):\ \sd\mapsto \mathds{R}$ as
\begin{align}\label{eq.def_flinf}
    \flinf(\xsmall)\defeq\frac{1}{n}\sum_{i=1}^n\xsmall^T\Xii\frac{\pi-\arccos(\xsmall^T\Xii)}{2\pi}\gbinf.
\end{align}
By the definition of the Dirac delta function $\delta_a(\cdot)$ with peak position at $a$, we can write $\flinf(\xsmall)$ as an integral
\begin{align*}
    \flinf(\xsmall)= \int_{\sd}\xsmall^T\bm{z}\frac{\pi-\arccos(\xsmall^T\bm{z})}{2\pi}\frac{1}{n}\sum_{i=1}^n\gbinf\delta_{\Xii}(\bm{z}) d\xdensity(\bm{z}).
\end{align*}
Notice that $\gbinf$ only depends on the training data and does not change with $p$ (and thus is finite). Therefore, we have $\flinf\in\learnableSet$. It remains to show why $\fl$ converges to $\flinf$ in probability. The following lemma shows what $\fl$ looks like.

\begin{lemma}\label{le.fl_delta}
$\fl(\xsmall)=\frac{1}{n}\sum_{i=1}^n\xsmall^T\Xii\frac{|\CXix|}{p}\gbi=\int_{\sd}\xsmall^T\bm{z}\frac{|\Czx|}{p}\frac{1}{n}\sum_{i=1}^n \gbi\delta_{\Xii}(\bm{z})d\xdensity(\bm{z})$.
\end{lemma}
\begin{proof}
For any $\xsmall\in\sd$, we have
\begin{align*}
    \fl(\xsmall)&=\hx \DVl\\
    &=\hx \Ht^T(\Ht\Ht^T)^{-1}\ysmall\text{ (by Eq.~\eqref{eq.DVl})}\\
    &=\hx\sum_{i=1}^n\Hti^T\bm{e}_i^T(\Ht\Ht^T)^{-1}\ysmall\\
    &=\frac{1}{np}\sum_{i=1}^n\hx\Hti^T\gbi\text{ (by Eq.~\eqref{eq.def_gbi})}\\
    &=\frac{1}{np}\sum_{i=1}^n\sum_{j=1}^p\xsmall^T\Xii\bm{1}_{\{\Xii^T\Vzeroj>0,\ \xsmall^T\Vzeroj>0\}}\gbi.
\end{align*}
By Eq.~\eqref{eq.card_c}, we thus have
\begin{align}\label{eq.temp_012401}
    \fl(\xsmall)=\frac{1}{n}\sum_{i=1}^n\xsmall^T\Xii\frac{|\CXix|}{p}\gbi.
\end{align}
By the definition of the Dirac delta function, we have
\begin{align*}
    \flmap(\xsmall)&=\frac{1}{n}\sum_{i=1}^n\xsmall^T\Xii\frac{|\CXix|}{p}\gbi=\int_{\sd}\xsmall^T\bm{z}\frac{|\Czx|}{p}\frac{1}{n}\sum_{i=1}^n \gbi\delta_{\Xii}(\bm{z})d\xdensity(\bm{z}).
\end{align*}
\end{proof}
Now we are ready to prove the statement of Proposition~\ref{th.equal_set}, i.e., uniformly over all $\xsmall\in\sd$,  $\fl(\xsmall)\stackrel{\text{P}}{\rightarrow}\flinf(\xsmall)$ as $p\to\infty$ (notice that we have already shown that $\flinf\in\learnableSet$). To be more specific, we restate that uniform convergence as the following lemma.
\begin{lemma}\label{le.uniform_convergence_sup_x}
For any given $\XX$, $\sup\limits_{\xsmall\in\sd}|\fl(\xsmall)-\flinf(\xsmall)|\stackrel{\text{P}}{\rightarrow}0$ as $p\to\infty$.
\end{lemma}
\begin{proof}
For any $\zeta>0$, define two events:
\begin{align*}
    &\mathcal{J}_1\defeq\left\{\sup_{\xsmall,\bm{z}\in\sd}\left|\frac{|\Czx|}{p}-\frac{\pi-\arccos (\xsmall^T\bm{z})}{2\pi}\right|<\zeta\right\},\\
    &\mathcal{J}_2\defeq\left\{\max_{i\in\{1,2,\cdots,n\}}|\gbi-\gbinf|<\zeta\right\}.
\end{align*}

By Lemma~\ref{le.c_divide_p_convergence}, there exists a threshold $p_0$ such that for any $p>p_0$,
\begin{align*}
    \prob [\mathcal{J}_1]>1-\zeta.
\end{align*}
By Eq.~\eqref{eq.gbi_convergence}, there exists a threshold $p_1$ such that for any $p>p_1$,
\begin{align*}
    \prob [\mathcal{J}_2]>1-\zeta.
\end{align*}
Thus, by the union bound, when $p>\max\{p_0,p_1\}$, we have
\begin{align}\label{eq.temp_020801}
    \prob[\mathcal{J}_1 \cap \mathcal{J}_2]>1-2\zeta.
\end{align}
When $\mathcal{J}_1 \cap \mathcal{J}_2$ happens, we have
\begin{align*}
    &\sup\limits_{\xsmall\in\sd}|\fl(\xsmall)-\flinf(\xsmall)|\\
    =&\sup\limits_{\xsmall\in\sd}\left|\frac{1}{n}\sum_{i=1}^n\xsmall^T\Xii\left(\frac{|\CXix|}{p}\gbi - \frac{\pi-\arccos(\xsmall^T\Xii)}{2\pi}\gbinf\right)\right|\\
    &\text{ (by Lemma~\ref{le.fl_delta} and Eq.~\eqref{eq.def_flinf})}\\
    \leq & \sup\limits_{\xsmall\in\sd, i\in\{1,2,\cdots,n\}}\left|\left(\frac{|\CXix|}{p}\gbi - \frac{\pi-\arccos(\xsmall^T\Xii)}{2\pi}\gbinf\right)\right|\text{ (because $|\bm{x}^T\Xii|\leq 1$)}\\
    =&\sup\limits_{\xsmall\in\sd, i\in\{1,2,\cdots,n\}}\left|\left(\frac{|\CXix|}{p}-\frac{\pi-\arccos(\xsmall^T\Xii)}{2\pi}\right)\gbinf + (\gbi-\gbinf)\frac{|\CXix|}{p}\right|\\
    \leq &\sup\limits_{\xsmall\in\sd, i\in\{1,2,\cdots,n\}}\left|\left(\frac{|\CXix|}{p}-\frac{\pi-\arccos(\xsmall^T\Xii)}{2\pi}\right)\gbinf\right| +\left| (\gbi-\gbinf)\frac{|\CXix|}{p}\right|\\
    \leq & \zeta\cdot\left(\max_i|\gbinf|+1\right)\text{ (because $\mathcal{J}_1 \cap \mathcal{J}_2$ happens, $\frac{|\CXix|}{p}\in[0,1]$, and $\frac{\pi-\arccos(\xsmall^T\Xii)}{2\pi}\in[0, 0.5]$)}.
\end{align*}
Because $\max_i|\gbinf|$ is fixed when $\XX$ is given, $\zeta\cdot\left(\max_i|\gbinf|+1\right)$ can be arbitrarily small as long as $\zeta$ is small enough. The conclusion of this lemma thus follows by Eq.~\eqref{eq.temp_020801}.
\end{proof}

If the ground-truth function $\f\notin \overline{\learnableSet}$ (or equivalently, $D(\f,\learnableSet)>0$), then the MSE of $\flinf$ (with respect to the ground-truth function $\f$) is at least $D(\f,\learnableSet)$ (because $\flinf\in\learnableSet$). 
Therefore, we have proved Proposition~\ref{th.equal_set}. Below we state an even stronger result than part (ii) of Proposition~\ref{th.equal_set}, i.e., it captures not only the MSE of $\flinf$, but also that of $\fl$ for sufficiently large $p$. 

\begin{lemma}
For any given $\XX$ and $\zeta>0$, there exists a threshold $p_0$ such that for all $p>p_0$, $\prob\{\sqrt{\text{MSE}}\geq D(\f,\learnableSet)-\zeta\}>1-\zeta$.
\end{lemma}
\begin{proof}
By Lemma~\ref{le.uniform_convergence_sup_x}, for any $\zeta>0$, there must exist a threshold $p_0$ such that for all $p>p_0$,
\begin{align*}
    \prob\left\{\sup\limits_{\xsmall\in\sd}|\fl(\xsmall)-\flinf(\xsmall)|<\zeta\right\}>1-\zeta.
\end{align*}
When $\sup\limits_{\xsmall\in\sd}|\fl(\xsmall)-\flinf(\xsmall)|<\zeta$, we have
\begin{align*}
    D(\fl, \flinf)=&\sqrt{\int_{\sd}\left(\fl(\xsmall)-\flinf(\xsmall)\right)^2d\xdensity(\xsmall)}\leq \zeta.
\end{align*}
Because $\flinf\in\learnableSet$, we have $D(\flinf, \f)\geq D(\f, \learnableSet)$. Thus, by the triangle inequality, we have $D(\f,\fl)\geq D(\f,\flinf)-D(\fl,\flinf)\geq D(\f, \learnableSet)-\zeta$. Putting these together, we have
\begin{align*}
    \prob\left\{D(\f, \fl)\geq D(\f,\learnableSet)-\zeta\right\}>1-\zeta.
\end{align*}
Notice that $\text{MSE}=(D(\f,\fl))^2$. The result of this lemma thus follows.
\end{proof}

\section{Details for Section~\ref{sec.feature_of_set} (hyper-spherical harmonics decomposition on \texorpdfstring{$\sd$}{the unit hyper-sphere})}\label{app.d_larger_than_three}

\subsection{Convolution on \texorpdfstring{$\sd$}{hyper-sphere}}\label{app.convolution}
First, we introduce the definition of the convolution on $\sd$. In \cite{dokmanic2009convolution}, the convolution on $\sd$ is defined as follows.
\begin{align*}
    f_1 \circledast f_2(\xsmall)\defeq\int_{\mathsf{SO}(d)}f_1(\mathbf{S}\bm{e})f_2(\mathbf{S}^{-1}\xsmall)d\mathbf{S},
\end{align*}
where $\mathbf{S}$ is a $d\times d$ orthogonal matrix that denotes a rotation in $\sd$, chosen from the set $\mathsf{SO}(d)$ of all rotations.
In the following, we will show Eq.~\eqref{eq.convolution}. To that end, we have
\begin{align}\label{eq.temp_112601}
    g\circledast\hf(\xsmall)&=\int_{\mathsf{SO}(d)}g(\mathbf{S}\bm{e})\hf(\mathbf{S}^{-1}\xsmall)d\mathbf{S}.
\end{align}
Now, we replace $\mathbf{S}\bm{e}$ by $\bm{z}$. Thus, we have
\begin{align*}
    \mathbf{S}\bm{e}=\bm{z}\implies \bm{e}=\mathbf{S}^{-1}\bm{z}\implies (\mathbf{S}^{-1}\xsmall)^T\bm{e}=(\mathbf{S}^{-1}\xsmall)^T\mathbf{S}^{-1}\bm{z}\implies (\mathbf{S}^{-1}\xsmall)^T\bm{e}=\xsmall^T(\mathbf{S}^{-1})^T\mathbf{S}^{-1}\bm{z}.
\end{align*}
Because $\mathbf{S}$ is an orthonormal matrix, we have $\mathbf{S}^T=\mathbf{S}^{-1}$. Therefore, we have $(\mathbf{S}^{-1}\xsmall)^T\bm{e}=\xsmall^T\bm{z}$. Thus, by Eq.~\eqref{eq.h_in_convolution}, we have
\begin{align}\label{eq.temp_112602}
    \hf(\mathbf{S}^{-1}\xsmall)=(\mathbf{S}^{-1}\xsmall)^T\bm{e}\frac{\pi-\arccos ((\mathbf{S}^{-1}\xsmall)^T\bm{e})}{2\pi}=\xsmall^T\bm{z}\frac{\pi-\arccos (\xsmall^T\bm{z})}{2\pi}.
\end{align}
By plugging Eq.~\eqref{eq.temp_112602} into Eq.~\eqref{eq.temp_112601}, we have
\begin{align*}
    g\circledast\hf(\xsmall)&=\int_{\sd}g(\bm{z})\xsmall^T\bm{z}\frac{\pi-\arccos (\xsmall^T\bm{z})}{2\pi}d\xdensity(\bm{z}).
\end{align*}
Eq.~\eqref{eq.convolution} thus follows.

The following lemma shows the intrinsic symmetry of such a convolution.
\begin{lemma}
Let $\mathbf{S}\in\mathds{R}^{d\times d}$ denotes any rotation in $\mathds{R}^d$. If $\f(\xsmall)\in\learnableSet$, then $f(\mathbf{S}\xsmall)\in \learnableSet$. 
\end{lemma}
\begin{proof}
Because $\f(\xsmall)\in\learnableSet$, we can find $g$ such that
\begin{align*}
    \f(\xsmall)=\int_{\sd}\xsmall^T\bm{z}\frac{\pi-\arccos(\xsmall^T\bm{z})}{2\pi}g(\bm{z})d\xdensity(\bm{z}).
\end{align*}
Thus, we have
\begin{align*}
    \f(\mathbf{S}\bm{x})=&\int_{\sd}(\mathbf{S}\xsmall)^T\bm{z}\frac{\pi-\arccos((\mathbf{S}\xsmall)^T\bm{z})}{2\pi}g(\bm{z})d\xdensity(\bm{z})\\
    =&\int_{\sd}\xsmall^T(\mathbf{S}^T\bm{z})\frac{\pi-\arccos(\xsmall^T(\mathbf{S}^T\bm{z}))}{2\pi}g(\bm{z})d\xdensity(\bm{z})\\
    =&\int_{\sd}\xsmall^T(\mathbf{S}^T\bm{z})\frac{\pi-\arccos(\xsmall^T(\mathbf{S}^T\bm{z}))}{2\pi}g(\mathbf{S}\mathbf{S}^T\bm{z})d\xdensity(\bm{z})\\
    &\text{ (because $\mathbf{S}$ is a rotation, we have $\mathbf{S}\mathbf{S}^T=\mathbf{I}$)}\\
    =&\int_{\sd}\xsmall^T\bm{z}\frac{\pi-\arccos(\xsmall^T\bm{z})}{2\pi}g(\mathbf{S}\bm{z})d\xdensity(\mathbf{S}\bm{z})\text{ (replace $\mathbf{S}^T\bm{z}$ by $\bm{z}$)}\\
    =&\int_{\sd}\xsmall^T\bm{z}\frac{\pi-\arccos(\xsmall^T\bm{z})}{2\pi}g(\mathbf{S}\bm{z})d\xdensity(\bm{z})\text{ (by Assumption~\ref{as.uniform})}
\end{align*}
The result of this lemma thus follows.
\end{proof}

\subsection{Hyper-spherical harmonics}

We follow the the conventions of hyper-spherical harmonics in \cite{dokmanic2009convolution}.  We express $\xsmall=[\xsmall_1\ \xsmall_2\ \cdots\ \xsmall_d]\in\sd$
in a set of hyper-spherical polar coordinates as follows.
\begin{align*}
    \xsmall_1&=\sin \theta_{d-1}\sin\theta_{d-2}\cdots\sin\theta_2\sin\theta_1,\\
    \xsmall_2&=\sin\theta_{d-1}\sin\theta_{d-2}\cdots\sin\theta_2\cos\theta_1,\\
    \xsmall_3&=\sin\theta_{d-1}\sin\theta_{d-2}\cdots\cos\theta_2,\\
    &\quad \vdots\\
    \xsmall_{d-1}&=\sin\theta_{d-1}\cos\theta_{d-2},\\
    \xsmall_d&=\cos\theta_{d-1}.
\end{align*}
Notice that $\theta_1\in[0,\ 2\pi)$ and $\theta_2,\theta_3,\cdots,\theta_{d-1}\in [0,\pi)$.
Let $\xi=[\theta_1\ \theta_2\ \cdots\ \theta_{d-1}]$.
In such coordinates, hyper-spherical harmonics are given by \cite{dokmanic2009convolution}
\begin{align}\label{eq.harmonics}
    \Xi_{\mathbf{K}}^l(\xi)=A_{\mathbf{K}}^l\times \prod_{i=0}^{d-3}C_{k_i-k_{i+1}}^{\frac{d-i-2}{2}+k_{i+1}}(\cos\theta_{d-i-1})\sin^{k_{i+1}}\theta_{d-i-1}e^{\pm jk_{d-2}\theta_1},
\end{align}
where the normalization factor is
\begin{align*}
    A_{\mathbf{K}}^l=\sqrt{\frac{1}{\Gamma\left(\frac{d}{2}\right)}\prod_{i=0}^{d-3}2^{2k_{i+1}+d-i-4}\times \frac{(k_i-k_{i+1})!(d-i+2k_i-2)\Gamma^2\left(\frac{d-i-2}{2}+k_{i+1}\right)}{\sqrt{\pi}\Gamma(k_i+k_{i+1}+d-i-2)}},
\end{align*}
and $C_d^\lambda(t)$ are the Gegenbauer polynomials of degree $d$. These Gegenbauer polynomials can be defined as the coefficients of $\alpha^n$ in the power-series expansion of the following function,
\begin{align*}
    (1-2t\alpha +\alpha^2)^{-\lambda}=\sum_{i=0}^{\infty}C_i^\lambda(t)\alpha^i.
\end{align*}
Further, the Gegenbauer polynomials can be computed by a three-term recursive relation,
\begin{align}\label{eq.temp_011302}
    (i+2)C_{i+2}^\lambda(t)=2(\lambda+i+1)tC_{i+1}^\lambda(t)-(2\lambda+i)C_i^\lambda(t),
\end{align}
with $C_0^\lambda(t)=1$ and $C_1^\lambda(t)=2\lambda t$.

\subsection{Calculate \texorpdfstring{$\Xi_{\mathbf{K}}^l(\xi)$}{harmonics} where \texorpdfstring{$\mathbf{K}=\mathbf{0}$}{K=0}}
Recall that $\mathbf{K}=(k_1,k_2,\cdots,k_{d-2})$ and $l=k_0$. By plugging $\mathbf{K}=\mathbf{0}$ into Eq.~\eqref{eq.harmonics}, we have
\begin{align}\label{eq.temp_011304}
    \Xi_{\mathbf{0}}^l(\xi)=A_{\mathbf{0}}^l\times C_l^{\frac{d-2}{2}}(\cos\theta_{d-1}).
\end{align}
The following lemma gives an explicit form of Gegenbauer polynomials.
\begin{lemma}\label{le.poly_Geg}
\begin{align}\label{eq.temp_011301}
    C_i^\lambda(t)=\sum_{k=0}^{\lfloor \frac{i}{2}\rfloor}(-1)^k\frac{\Gamma(i-k+\lambda)}{\Gamma(\lambda)k!(i-2k)!}(2t)^{i-2k}.
\end{align}
\end{lemma}
\begin{proof}
We use mathematical induction. We already know that $C_0^\lambda(t)=1$ and $C_1^\lambda(t)=2\lambda t$, which both satisfy Eq.~\eqref{eq.temp_011301}. Suppose that $C_i^\lambda(t)$ and $C_{i+1}^\lambda(t)$ satisfy Eq.~\eqref{eq.temp_011301}, i.e.,
\begin{align*}
    &C_i^\lambda(t)=\sum_{k=0}^{\lfloor \frac{i}{2}\rfloor}(-1)^k\frac{\Gamma(i-k+\lambda)}{\Gamma(\lambda)k!(i-2k)!}(2t)^{i-2k},\\
    &C_{i+1}^\lambda(t)=\sum_{k=0}^{\lfloor \frac{i+1}{2}\rfloor}(-1)^k\frac{\Gamma(i-k+\lambda+1)}{\Gamma(\lambda)k!(i-2k+1)!}(2t)^{i-2k+1}.
\end{align*}

It remains to show that $C_{i+2}^\lambda(t)$ also satisfy Eq.~\eqref{eq.temp_011301}. By Eq.~\eqref{eq.temp_011302}, it suffices to show that
\begin{align}
    &(i+2)\sum_{k=0}^{\lfloor \frac{i+2}{2}\rfloor}(-1)^k\frac{\Gamma(i-k+\lambda+2)}{\Gamma(\lambda)k!(i-2k+2)!}(2t)^{i-2k+2}\nonumber\\
    =&2(\lambda+i+1)t\sum_{k=0}^{\lfloor \frac{i+1}{2}\rfloor}(-1)^k\frac{\Gamma(i-k+\lambda+1)}{\Gamma(\lambda)k!(i-2k+1)!}(2t)^{i-2k+1}\nonumber\\
    &-(2\lambda+i)\sum_{k=0}^{\lfloor \frac{i}{2}\rfloor}(-1)^k\frac{\Gamma(i-k+\lambda)}{\Gamma(\lambda)k!(i-2k)!}(2t)^{i-2k}.\label{eq.temp_011303}
\end{align}
To that end, it suffices to show that the coefficients of $(2t)^{i-2k+2}$ are the same for both sides of Eq.~\eqref{eq.temp_011303}, for $k=0,1,\cdots,\lfloor\frac{i+2}{2}\rfloor$. For the first step, we verify the coefficients of $(2t)^{i-2k+2}$ for $k=1,\cdots,\lfloor\frac{i+1}{2}\rfloor$. We have
\begin{align*}
    &\text{coefficients of $(2t)^{i-2k+2}$ on the  right-hand-side of Eq.~\eqref{eq.temp_011303}}\\
    =&(\lambda+i+1)(-1)^k\frac{\Gamma(i-k+\lambda+1)}{\Gamma(\lambda)k!(i-2k+1)!}-(2\lambda+i)(-1)^{k-1}\frac{\Gamma(i-k+\lambda+1)}{\Gamma(\lambda)(k-1)!(i-2k+2)!}\\
    =&(-1)^k\frac{\Gamma(i-k+\lambda+1)}{\Gamma(\lambda)k!(i-2k+2)!}\left((\lambda+i+1)(i-2k+2)+(2\lambda+i)k\right)\\
    =&(-1)^k\frac{\Gamma(i-k+\lambda+1)}{\Gamma(\lambda)k!(i-2k+2)!}\left((\lambda+i+1)(i+2)+(2\lambda+i)k-2k(\lambda+i+1)\right)\\
    =&(-1)^k\frac{\Gamma(i-k+\lambda+1)}{\Gamma(\lambda)k!(i-2k+2)!}\left((\lambda+i+1)(i+2)-k(i+2)\right)\\
    =&(-1)^k\frac{\Gamma(i-k+\lambda+1)}{\Gamma(\lambda)k!(i-2k+2)!}(\lambda-k+i+1)(i+2)\\
    =&(i+2)(-1)^k\frac{\Gamma(i-k+\lambda+2)}{\Gamma(\lambda)k!(i-2k+2)!}\\
    =&\text{coefficients of $(2t)^{i-2k+2}$ on the  left-hand-side of Eq.~\eqref{eq.temp_011303}}.
\end{align*}
For the second step, we verify the coefficient of $(2t)^{i-2k+2}$ for $k=0$, i.e., the coefficient of $(2t)^{i+2}$. We have
\begin{align*}
    &\text{coefficients of $(2t)^{i+2}$ on the  right-hand-side of Eq.~\eqref{eq.temp_011303}}\\
    =&(\lambda+i+1)\frac{\Gamma(i+\lambda+1)}{\Gamma(\lambda)(i+1)!}\\
    =&(i+2)\frac{\Gamma(i+2+\lambda)}{\Gamma(\lambda)(i+2)!}\\
    =&\text{coefficients of $(2t)^{i+2}$ on the left-hand-side of Eq.~\eqref{eq.temp_011303}}.
\end{align*}
For the third step, we verify the coefficient of $(2t)^{i-2k+2}$ for $k=\lfloor\frac{i+2}{2}\rfloor=\lfloor\frac{i}{2}\rfloor+1$. We consider two cases: 1) $i$ is even, and 2) $i$ is odd. When $i$ is even, we have $\lfloor\frac{i}{2}\rfloor+1=\frac{i}{2}+1$, i.e., $i-2k+2=0$. Thus, we have
\begin{align*}
    &\text{coefficients of $(2t)^0$ on the  right-hand-side of Eq.~\eqref{eq.temp_011303}}\\
    =&-(2\lambda+i)(-1)^{\frac{i}{2}}\frac{\Gamma\left(\frac{i}{2}+\lambda\right)}{\Gamma(\lambda)\left(\frac{i}{2}\right)!}\\
    =&(i+2)(-1)^{\frac{i}{2}+1}\frac{\Gamma\left(\frac{i}{2}+1+\lambda\right)}{\Gamma(\lambda)\left(\frac{i}{2}+1\right)!}\\
    =&\text{coefficients of $(2t)^0$ on the  left-hand-side of Eq.~\eqref{eq.temp_011303}}.
\end{align*}
When $i$ is odd, we have $k=\lfloor\frac{i}{2}\rfloor+1=\frac{i+1}{2}=\lfloor\frac{i+1}{2}\rfloor$ and this case has already been verified in the first step.

In conclusion, the coefficients of $(2t)^{i-2k+2}$ are the same for both sides of Eq.~\eqref{eq.temp_011303}, for $k=0,1,\cdots,\lfloor\frac{i+2}{2}\rfloor$. Thus, by mathematical induction, the result of this lemma thus follows.
\end{proof}

Applying Lemma~\ref{le.poly_Geg} in Eq.~\eqref{eq.temp_011304}, we have
\begin{align}
    \Xi_{\bm{0}}^l(\xi)=A_{\bm{0}}^l\sum_{k=0}^{\lfloor\frac{l}{2}\rfloor}(-1)^k\frac{\Gamma(l-k+\frac{d-2}{2})}{\Gamma(\frac{d-2}{2})k!(l-2k)!}(2\cos\theta_{d-1})^{l-2k}\label{eq.temp_011305}.
\end{align}
We give a few examples of $\Xi_{\mathbf{0}}^l(\xi)$ as follows.
\begin{align*}
    &\Xi_{\mathbf{0}}^0(\xi)=A_{\bm{0}}^0,\\
    &\Xi_{\mathbf{0}}^1(\xi)=A_{\bm{0}}^1(d-2)\cos\theta_{d-1},\\
    &\Xi_{\mathbf{0}}^2(\xi)=A_{\bm{0}}^2\frac{d-2}{2}\left(d\cos^2\theta_{d-1}-1\right),\\
    &\Xi_{\mathbf{0}}^3(\xi)=A_{\bm{0}}^3\frac{d-2}{2}\cdot d\cdot \left(\frac{d+2}{3}\cos^3\theta_{d-1}-\cos\theta_{d-1}\right).
\end{align*}

\subsection{Proof of Proposition~\ref{prop.coefficients}}\label{app.zero_component}
Recall that
\begin{align*}
    \hf(\xsmall)\defeq \xsmall^T\bm{e}\frac{\pi-\arccos (\xsmall^T\bm{e})}{2\pi},\quad \bm{e}\defeq [0\ 0\ \cdots\ 0\ 1]^T\in\mathds{R}^d.
\end{align*}
Notice that $\xsmall^T\bm{e}=\cos\theta_{d-1}$. Thus, we have
\begin{align*}
    h(\xsmall)=\cos\theta_{d-1}\frac{\pi-\arccos(\cos\theta_{d-1})}{2\pi}.
\end{align*}
The $\arccos$ function has a Taylor Series Expansion:
\begin{align*}
    \arccos (a) = \frac{\pi}{2}-\sum_{i=0}^\infty \frac{(2i)!}{2^{2i}(i!)^2}\frac{a^{2i+1}}{2i+1},
\end{align*}
which converges when $-1\leq a\leq 1$.
Thus, we have
\begin{align}\label{eq.temp_011401}
    h(\xsmall)=\frac{1}{4}\cos\theta_{d-1}+\frac{1}{2\pi}\sum_{i=0}^\infty\frac{(2i)!}{2^{2i}(i!)^2}\frac{\cos^{2i+2}\theta_{d-1}}{2i+1}.
\end{align}
By comparing terms of even and odd power of $\cos\theta_{d-1}$ in Eq.~\eqref{eq.temp_011305} and Eq.~\eqref{eq.temp_011401}, we immediately see that $h(\xsmall)\not\perp \Xi_{\bm{0}}^l(\xsmall)$ when $l=1$, and $h(\xsmall)\perp \Xi_{\bm{0}}^l(\xsmall)$ when $l=3, 5, 7,\cdots$. It remains to examine whether $h(\xsmall)\perp \Xi_{\bm{0}}^l(\xsmall)$ or $h(\xsmall)\not\perp \Xi_{\bm{0}}^l(\xsmall)$ for  $l\in\{0,1,2,4,6,\cdots\}$. We first introduce the following lemma.

\begin{lemma}\label{le.positive_component}
Let $a$ and $b$ be two non-negative integers. Define the function
\begin{align*}
    Q(a,b)\defeq\int_{\sd}\cos^{a}(\theta_{d-1})\Xi_{\bm{0}}^{b}(\xi)d\xdensity(\xsmall).
\end{align*}
We must have
\begin{align}\label{eq.temp_011701}
    Q(2k,2m)\begin{cases}
    >0,\text{ if $m\leq k$},\\
    =0,\text{ if $m>k$}.
    \end{cases}
\end{align}
\end{lemma}
\begin{proof}
We have
\begin{align*}
    Q(2k, 0)=\int_{\sd}\cos^{2k}(\theta_{d-1})\Xi_{\bm{0}}^0(\xi)d\xdensity(\xsmall)=A_{\bm{0}}^0 \int_{\sd}\cos^{2k}(\theta_{d-1})d\xdensity(\xsmall)>0.
\end{align*}
Thus, to finish the proof, we only need to consider the case of $m\geq 1$ in Eq.~\eqref{eq.temp_011701}.
We then prove by mathematical induction on the first parameter of $Q(\cdot,\cdot)$, i.e., $k$ in Eq.~\eqref{eq.temp_011701}.
When $m>0$, we have
\begin{align*}
    Q(0,2m)=&\int_{\sd}\Xi_{\bm{0}}^{2m}(\xi)d\xdensity(\xsmall)=\frac{1}{A_{\bm{0}}^0}\int_{\sd}\Xi_{\bm{0}}^{0}(\xi)\Xi_{\bm{0}}^{2m}(\xi)d\xdensity(\xsmall)=0\\
    &\text{ (by the orthogonality of the basis)}.
\end{align*}
Thus, Eq.~\eqref{eq.temp_011701} holds for all $m$ when $k=0$. Suppose that Eq.~\eqref{eq.temp_011701} holds when $k=i$. To complete the mathematical induction, it only remains to show that Eq.~\eqref{eq.temp_011701} also holds for all $m$ when $k=i+1$.
By Eq.~\eqref{eq.temp_011302} and Eq.~\eqref{eq.temp_011304}, for any $l$, we have
\begin{align*}
    \cos(\theta_{d-1})\Xi_{\bm{0}}^{l+1}(\xi)=\frac{(l+2)A_{\bm{0}}^{l+1}}{(d+2l)A_{\bm{0}}^{l+2}}\Xi_{\bm{0}}^{l+2}(\xi)+\frac{(d-2+l)A_{\bm{0}}^{l+1}}{(d+2l)A_{\bm{0}}^{l}}\Xi_{\bm{0}}^{l}(\xi).
\end{align*}
Thus, we have
\begin{align}\label{eq.temp_011702}
    Q(a+1,l+1)=q_{l,1}\cdot Q(a,l+2)+q_{l,2}\cdot Q(a,l),
\end{align}
where
\begin{align*}
    q_{l,1}\defeq \frac{(l+2)A_{\bm{0}}^{l+1}}{(d+2l)A_{\bm{0}}^{l+2}},\quad q_{l,2}\defeq \frac{(d-2+l)A_{\bm{0}}^{l+1}}{(d+2l)A_{\bm{0}}^{l}}.
\end{align*}
It is obvious that $q_{l,1}>0$ and $q_{l,2}>0$. Applying Eq.~\eqref{eq.temp_011702} multiple times, we have
\begin{align}
    &Q(2i+2,2m)=q_{2m-1,1}\cdot Q(2i+1,2m+1)+q_{2m-1,2}\cdot Q(2i+1,2m-1),\label{eq.temp_011703}\\
    &Q(2i+1,2m+1)=q_{2m,1}\cdot Q(2i,2m+2)+q_{2m,2}\cdot Q(2i,2m),\label{eq.temp_011704}\\
    &Q(2i+1,2m-1)=q_{2m-2,1}\cdot Q(2i,2m)+q_{2m-2,2}\cdot Q(2i,2m-2).\label{eq.temp_011705}
\end{align}
(Notice that we have already let $m\geq 1$, so all $q_{\cdot,1},q_{\cdot,2},Q(\cdot,\cdot)$ in those equations are well-defined.) By plugging Eq.~\eqref{eq.temp_011704} and Eq.~\eqref{eq.temp_011705} into Eq.~\eqref{eq.temp_011703}, we have
\begin{align}
    Q(2i+2,2m)=&q_{2m,1}q_{2m-1,1}Q(2i,2m+2)+(q_{2m-1,1}q_{2m,2}+q_{2m-1,2}q_{2m-2,1})Q(2i,2m)\nonumber\\
    &+q_{2m-1,2}q_{2m-2,2}Q(2i,2m-2).\label{eq.temp_011706}
\end{align}

To prove that Eq.~\eqref{eq.temp_011701} holds when $k=i+1$ for all $m$, we consider two cases, Case 1: $m\leq i+1$, and Case 2: $m>i+1$. Notice that by the induction hypothesis, we already know that Eq.~\eqref{eq.temp_011701} holds when $k=i$ for all $m$.

\emph{Case 1.} When $m\leq i+1$, we have $m-1\leq i$. Thus, by the induction hypothesis for $k=i$, we have $Q(2i,2m-2)>0$ (by $m-1\leq i$), which implies that the third term of the right-hand-side of Eq.~\eqref{eq.temp_011706} is positive. Further, by the induction hypothesis for $k=i$, we also know that $Q(2i,2m+2)\geq 0$ and $Q(2i,2m)\geq 0$ (regardless of the value of $m$), which means that the first and the second term of Eq.~\eqref{eq.temp_011706} is non-negative. Thus, by considering all three terms in Eq.~\eqref{eq.temp_011706} together, we have $Q(2i+2,2m)>0$ when $m\leq i+1$.

\emph{Case 2.} When $m>i+1$, we have $m+1>i$, $m>i$, and $m-1>i$. Thus, by the induction hypothesis for $k=i$, we have $Q(2i, 2m+2)=Q(2i,2m)=Q(2i,2m-2)=0$. Therefore, by Eq.~\eqref{eq.temp_011706}, we have $Q(2i+2,2m)=0$.

In summary, Eq.~\eqref{eq.temp_011701} holds when $k=i+1$ for all $m$. The mathematical induction is completed and the result of this lemma follows.
\end{proof}




By Lemma~\ref{le.positive_component}, for all $k\geq 0$, we have
\begin{align*}
    &\int_{\sd}\frac{1}{2\pi}\sum_{i=0}^\infty\frac{(2i)!}{2^{2i}(i!)^2}\frac{\cos^{2i+2}\theta_{d-1}}{2i+1}\Xi_{\bm{0}}^{2k}(\xi) d\xdensity(\xsmall)\\
    =&\frac{1}{2\pi}\sum_{i=0}^\infty\frac{(2i)!}{2^{2i}(i!)^2}\frac{1}{2i+1}\int_{\sd}\cos^{2i+2}\theta_{d-1}\Xi_{\bm{0}}^{2k}(\xi) d\xdensity(\xsmall)\\
    >&0.
\end{align*}
Thus, by Eq.~\eqref{eq.temp_011401}, we know that $h(\xsmall)\not\perp \Xi_{\bm{0}}^l(\xsmall)$ for all $l\in\{0,2,4,\cdots\}$.

\subsection{A special case: when \texorpdfstring{$d=2$}{d=2}}\label{app.special_case_d2}




When $d=2$, $\sd$ denotes a unit circle. Therefore, every $\xsmall$ corresponds to an angle $\varphi\in [-\pi,\ \pi]$ such that $\xsmall=[\cos\varphi\ \sin\varphi]^T$. 
In this situation, the hyper-spherical harmonics are the well-known Fourier series, i.e., $1,\cos(\theta),\sin(\theta),\cos(2\theta),\sin(2\theta),\cdots$. Thus, we can explicitly calculate all Fourier coefficients of $h$ more easily. 

Similarly to Appendix~\ref{app.convolution}, we first write down the convolution for $d=2$, which is also in a simpler form.  
For any function $\fg\in\learnableSet$, we have
\begin{align*}
    \fg(\varphi)&=\frac{1}{2\pi}\int_{\varphi-\pi}^{\varphi+\pi}\frac{\pi-|\theta-\varphi|}{2\pi}\cos(\theta - \varphi)g(\theta) d\theta\\
    &=\frac{1}{2\pi}\int_{-\pi}^{\pi}\frac{\pi-|\theta|}{2\pi}\cos \theta\ g(\theta + \varphi)\ d\theta\text{ (replace $\theta$ by $\theta-\varphi$)}\\
    &=\frac{1}{2\pi}\int_{-\pi}^{\pi}\frac{\pi-|\theta|}{2\pi}\cos \theta\ g(\varphi-\theta)\ d\theta\text{ (replace $\theta$ by $-\theta$)}.
\end{align*}
Define $h(\theta)\defeq \frac{\pi-|\theta|}{2\pi}\cos \theta$. We then have
\begin{align*}
    \fg(\varphi)=\frac{1}{2\pi}h(\varphi)\circledast g(\varphi),
\end{align*}
where $\circledast$ denotes (continuous) circular convolution. Let $c_{\fg}(k), c_h(k)$ and $ c_g(k)$ (where $k=\cdots,-1,0,1,\cdots$) denote the (complex) Fourier series coefficients for $\fg(\varphi)$, $h(\varphi)$, and $g(\varphi)$, correspondingly. Specifically, we have
\begin{align*}
    \fg(\varphi)=\sum_{k=-\infty}^{\infty}c_{\fg}(k)e^{ik\varphi},\quad h(\varphi)=\sum_{k=-\infty}^{\infty}c_h(k)e^{ik\varphi},\quad g(\varphi)=\sum_{k=-\infty}^{\infty}c_g(k)e^{ik\varphi}.
\end{align*}
Thus, we have
\begin{align}\label{eq.temp_102401}
    c_{\fg}(k)= c_h(k)c_g(k).
\end{align}
Now we calculate $c_h(k)$, i.e., the Fourier decomposition of $h(\cdot)$. We have
\begin{align*}
    c_h(k)&=\frac{1}{2\pi}\int_{-\pi}^{\pi}\frac{\pi-|\theta|}{2\pi}\cos \theta\  e^{-ik\theta}d\theta\\
    &=\frac{1}{4\pi}\int_{-\pi}^{\pi}\left(1-\frac{|\theta|}{\pi}\right)\frac{e^{-i(k+1)\theta}+e^{-i(k-1)\theta}}{2}d \theta\\
    &=-\frac{1}{8\pi^2}\int_{-\pi}^{\pi}|\theta|\left(e^{-i(k+1)\theta}+e^{-i(k-1)\theta}\right) d\theta+\frac{1}{8\pi}\int_{-\pi}^{\pi}\left(e^{-i(k+1)\theta}+e^{-i(k-1)\theta}\right) d\theta.
\end{align*}
It is easy to verify that
\begin{align*}
    \int x e^{cx}dx = e^{cx}\left(\frac{cx-1}{c^2}\right),\quad \forall c\neq 0.
\end{align*}
Thus, we have
\begin{align*}
    c_h(1)&=-\frac{1}{8\pi^2}\int_{-\pi}^{\pi}|\theta|\left(e^{-i2\theta}+1\right)d\theta+\frac{1}{4}\\
    &=-\frac{1}{8\pi^2}\left(\pi^2-\int_{-\pi}^0\theta e^{-i2\theta}d\theta+\int_0^{\pi}\theta e^{-i2\theta}d\theta\right)+\frac{1}{4}\\
    &=-\frac{1}{8\pi^2}\left(\pi^2+\frac{i2\pi}{-4}+\frac{-i2\pi}{-4}\right)+\frac{1}{4}\\
    &=-\frac{1}{8}+\frac{1}{4}\\
    &=\frac{1}{8}.
\end{align*}
Similarly, we have
\begin{align*}
    c_h(-1)=\frac{1}{8}.
\end{align*}
Now we consider the situation of $n\neq \pm 1$. We have
\begin{align*}
    &\int_{-\pi}^0 |\theta|e^{-i(k+1)\theta}d\theta=-e^{-i(k+1)\theta}\cdot\frac{-i(k+1)\theta-1}{-(k+1)^2}\ \Bigg|_{-\pi}^0=-\frac{1}{(k+1)^2}+\frac{1-i(k+1)\pi}{(k+1)^2}e^{i(k+1)\pi},\\
    &\int_{0}^{\pi} |\theta|e^{-i(k+1)\theta}d\theta=e^{-i(k+1)\theta}\cdot\frac{-i(k+1)\theta-1}{-(k+1)^2}\ \Bigg|_0^\pi=-\frac{1}{(k+1)^2}+\frac{1+i(k+1)\pi}{(k+1)^2}e^{-i(k+1)\pi}.
\end{align*}
Notice that $e^{-i(k+1)\pi}=e^{-i(k+1)2\pi}e^{i(k+1)\pi}=e^{i(k+1)\pi}$. Therefore, we have
\begin{align*}
    \int_{-\pi}^\pi |\theta| e^{-i(k+1)\theta}d\theta=\frac{2}{(k+1)^2}\left(e^{i(k+1)\pi}-1\right).
\end{align*}
Similarly, we have
\begin{align*}
    \int_{-\pi}^\pi |\theta| e^{-i(k-1)\theta}d\theta=\frac{2}{(k-1)^2}\left(e^{i(k-1)\pi}-1\right).
\end{align*}
In summary, we have
\begin{align*}
    c_h(k)&=\begin{cases}
    \frac{1}{8},\quad &k=\pm 1\\
    -\frac{1}{4\pi^2}\left(\frac{1}{(k+1)^2}+\frac{1}{(k-1)^2}\right)\left(e^{i(k+1)\pi}-1\right),\quad &\text{otherwise}
    \end{cases}\\
    &=\begin{cases}
    \frac{1}{8},\quad &k=\pm 1\\
    \frac{1}{2\pi^2}\left(\frac{1}{(k+1)^2}+\frac{1}{(k-1)^2}\right),\quad &k=0,\pm 2,\pm 4,\cdots\\
    0,\quad &k= \pm 3,\pm 5,\cdots
    \end{cases}.
\end{align*}

By Eq.~\eqref{eq.temp_102401}, we thus have
\begin{align*}
    c_{\fg}(k)&=\begin{cases}
    \frac{1}{8}c_g(k),\quad &k=\pm 1\\
    \frac{1}{2\pi^2}\left(\frac{1}{(k+1)^2}+\frac{1}{(k-1)^2}\right)c_g(k),\quad &k=0,\pm 2,\pm 4,\cdots\\
    0,\quad &k= \pm 3,\pm 5,\cdots
    \end{cases}.
\end{align*}
In other words, when $d=2$, functions in $\learnableSet$ can only contain frequencies $0, \theta, 2\theta, 4\theta,6\theta,\cdots$, and cannot contain other frequencies $3\theta,5\theta,7\theta,\cdots$.

\subsection{Details of Remark~\ref{remark.bias_enlarge}}\label{proof.bias_enlarge}

As we discussed in Remark~\ref{remark.bias_enlarge}, a ReLU activation function with bias that operates on $\tilde{\xsmall}\in\mathds{R}^{d-1}$, $\|\tilde{\xsmall}\|_2^2=\frac{d-1}{d}$ can be equivalently viewed as one without bias that operates on $\xsmall\in\sd$, but with the last element of $\xsmall$ fixed at $1/\sqrt{d}$. Note that
by fixing the last element of $\xsmall\in\sd$ at a constant $\frac{1}{\sqrt{d}}$, we essentially consider ground-truth functions with a much smaller domain $ \mathcal{D}\defeq \left\{\xsmall=\begin{bsmallmatrix}\tilde{\xsmall}\\ 1/\sqrt{d}\end{bsmallmatrix}\ \big|\ \tilde{\xsmall}\in\mathds{R}^{d-1},\|\tilde{\xsmall}\|_2^2=\frac{d-1}{d}\right\}\subset \sd$.   Correspondingly, define a vector $\tilde{\bm{a}}\in \mathds{R}^{d-1}$ and $a_0\in\mathds{R}$ such that $\bm{a}=\begin{bsmallmatrix}\tilde{\bm{a}}\\ a_0\end{bsmallmatrix}\in\mathds{R}^d$. We claim that for any $\bm{a}\in\mathds{R}^d$ and for all non-negative integer $l$, a ground-truth function $f(\xsmall)=(\xsmall^T\bm{a})^l, \xsmall\in\mathcal{D}$ must be learnable. In other words, all polynomials can be learned in the constrained domain $\mathcal{D}$. 
Towards this end, recall that we have already shown that polynomials (of $\xsmall\in\sd$) to the power of $l=0,1,2,4,6,\cdots$ are learnable. Thus, it suffices to prove that polynomials of $\xsmall\in\mathcal{D}$ to the power of $l=3,5,7,\cdots$ can be represented by a finite sum of those to the power of $l=0,1,2,4,6,\cdots$. The idea is to utilize the fact that the binomial expansion of $(\tilde{\xsmall}^T\tilde{\bm{a}}+\frac{a_0}{\sqrt{d}})^l$ contains $(\tilde{\xsmall}^T\tilde{\bm{a}})^k$ for all $k=0,1,2,3,\cdots,l$. Here we give an example for writing $(\xsmall^T\bm{a})^3$ as a linear combination of learnable components. Other values of $l=5,7,9,\cdots$ can be proved in a similar way.
Notice that
\begin{align}
    (\tilde{\xsmall}^T\tilde{\bm{a}})^3=&\frac{1}{4}\left((\tilde{\xsmall}^T\tilde{\bm{a}}+1)^4-(\tilde{\xsmall}^T\tilde{\bm{a}})^4-6(\tilde{\xsmall}^T\tilde{\bm{a}})^2-4(\tilde{\xsmall}^T\tilde{\bm{a}})^2-1\right)\text{ (by the binomial expansion of $(\tilde{\xsmall}^T\tilde{\bm{a}}+1)^4$)}\nonumber\\
    =&\frac{1}{4}\left(\left(\xsmall^T\begin{bmatrix}\tilde{\bm{a}}\\ \sqrt{d}\end{bmatrix}\right)^4-\left(\xsmall^T\begin{bmatrix}\tilde{\bm{a}}\\ 0\end{bmatrix}\right)^4-6\left(\xsmall^T\begin{bmatrix}\tilde{\bm{a}}\\ 0\end{bmatrix}\right)^2-4\left(\xsmall^T\begin{bmatrix}\tilde{\bm{a}}\\ 0\end{bmatrix}\right)-1\right).\label{eq.temp_031601}
\end{align}
Thus, for all $\xsmall=\begin{bsmallmatrix}\tilde{\xsmall}\\ 1/\sqrt{d}\end{bsmallmatrix}$ and $\bm{a}=\begin{bsmallmatrix}\tilde{\bm{a}}\\ a_0\end{bsmallmatrix}$, we have
\begin{align*}
    (\xsmall^T\bm{a})^3=&\left(\tilde{\xsmall}^T\tilde{\bm{a}}+\frac{a_0}{\sqrt{d}}\right)^3\\
    =&(\tilde{\xsmall}^T\tilde{\bm{a}})^3+3\left(\frac{a_0}{\sqrt{d}}\right)(\tilde{\xsmall}^T\tilde{\bm{a}})^2+3\left(\frac{a_0}{\sqrt{d}}\right)^2(\tilde{\xsmall}^T\tilde{\bm{a}})+\left(\frac{a_0}{\sqrt{d}}\right)^3\\
    =&(\tilde{\xsmall}^T\tilde{\bm{a}})^3+3\left(\frac{a_0}{\sqrt{d}}\right)\left(\xsmall^T\begin{bmatrix}\tilde{\bm{a}}\\ 0\end{bmatrix}\right)^2+3\left(\frac{a_0}{\sqrt{d}}\right)^2\left(\xsmall^T\begin{bmatrix}\tilde{\bm{a}}\\ 0\end{bmatrix}\right)+\left(\frac{a_0}{\sqrt{d}}\right)^3\\
    =&\frac{1}{4}\left(\xsmall^T\begin{bmatrix}\tilde{\bm{a}}\\ \sqrt{d}\end{bmatrix}\right)^4-\frac{1}{4}\left(\xsmall^T\begin{bmatrix}\tilde{\bm{a}}\\ 0\end{bmatrix}\right)^4+\left(3\left(\frac{a_0}{\sqrt{d}}\right)-\frac{3}{2}\right)\left(\xsmall^T\begin{bmatrix}\tilde{\bm{a}}\\ 0\end{bmatrix}\right)^2\\
    &+\left(3\left(\frac{a_0}{\sqrt{d}}\right)^2-1\right)\left(\xsmall^T\begin{bmatrix}\tilde{\bm{a}}\\ 0\end{bmatrix}\right)+\left(\left(\frac{a_0}{\sqrt{d}}\right)^3-\frac{1}{4}\right)\text{ (by Eq.~\eqref{eq.temp_031601})},
\end{align*}
which is a sum of $5$ learnable components (corresponding to the polynomials with power of $4$, $4$, $2$, $1$, and $0$, respectively).

\section{Discussion when \texorpdfstring{$g$}{g} is a \texorpdfstring{$\delta$}{delta}-function (\texorpdfstring{$\|g\|_\infty=\infty$}{g is not bounded})}\label{app.g_is_delta}

We now discuss what happens to the conclusion of Theorem~\ref{th.combine} if $g$ contains a $\delta$-function, in which case $\|g\|_\infty=\infty$.
In Eq.~\eqref{eq.main_conclusion} of Theorem~\ref{th.combine}, only Term 1 and Term 4 (come from Proposition~\ref{th.fixed_Vzero}) will be affected when $\|g\|_\infty=\infty$. That is because only Proposition~\ref{th.fixed_Vzero} requires $\|g\|_\infty<\infty$ during the proof of Theorem~\ref{th.combine}. To accommodate the situation when $g$ contains a $\delta$-function ($\|g\|_\infty=\infty$), we need a new version of Proposition~\ref{th.fixed_Vzero}. In other words, we need to know the performance of the overfitted NTK solution in learning the pseudo ground-truth when $\|g\|_\infty=\infty$.

Without loss of generality, we consider the situation that $g=\delta_{\zzero}$. 
We have the following proposition.
\begin{proposition}\label{prop.g_is_delta}
If the ground-truth function is $\f=\fv$ in Definition~\ref{def.fv} with $g=\delta_{\zzero}$ and $\esmall=\bm{0}$, for any $\xsmall\in\sd$ and $q\in(1,\ \infty)$, we have
\begin{align*}
    &\prob_{\XX,\Vzero}\left\{|\fl(\xsmall)-\f(\xsmall)|\leq \left(\sqrt{\frac{3}{4}+\frac{\pi^2}{2}}\right)\left((d-1)B(\frac{d-1}{2},\frac{1}{2})\right)^{\frac{1}{2(d-1)}}n^{-\frac{1}{2(d-1)}(1-\frac{1}{q})}\right\}\\
    \geq & 1 - \exp\left(-n^{\frac{1}{q}}\right) - 2\exp\left(-\frac{p}{24}\left((d-1)B(\frac{d-1}{2},\frac{1}{2})\right)^{\frac{1}{d-1}}n^{-\frac{1}{d-1}(1-\frac{1}{q})}\right),
\end{align*}
when
\begin{align}\label{eq.temp_030502}
    n\geq \left((d-1)B(\frac{d-1}{2},\frac{1}{2})\right)^{\frac{q}{q-1}}\text{, i.e., }\left((d-1)B(\frac{d-1}{2},\frac{1}{2})\right)n^{-(1-\frac{1}{q})}\leq 1.
\end{align}
(Estimates of $B(\frac{d-1}{2},\frac{1}{2})$ can be found in Lemma~\ref{le.bound_B}.)
\end{proposition}

Proposition~\ref{prop.g_is_delta} implies that when $n$ is large and $p$ is much larger than $n^{-\frac{1}{2(d-1)}(1-\frac{1}{q})}$, the test error between the pseudo ground-truth and learned result decreases with $n$ at the speed $O(n^{-\frac{1}{2(d-1)}(1-\frac{1}{q})})$. Further, if we let $q$ be large, then the decreasing speed with $n$ is almost $O(n^{-\frac{1}{2(d-1)}})$. When $d\geq 3$, this speed is slower than $O(n^{-\frac{1}{2}})$ described in Proposition~\ref{th.fixed_Vzero} (i.e., Term 1 in Eq.~\eqref{eq.main_conclusion} of Theorem~\ref{th.combine}). When $d=2$, the decreasing speed with respect to $n$ is $O(n^{-\frac{1}{2}})$ for both Proposition~\ref{th.fixed_Vzero} and Proposition~\ref{prop.g_is_delta}. Nonetheless, Proposition~\ref{prop.g_is_delta} implies that the ground-truth functions $\fg\in\learnableSet$ is still learnable even when $g$ is a $\delta$-function (i.e., $\|g\|_\infty=\infty$), but the test error potentially suffers a slower convergence speed with respect to $n$ when $d$ is large.

\begin{figure}[t]
\centering
\includegraphics[width=3in]{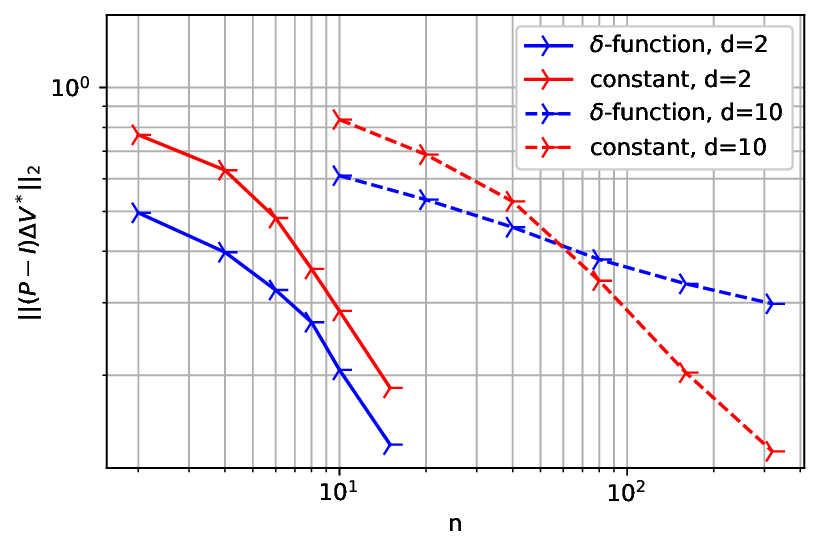}
\caption{The curves of the model error $\|(\mathbf{P}-\mathbf{I})\DVs\|_2$ for learning the pseudo ground-truth $\fv$ with respect to $n$ for different $g$ and different $d$, where $p=20000$, and $\esmall=\bm{0}$. Every curve is the average of 10 random simulation runs.}\label{fig.delta}
\end{figure}

In Fig.~\ref{fig.delta}, we plot the curves of the model error $\|(\mathbf{P}-\mathbf{I})\DVs\|_2$ for learning the pseudo ground-truth $\fv$ with respect to $n$ when $g=\delta_{\zzero}$ (two blue curves) and when $g$ is constant (two red curves). We plot both the case when $d=2$ (two solid curves) and the case when $d=10$ (two dashed curves). By Lemma~\ref{le.f_minus_f}, the model error $\|(\mathbf{P}-\mathbf{I})\DVs\|_2$ can represent the generalization performance for learning the pseudo ground-truth $\fv$ when there is no noise. In Fig.~\ref{fig.delta}, we can see that those two curves corresponding to $d=10$ have different slopes and the other two curves corresponding to $d=2$ have a similar slope, which confirms our prediction in the earlier paragraph (i.e., when $d=2$ the test error will decay at the same speed regardless of whether $g$ contains a $\delta$-function or not, but when $d>2$ the test error will decay more slowly when $g$ contains a $\delta$-function). 

\subsection{Proof of Proposition~\ref{prop.g_is_delta}}
We first show two useful lemmas.
\begin{lemma}\label{le.temp_030301}
For any $q\in(1,\infty)$, if $b\in[ n^{-(1-1/q)}, 1]$, then
\begin{align*}
    (1-b)^n\leq \exp\left(-n^{\frac{1}{q}}\right).
\end{align*}
\end{lemma}
\begin{proof}
By Lemma~\ref{le.lnx}, we have
\begin{align*}
    &e^{-b}\geq 1-b\\
    \implies &e^{-1}\geq (1-b)^{\frac{1}{b}}\\
    \implies &\exp\left(-n^{\frac{1}{q}}\right)\geq (1-b)^{n^{\frac{1}{q}}/ b}\\
    \implies &\exp\left(-n^{\frac{1}{q}}\right)\geq (1-b)^n\text{ because $b\in[ n^{-(1-1/q)}, 1]$}.
\end{align*}
\end{proof}

\begin{lemma}\label{le.temp_030501}
Consider $\xsmall_1\in\sd$ where $\varphi=\arccos (\xsmall_1^T\zzero)$. For any $\theta\in [\varphi, \pi]$, there must exist $\xsmall_2\in\sd$ such that $\arccos(\xsmall_2^T\zzero)=\theta$ and
\begin{align}\label{eq.temp_030501}
    \mathcal{C}_{-\xsmall_1,\zzero}^{\Vzero}\subseteq \mathcal{C}_{-\xsmall_2,\zzero}^{\Vzero},\quad \mathcal{C}_{\xsmall_1,-\zzero}^{\Vzero}\subseteq \mathcal{C}_{\xsmall_2,-\zzero}^{\Vzero}.
\end{align}
\end{lemma}
We will explain the intuition of Lemma~\ref{le.temp_030501} in Remark~\ref{re.temp_030801} right after we use the lemma. We put the proof of Lemma~\ref{le.temp_030501} in Section~\ref{subsec.le030501}.

Now we are ready to prove Proposition~\ref{prop.g_is_delta}.
Recall $\DVs$ defined in Eq.~\eqref{eq.temp_0929}. By Eq.~\eqref{eq.def_hx} and $g=\delta_{\zzero}$, we have
\begin{align*}
    \DVs = \frac{(\hsmall_{\Vzero,\zzero})^T}{p}.
\end{align*}
Define
\begin{align*}
    &i^*=\argmin_{i\in\{1,2,\cdots,n\}} \|\Xii-\zzero\|_2,\\
    &\theta^*=\arccos(\XX_{i^*}^T\zzero).
\end{align*}
Thus, we have
\begin{align}\label{eq.temp_030201}
    \|\XX_{i^*}-\zzero\|_2=&\sqrt{2-2\cos\theta^*}\text{ (by the law of cosines)}\nonumber\\
    =&2\sin\frac{\theta^*}{2}\text{ (by the half angle identity)}\nonumber\\
    \leq &\theta^*\text{ (by Lemma~\ref{le.sin})}.
\end{align}
(Graphically, Eq.~\eqref{eq.temp_030201} means that a chord is not longer than the corresponding arc.)

As we discussed in the proof sketch of Proposition~\ref{th.fixed_Vzero}, we now construct the vector $\bm{a}$ such that $\Ht^T\bm{a}$ is close to $\DVs$. Define $\bm{a}\in\mathds{R}^n$ whose $i$-th element is \begin{align*}
    \bm{a}_i=\begin{cases}
    1/p,\quad &\text{ if }i=i^*\\
    0,&\text{ if }i\in \{1,2,\cdots,n\}\setminus \{i^*\}
    \end{cases}.
\end{align*}
Thus, we have $\Ht^T\bm{a}=(\hsmall_{\Vzero,\XX_{i^*}})^T/p$. Therefore, we have
\begin{align*}
    \|\Ht^T\bm{a}-\DVs\|_2^2=&\sum_{j=1}^p\|(\Ht^T\bm{a})[j]-\DVs[j]\|_2^2\\
    =&\frac{1}{p^2}\sum_{j=1}^p\left(\bm{1}_{\{\XX_{i^*}^T\Vzeroj>0,\zzero^T\Vzeroj>0\}}\|\XX_{i^*}-\zzero\|_2^2+\bm{1}_{\{(\XX_{i^*}^T\Vzeroj)(\zzero^T\Vzeroj)<0\}}\right)\\
    \leq & \frac{1}{p^2}\left(p\|\XX_{i^*}-\zzero\|_2^2+|\mathcal{C}_{-\XX_{i^*},\zzero}^{\Vzero}|+|\mathcal{C}_{\XX_{i^*},-\zzero}^{\Vzero}|\right)\text{ (by Eq.~\eqref{eq.card_c})}\\
    \leq & \frac{1}{p^2}\left(p\cdot(\theta^*)^2+|\mathcal{C}_{-\XX_{i^*},\zzero}^{\Vzero}|+|\mathcal{C}_{\XX_{i^*},-\zzero}^{\Vzero}|\right)\text{ (by Eq.~\eqref{eq.temp_030201})}.
\end{align*}
Thus, we have
\begin{align}
    \sqrt{p}\|\Ht\bm{a}-\DVs\|_2\leq& \sqrt{(\theta^*)^2+\frac{|\mathcal{C}_{-\XX_{i^*},\zzero}^{\Vzero}|+|\mathcal{C}_{\XX_{i^*},-\zzero}^{\Vzero}|}{p}}\nonumber\\
    \leq & \sqrt{\pi\theta^*+\frac{|\mathcal{C}_{-\XX_{i^*},\zzero}^{\Vzero}|+|\mathcal{C}_{\XX_{i^*},-\zzero}^{\Vzero}|}{p}}\text{ (because $\theta^*\leq \pi$)}.\label{eq.temp_030303}
\end{align}

\begin{remark}\label{re.geo_interp_orth_region}
We give a geometric interpretation of Eq.~\eqref{eq.temp_030303} when $d=2$ by Fig.~\ref{fig.2D}, where  $\overrightarrow{\mathrm{OA}}$ denotes $\zzero$, $\overrightarrow{\mathrm{OB}}$ denotes $\XX_{i^*}$. Then, $|\mathcal{C}_{-\XX_{i^*},\zzero}^{\Vzero}|+|\mathcal{C}_{\XX_{i^*},-\zzero}^{\Vzero}|$ corresponds to the number of $\Vzeroj$'s whose direction is in the arc $\stackrel{\frown}{\mathrm{CE}}$ or the arc $\stackrel{\frown}{\mathrm{FD}}$, and $\theta^*$ corresponds to the angle $\angle\mathrm{AOB}$. Intuitively, when $n$ increases, $\XX_{i^*}$ and $\zzero$ get closer, so $\theta^*$ becomes smaller. At the same time, both the arc $\stackrel{\frown}{\mathrm{CE}}$ and the arc $\stackrel{\frown}{\mathrm{FD}}$ become shorter. Consequently,  the value of Eq.~\eqref{eq.temp_030303} decreases as $n$ increases. In the rest of the proof, we will quantitatively estimate the above relationship.
\end{remark}

Recall $C_d$ in Eq.~\eqref{eq.def_Cd}. Define
\begin{align}\label{eq.temp_030305}
    \theta \defeq \frac{\pi}{2} \left(\frac{2\sqrt{2}(d-1)}{C_d}\right)^{\frac{1}{d-1}}n^{-\frac{1}{d-1}(1-\frac{1}{q})}\in \left[0, \frac{\pi}{2}\right]\quad \text{ (by Eq.~\eqref{eq.temp_030502})}.
\end{align}
For any $q\in(1, \infty)$, we define two events:
\begin{align*}
    &\mathcal{J}_1 \defeq \left\{\frac{|\mathcal{C}_{-\XX_{i^*},\zzero}^{\Vzero}|+|\mathcal{C}_{\XX_{i^*},-\zzero}^{\Vzero}|}{p}\leq \frac{3\theta}{2\pi}\right\},\\
    &\mathcal{J}_2 \defeq \left\{\theta^*\leq\theta\right\}.
\end{align*}
If both $\mathcal{J}_1$ and $\mathcal{J}_2$ happen, by Eq.~\eqref{eq.temp_030303}, we must then have
\begin{align*}
    \sqrt{p}\|\Ht\bm{a}-\DVs\|_2&\leq\left(\sqrt{\frac{3}{2\pi}+\pi}\right)\cdot\sqrt{\theta}\\
    &= \left(\sqrt{\frac{3}{4}+\frac{\pi^2}{2}}\right)\left(\frac{2\sqrt{2}(d-1)}{C_d}\right)^{\frac{1}{2(d-1)}}n^{-\frac{1}{2(d-1)}(1-\frac{1}{q})}.
\end{align*}
Thus, by Lemma~\ref{le.f_minus_f} and Lemma~\ref{le.hxDV}, if $\f=\fv$ and both $\mathcal{J}_1$ and $\mathcal{J}_2$ happen, then for any $\xsmall\in\sd$, we must have
\begin{align}\label{eq.temp_030306}
    |\fl(\xsmall)-\f(\xsmall)|\leq \left(\sqrt{\frac{3}{4}+\frac{\pi^2}{2}}\right)\left(\frac{2\sqrt{2}(d-1)}{C_d}\right)^{\frac{1}{2(d-1)}}n^{-\frac{1}{2(d-1)}(1-\frac{1}{q})}.
\end{align}

It then only remains to estimate the probability of $\mathcal{J}_1\cap \mathcal{J}_2$.

\textbf{Step 1: Estimate the probability of $\mathcal{J}_1$ conditional on $\mathcal{J}_2$.}

When $\mathcal{J}_2$ happens, we have $\theta^*<\theta$. By Lemma~\ref{le.temp_030501}, we can find $\xsmall\in\sd$ such that the angle between $\xsmall$ and $\zzero$ is exactly $\theta$ and
\begin{align}\label{eq.temp_030503}
    \frac{|\mathcal{C}_{-\XX_{i^*},\zzero}^{\Vzero}|+|\mathcal{C}_{\XX_{i^*},-\zzero}^{\Vzero}|}{p}\leq \frac{|\mathcal{C}_{-\xsmall,\zzero}^{\Vzero}|+|\mathcal{C}_{\xsmall,-\zzero}^{\Vzero}|}{p}.
\end{align}

\begin{remark}\label{re.temp_030801}
We give a geometric  interpretation of Eq.~\eqref{eq.temp_030503} (i.e., Lemma~\ref{le.temp_030501}) when $d=2$ by Fig.~\ref{fig.2D}. Recall in Remark~\ref{re.geo_interp_orth_region} that, if we take $\overrightarrow{\mathrm{OA}}$ as $\zzero$ and $\overrightarrow{\mathrm{OB}}$ as $\XX_{i^*}$, then $|\mathcal{C}_{-\XX_{i^*},\zzero}^{\Vzero}|+|\mathcal{C}_{\XX_{i^*},-\zzero}^{\Vzero}|$ corresponds to the number of $\Vzeroj$'s whose direction is in the arc $\stackrel{\frown}{\mathrm{CE}}$ or the arc $\stackrel{\frown}{\mathrm{FD}}$. If we fix $\overrightarrow{\mathrm{OA}}$ (i.e., $\zzero$) and increase the angle $\angle\mathrm{AOB}$ (corresponding to $\theta^*$), then both the arc $\stackrel{\frown}{\mathrm{CE}}$ and the arc $\stackrel{\frown}{\mathrm{FD}}$ will become longer. In other words, if we replace $\Xii$ by $\xsmall$ such that the angle $\theta^*$ (between $\zzero$ and $\Xii$) increases to the angle $\theta$ (between $\zzero$ and $\xsmall$), then $\mathcal{C}_{-\XX_{i^*},\zzero}^{\Vzero}\subseteq \mathcal{C}_{-\xsmall,\zzero}^{\Vzero}$ and $\mathcal{C}_{\XX_{i^*},-\zzero}^{\Vzero}\subseteq \mathcal{C}_{\xsmall,-\zzero}^{\Vzero}$, and thus Eq.~\eqref{eq.temp_030503} follows.
\end{remark}

We next estimate the probability that the right-hand-side of Eq.~\eqref{eq.temp_030503} is greater than $\frac{3\theta}{2\pi}$. By Eq.~\eqref{eq.card_c}, we have 
\begin{align}\label{eq.temp_030304}
    \frac{|\mathcal{C}_{-\xsmall,\zzero}^{\Vzero}|+|\mathcal{C}_{\xsmall,-\zzero}^{\Vzero}|}{p}=\frac{1}{p}\sum_{j=1}^p \underbrace{\bm{1}_{\{-\xsmall^T\Vzeroj>0,\zzero^T\Vzeroj>0\text{ OR }\xsmall^T\Vzeroj>0,-\zzero^T\Vzeroj<0\}}}_{\text{Term A}}.
\end{align}
Notice that the angle between $-\xsmall$ and $\zzero$ is $\pi-\theta$, and the angle between $\xsmall$ and $-\zzero$ is also $\pi-\theta$.
By Lemma~\ref{le.spherePortion} and Assumption~\ref{as.uniform}, we know that the Term A in Eq.~\eqref{eq.temp_030304} follows Bernoulli distribution with the probability $2\cdot\frac{\pi-(\pi-\theta)}{2\pi}=\frac{\theta}{\pi}$. By letting $\delta = 1/2$, $a=p$, $b=\frac{\theta}{\pi}$ in Lemma~\ref{le.bino}, we have
\begin{align*}
    \prob_{\Vzero}\left\{\left||\mathcal{C}_{-\xsmall,\zzero}^{\Vzero}|+|\mathcal{C}_{\xsmall,-\zzero}^{\Vzero}|-\frac{p\theta}{\pi}\right|>\frac{p\theta}{2\pi}\right\}\leq 2\exp\left(-\frac{p\theta}{12\pi}\right).
\end{align*}
By Eq.~\eqref{eq.temp_030503}, we then have
\begin{align*}
    \prob_{\Vzero}[\mathcal{J}_1^c\ |\ \mathcal{J}_2]\leq \prob_{\Vzero}\left\{\frac{|\mathcal{C}_{-\xsmall,\zzero}^{\Vzero}|+|\mathcal{C}_{\xsmall,-\zzero}^{\Vzero}|}{p}>\frac{3\theta}{2\pi}\right\}\leq 2\exp\left(-\frac{p\theta}{12\pi}\right).
\end{align*}

\textbf{Step 2: Estimate the probability of $\mathcal{J}_2$.}

By Lemma~\ref{le.original_cap} and Assumption~\ref{as.uniform}, for any $i\in\{1,2,\cdots,n\}$ and because $\theta\in [0, \pi/2]$, we have
\begin{align*}
    \prob_{\XX}\left\{\arccos (\Xii^T\zzero)>\theta\right\}=&1-\frac{1}{2}I_{\sin^2\theta}\left(\frac{d-1}{2},\frac{1}{2}\right)\\
    \leq &1-\frac{C_d}{2\sqrt{2}(d-1)}\sin^{d-1}\theta\text{ (by Lemma~\ref{le.estimate_Ix})}.
\end{align*}
Note that since $\prob_{\XX}\left\{\arccos (\Xii^T\zzero)>\theta\right\}\geq 0$, we must have
\begin{align}\label{eq.temp_052801}
    \frac{C_d}{2\sqrt{2}(d-1)}\sin^{d-1}\theta\leq 1.
\end{align}
Further, because all $\Xii$'s are \emph{i.i.d.} for $i\in\{1,2,\cdots,n\}$, we have
\begin{align}\label{eq.temp_030302}
    \prob_{\XX}\{\theta^*>\theta\}=\prob_{\XX}\left\{\min_{i\in\{1,2,\cdots,n\}}\arccos (\Xii^T\zzero)>\theta\right\}\leq &\left(1-\frac{C_d}{2\sqrt{2}(d-1)}\sin^{d-1}\theta\right)^n.
\end{align}
By Eq.~\eqref{eq.temp_030305} and  Lemma~\ref{le.sin}, we then have
\begin{align*}
    \sin\theta\geq \left(\frac{2\sqrt{2}(d-1)}{C_d}\right)^{\frac{1}{d-1}}n^{-\frac{1}{d-1}(1-\frac{1}{q})},
\end{align*}
i.e.,
\begin{align*}
    \frac{C_d}{2\sqrt{2}(d-1)}\sin^{d-1}\theta\geq n^{-(1-1/q)}.
\end{align*}
Thus, by Eq.~\eqref{eq.temp_052801}, Eq.~\eqref{eq.temp_030302}, and  Lemma~\ref{le.temp_030301}, we have
\begin{align*}
    &\prob_{\XX}[\mathcal{J}_2^c]=\prob_{\XX}\left\{\theta^*>\theta\right\}\leq \exp\left(-n^{\frac{1}{q}}\right).
\end{align*}

Combining the results of Step 1 and Step 2, we thus have
\begin{align*}
    \prob_{\XX,\Vzero}[\mathcal{J}_1\cap\mathcal{J}_2]=&\prob_{\XX,\Vzero}[\mathcal{J}_1\ |\ \mathcal{J}_2]\cdot  \prob_{\XX,\Vzero}[\mathcal{J}_2]\\
    =&\prob_{\Vzero}[\mathcal{J}_1\ |\ \mathcal{J}_2]\cdot \prob_{\XX}[\mathcal{J}_2]\text{ (because of $\Vzero$ and $\XX$ are independent)}\\
    \geq & \left(1-2\exp\left(-\frac{p\theta}{12\pi}\right)\right)\left(1- \exp\left(-n^{\frac{1}{q}}\right) \right)\\
    \geq &1 - \exp\left(-n^{\frac{1}{q}}\right)-2\exp\left(-\frac{p\theta}{12\pi}\right)\\
    = & 1 - \exp\left(-n^{\frac{1}{q}}\right) - 2\exp\left(-\frac{p}{24}\left(\frac{2\sqrt{2}(d-1)}{C_d}\right)^{\frac{1}{d-1}}n^{-\frac{1}{d-1}(1-\frac{1}{q})}\right)\text{ (by Eq.~\eqref{eq.temp_030306})}.
\end{align*}
By Eq.~\eqref{eq.def_Cd}, the conclusion of Proposition~\ref{prop.g_is_delta} thus follows.

\subsection{Proof of Lemma~\ref{le.temp_030501}}\label{subsec.le030501}
\begin{proof}
When $\xsmall_1=\zzero$, the conclusion of this lemma trivially holds because $\mathcal{C}_{-\xsmall_1,\zzero}^{\Vzero}=\mathcal{C}_{\xsmall_1,-\zzero}^{\Vzero}=\varnothing$ (because $-\xsmall^T\Vzeroj$ and $\zzero^T\Vzeroj$ cannot be both positive or negative at the same time when $\xsmall_1=\zzero$.). It remains to consider $\xsmall_1\neq \zzero$. Define 
\begin{align*}
    \zzeroorth\defeq \frac{\xsmall_1-(\xsmall_1^T\zzero)\zzero}{\|\xsmall_1-(\xsmall_1^T\zzero)\zzero\|_2}.
\end{align*}
Thus, we have $\zzeroorth^T\zzero=0$ and $\|\zzeroorth\|_2=1$, i.e., $\zzero$ and $\zzeroorth$ are orthonormal basis vectors on the 2D plane $\mathcal{L}$ spanned by $\xsmall_1$ and $\zzero$. Thus, we can represent $\xsmall_1$ as
\begin{align*}
    \xsmall_1=\cos \varphi \cdot \zzero + \sin \varphi \cdot \zzeroorth\in \mathcal{L}.
\end{align*}
For any $\theta\in[\varphi,\pi]$, we construct $\xsmall_2$ as
\begin{align*}
    \xsmall_2\defeq \cos \theta \cdot \zzero + \sin \theta \cdot \zzeroorth\in\mathcal{L}.
\end{align*}
In order to show $\mathcal{C}_{-\xsmall_1,\zzero}^{\Vzero}\subseteq \mathcal{C}_{-\xsmall_2,\zzero}^{\Vzero}$, we only need to prove any $j\in \mathcal{C}_{-\xsmall_1,\zzero}^{\Vzero}$ must in $\mathcal{C}_{-\xsmall_2,\zzero}^{\Vzero}$.
For any $\Vzeroj$, $j=1,2,\cdots,p$, define the angle $\theta_j\in [0,2\pi]$ as the angle between $\zzero$ and $\Vzeroj$'s projected component $\bm{v}_j$ on $\mathcal{L}$\footnote{Note that such an angle $\theta_j$ is well defined as long as $\Vzeroj$ is not perpendicular to $\mathcal{L}$. The reason that we do not need to worry about those $j$'s such that $\Vzeroj\perp \mathcal{L}$ is as follows. When $\Vzeroj\perp \mathcal{L}$, we then have $\xsmall_1^T \Vzeroj=\xsmall_2^T\Vzeroj=\zzero^T \Vzeroj = 0$. Thus, those $j$'s do not belong to any set $\mathcal{C}_{-\xsmall_1,\zzero}^{\Vzero}$, $\mathcal{C}_{-\xsmall_2,\zzero}^{\Vzero}$, $\mathcal{C}_{\xsmall_1,-\zzero}^{\Vzero}$, or $\mathcal{C}_{\xsmall_2,-\zzero}^{\Vzero}$  in Eq.~\eqref{eq.temp_030501}.}, i.e.,
\begin{align*}
    \bm{v}_j=\cos \theta_j \cdot \zzero + \sin \theta_j \cdot \zzeroorth\in\mathcal{L}.
\end{align*}
By the proof of Lemma~\ref{le.spherePortion}, we know that $j\in \mathcal{C}_{-\xsmall_1,\zzero}^{\Vzero}$ if and only if $\theta_j\in (-\frac{\pi}{2},\frac{\pi}{2}) \cap (\pi+\varphi-\frac{\pi}{2},\pi+\varphi+\frac{\pi}{2})$ (mod $2\pi$). Similarly, $j\in \mathcal{C}_{-\xsmall_2,\zzero}^{\Vzero}$ if and only if $\theta_j\in (-\frac{\pi}{2},\frac{\pi}{2}) \cap (\pi+\theta-\frac{\pi}{2},\pi+\theta+\frac{\pi}{2})$ (mod $2\pi$). Because $\varphi \in [0, \pi]$ and $\theta \in [\varphi, \pi]$, we have
\begin{align*}
    (-\frac{\pi}{2},\frac{\pi}{2}) \cap (\pi+\varphi-\frac{\pi}{2},\pi+\varphi+\frac{\pi}{2})\subseteq (-\frac{\pi}{2},\frac{\pi}{2}) \cap (\pi+\theta-\frac{\pi}{2},\pi+\theta+\frac{\pi}{2})\text{ (mod $2\pi$)}.
\end{align*}
Thus, whenever $j\in \mathcal{C}_{-\xsmall_1,\zzero}^{\Vzero}$, we must have $j\in \mathcal{C}_{-\xsmall_2,\zzero}^{\Vzero}$. Therefore, we conclude that $\mathcal{C}_{-\xsmall_1,\zzero}^{\Vzero}\in \mathcal{C}_{-\xsmall_2,\zzero}^{\Vzero}$. Using a similar method, we can also show that $\mathcal{C}_{\xsmall_1,-\zzero}^{\Vzero}\subseteq \mathcal{C}_{\xsmall_2,-\zzero}^{\Vzero}$. The result of this lemma thus follows.
\end{proof}




\end{document}